\documentclass{article}

\usepackage{microtype}
\usepackage{graphicx}
\usepackage{caption}
\usepackage{subcaption}
\usepackage{booktabs} 
\usepackage{svg}
\usepackage{hyperref}
\usepackage{multirow}

\usepackage[preprint]{icml2026}
\usepackage{amssymb}
\usepackage{bbm}
\usepackage{mathtools}
\usepackage{amsthm}
\usepackage[capitalize,noabbrev]{cleveref}
\usepackage{enumitem}
\usepackage{etoc}
\usepackage{xcolor}
\usepackage{amsmath}
\usepackage{pifont}
\usepackage{siunitx}
\usepackage[table]{xcolor}
\usepackage{etoc}
\usepackage{enumitem}

\theoremstyle{plain}
\newtheorem{theorem}{Theorem}[section]
\newtheorem{proposition}[theorem]{Proposition}
\newtheorem{lemma}[theorem]{Lemma}

\theoremstyle{definition}
\newtheorem{definition}[theorem]{Definition}
\newtheorem{assumption}[theorem]{Assumption}
\theoremstyle{remark}
\newtheorem{remark}[theorem]{Remark}

\definecolor{mygrey}{RGB}{239, 239, 239}
\newcommand{\mygrey}{\cellcolor{mygrey}}

\DeclareMathOperator{\amax}{amax}
\DeclareMathOperator{\MSE}{MSE}

\usepackage[textsize=tiny]{todonotes}
\icmltitlerunning{Dissecting Outlier Dynamics in LLM NVFP4 Pretraining}
\begin{document}

\twocolumn[
\icmltitle{Dissecting Outlier Dynamics in LLM NVFP4 Pretraining}




\icmlsetsymbol{equal}{*}
\begin{icmlauthorlist}
\icmlauthor{Peijie Dong}{equal,hkustgz}
\icmlauthor{Ruibo Fan}{equal,hkustgz}
\icmlauthor{Yuechen Tao}{alibaba}
\icmlauthor{Di Mou}{alibaba}
\icmlauthor{Wenhu Hu}{alibaba}
\icmlauthor{Zhenheng Tang}{hkust}
\icmlauthor{Yinghao Yu}{alibaba}
\icmlauthor{Jiamang Wang}{alibaba}
\icmlauthor{Wenbo Su}{alibaba}
\icmlauthor{Guodong Yang}{alibaba}
\icmlauthor{Liping Zhang}{alibaba}
\icmlauthor{Xiaowen Chu}{hkustgz}
\icmlauthor{Baochun Li}{toronto}
\icmlauthor{Bo Li}{hkust}

\end{icmlauthorlist}

\icmlaffiliation{hkustgz}{The Hong Kong University of Science and Technology (GuangZhou)}
\icmlaffiliation{hkust}{The Hong Kong University of Science and Technology}
\icmlaffiliation{alibaba}{Alibaba Group}
\icmlaffiliation{toronto}{University of Toronto}

\icmlcorrespondingauthor{Xiaowen Chu}{xwchu@hkust-gz.edu.cn}

\icmlkeywords{Linear Attention, FP4, NVFP4, Low-Precision Training, Large Language Models}

\vskip 0.3in
]

\printAffiliationsAndNotice{\icmlEqualContribution}

\begin{abstract}

Training large language models using 4-bit arithmetic enhances throughput and memory efficiency. Yet, the limited dynamic range of FP4 increases sensitivity to outliers. While NVFP4 mitigates quantization error via hierarchical microscaling, a persistent loss gap remains compared to BF16. This study conducts a longitudinal analysis of outlier dynamics across architecture during NVFP4 pretraining, focusing on \textit{where they localize, why they occur, and how they evolve temporally}. We find that, compared with Softmax Attention (\texttt{SA}), Linear Attention (\texttt{LA}) reduces per-tensor heavy tails but still exhibits persistent block-level spikes under block quantization. Our analysis attributes outliers to specific architectural components: Softmax in \texttt{SA}, gating in \texttt{LA}, and SwiGLU in \texttt{FFN}, with ``post-QK'' operations exhibiting higher sensitivity to quantization. Notably, outliers evolve from transient spikes early in training to a small set of persistent hot channels (i.e., channels with persistently large magnitudes) in later stages. Based on these findings, we introduce \textbf{\underline{H}ot-\underline{C}hannel \underline{P}atch (HCP)}, an online compensation mechanism that identifies hot channels and reinjects residuals using hardware-efficient kernels. We then develop \textbf{CHON}, an NVFP4 training recipe integrating HCP with post-QK operation protection. On GLA-1.3B model trained for 60B tokens, CHON reduces the loss gap to BF16 from 0.94\% to 0.58\% while maintaining downstream accuracy.

\end{abstract}

\section{Introduction}
\label{sec:intro}

\begin{figure}[t]
    \centering
    \includegraphics[width=0.95\linewidth]{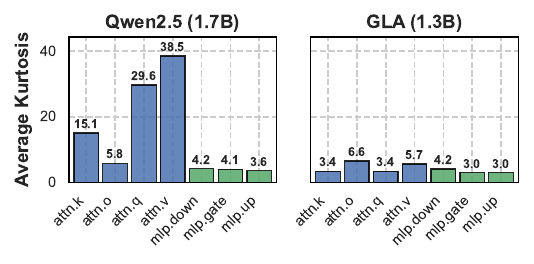}
    \vspace{-10pt}
    \caption{\textbf{Activation kurtosis of GLA vs. Qwen2.5.} Lower kurtosis values in GLA reflect reduced heavy-tailedness and fewer systematic outliers across both attention and MLP layers.}
    \vspace{-15pt}
    \label{fig:fig1_kurtosis}
\end{figure}

The computational cost associated with pretraining large language models (LLMs) continues to escalate as parameter counts, token budgets, and context lengths increase. Consequently, mixed-precision training, particularly FP8 pipelines~\cite{peng2023fp8,fishman2024scaling,liu2024deepseekv3}, has emerged as a key solution, reducing memory usage and accelerating computation while maintaining performance at scale. Nevertheless, the efficiency frontier is now increasingly determined by native support for 4-bit arithmetic, which offers further reductions in memory footprint compared to FP8 and increases peak arithmetic throughput on next-generation accelerators.

\begin{figure*}[t]
    \centering
    \includegraphics[width=1\linewidth]{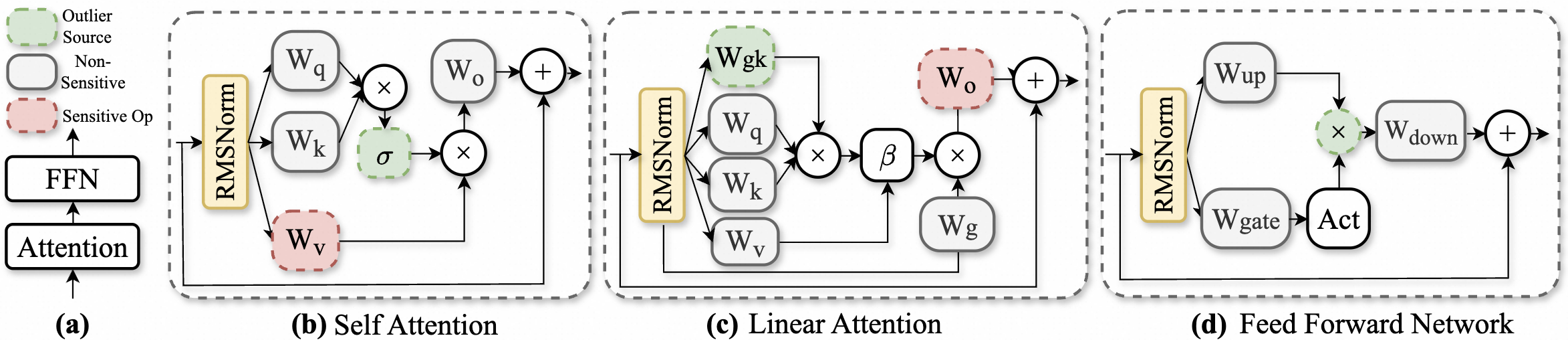}
    \vspace{-14pt}
    \caption{\textbf{Architectural overview and sensitivity analysis.} Detailed components of (a) residual blocks, (b) Softmax Attention, (c) Gated Linear Attention, and (d) gated-FFN. Color coding identifies outlier sources (green) and sensitive ``post-QK'' operations (red), providing a basis for targeted precision allocation in the CHON recipe.}
    \vspace{-10pt}
    \label{fig:component}
\end{figure*}

Transitioning from FP8 to FP4 represents a significant reduction in precision. The limited representational capacity of FP4 increases sensitivity to heavy-tailed distributions and outliers in activations and weights. Although most tensor values are well-behaved, a small fraction of outliers can saturate quantizers, producing large quantization errors that propagate through depth and destabilize training. To address this, modern hardware introduces microscaling formats~\cite{rouhani2023microscaling} like NVFP4, employing tensor-level and block-level scaling to suppress quantization noise~\cite{rouhani2023microscaling}. Nevertheless, state-of-the-art NVFP4 recipe~\cite{nvidia_recipe} still exhibit a notable loss gap relative to BF16, and this gap may increase during the late learning-rate decay stage.

A primary challenge in narrowing this gap is the dynamic nature of outliers during pretraining. Most existing outlier analyses focus on post-training quantization~\cite{dettmers2022llmint8,wei2023outlier,lee2024owq,zhao2024atom}, which are typically conducted offline on trained checkpoints, aiming to characterize static channel-wise extremes. In contrast, the distributions of internal tensors evolve during pretraining, and the interaction between outliers and block-wise quantization shifts across training stages and model architectures. Therefore, systematically dissecting these outlier dynamics is a critical yet underexplored prerequisite for enabling large-scale NVFP4 pretraining.

In this paper, we present the first longitudinal analysis of outlier dynamics across Softmax (\texttt{SA}) and Linear Attention (\texttt{LA}) architectures in NVFP4 pretraining. We use Qwen3~\cite{yang2025qwen3} and Gated Linear Attention (GLA) as representative \texttt{SA} and \texttt{LA} models, respectively.
Training runs are instrumented to monitor indicators, including kurtosis, top-$k$ magnitude, alignment, flush-to-zero, and quantization errors.
This study yields three consistent findings. \ding{172} \texttt{LA} demonstrates milder per-tensor heavy tails than \texttt{SA}, as indicated by activation kurtosis trends (Fig.~\ref{fig:fig1_kurtosis}), although localized spikes at the block level relevant for NVFP4 quantization still occur. \ding{173} We find ``post-QK'' operations (i.e. projections after QK interactions such as $W_v$ or $W_o$) are more sensitive than other operations, as shown in Fig.~\ref{fig:component}. \ding{174} Outliers transition from drifting spikes early in training to persistent hot channels (i.e., channels with large magnitude) later, indicating that they can be identified and protected without recomputation each step.

Building on these observations, we introduce Hot-Channel Patch (HCP), a lightweight mechanism that periodically identifies hot channels and compensates their quantization residuals using hardware-efficient kernels. We show that HCP yields a lower MSE upper bound in \cref{lemma:mse_upper_bound}.
HCP is integrated with the NVFP4 recipe and targeted post-QK operation protection for sensitive post-QK operations, leading to a new recipe called Compensated Hot-channel Optimization for NVFP4 (CHON). CHON reduces the NVFP4-to-BF16 loss gap from 0.94\% to 0.58\% on a 1.3B GLA model trained for 60B tokens, while maintaining downstream accuracy. Further, we evaluate CHON for supervised fine-tuning and reinforcement learning to assess transferability beyond pretraining.
Our contributions are as follows:
\begin{itemize}[leftmargin=*,itemsep=2pt,topsep=2pt]
    \item \textbf{Longitudinal outlier analysis for NVFP4 pretraining.} We track outlier dynamics throughout training and uncover a transition from transient spikes to persistent hot channels in later stages.
    \item \textbf{Architecture- and operation-level characterization.} We show that \texttt{LA} reduces global heavy tails relative to \texttt{SA}, and ``post-QK'' operations are more sensitive.
    \item \textbf{CHON recipe and Hot-Channel Patch.} We propose the CHON recipe that incorporates ``post-QK'' protection and Hot-Channel Patch (HCP) to reduce the quantization error in persistent hot channels.
\end{itemize}

\section{Related Work}
\label{sec:related_work}

\paragraph{Outliers in LLMs.} Outliers refer to data points that differ significantly from other data points in a distribution.
They are observed during both training~\cite{park2025osp, nvidia_recipe} and inference~\cite{sun2024massive, xiao2022smoothquant}.
Although outliers can be beneficial and are often necessary for model expressiveness and capability~\cite{raman2025rethinking, dettmers2022llmint8, An2025SystematicOI}, they also pose significant challenges for quantization and low-precision training.
For instance, LLM.int8()~\cite{dettmers2022llmint8} demonstrates that forcing outliers into low precision leads to performance collapse, proposing instead to retain them in FP16. 
Moreover, SmoothQuant~\cite{xiao2022smoothquant} and AWQ~\cite{lin2023awq} identify that there are more channel outliers in activations than weights. They employ smoothing factors to migrate quantization difficulty from activations to weights, balancing their respective ranges. 
However, existing studies primarily analyze outlier behavior during inference, and few studies examine how outliers evolve over the course of LLM pretraining.

\paragraph{Architectural factors shaping outliers.}
Prior analyses~\cite{An2025SystematicOI} attribute activation outliers in Transformers primarily to the Softmax and context-aware scaling within Softmax Attention. Softmax Attention combines quadratic time and memory with exponential weighting, which can amplify rare extreme dot-product scores into activation spikes~\cite{vaswani2017attention}. Linear Attention~\cite{mamba, dao2024mamba2,yang2024gated,sun2023retentive,Qin2024HGRN2GL} removes Softmax via associative or kernelized normalization and achieves $O(T)$ computation, typically with smoother activation statistics~\cite{katharopoulos2020transformers, choromanski2020performer, peng2021random}. 
There are also studies on the architectural factors beyond attention that contribute to outliers. For example, \citet{wei2023outlier} identify that the SwiGLU activation in FFN layers can produce outliers due to its quadratic weight alignment. Furthermore, Attention Sink~\cite{xiao2023efficient,openai2025gptoss120bgptoss20bmodel} can alleviate outliers in Softmax Attention based LLMs~\cite{son2024prefixing}. QK Normalization~\cite{yang2025qwen3} can also suppress outliers by normalizing the query and key vectors before dot-product attention computation.
However, how outliers evolve during low-bit LLM pretraining has not been well studied.

\paragraph{NVFP4 Pretraining.} With the advent of next-generation hardware (e.g., Blackwell GPU), NVFP4 and MXFP4 have emerged as promising formats for enhancing kernel efficiency and throughput. Recent research attempts NVFP4 pretraining. For example, MSRA~\cite{wang2025optimizing} updates the straight-through estimator with a differentiable gradient estimator during backward. AWS~\cite{aws_mxfp4} employs Random Hadamard Transform (RHT) and Stochastic Rounding (SR) during backward to ensure unbiased gradient estimates on MXFP4 pretraining. 
Chmiel et al.~\cite{intel_fp4} demonstrate that using RTN for the forward pass and SR for the backward pass stabilizes training, with an additional recovery phase in BF16 in the final stage. 
Quartet~\cite{castro2025quartet} introduces native 4-bit training on Blackwell GPUs using the MXFP4 format. By employing a recipe that combines Hadamard transforms with clipping-aware quantization, it achieves near-lossless performance relative to FP16. Metis~\cite{metis_fp4} proposes to split tensors into a small low-rank branch and quantization-friendly residual. TetraJet-v2~\cite{tetrajet} introduces unbiased double-block quantization incorporating OsciReset and OutControl mechanisms. The NVIDIA NVFP4 recipe~\cite{nvidia_recipe} establishes a robust quantization framework tailored for large-scale pretraining. When evaluated on a 12B Hybrid Mamba-Transformer model with 10T tokens, this approach maintains a 1\% loss gap relative to BF16, though this disparity expands to 1.5\% during the learning rate decay stage.

\begin{figure}[t]
    \centering
    \includegraphics[width=0.48\linewidth]{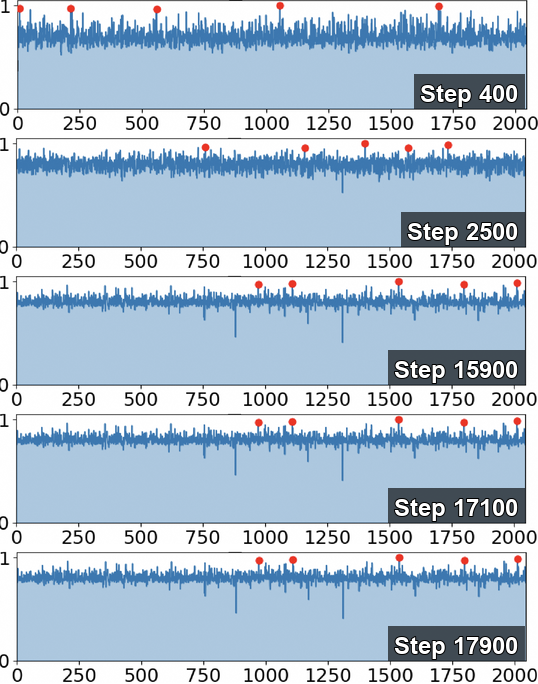}
    \hfill
    \includegraphics[width=0.48\linewidth]{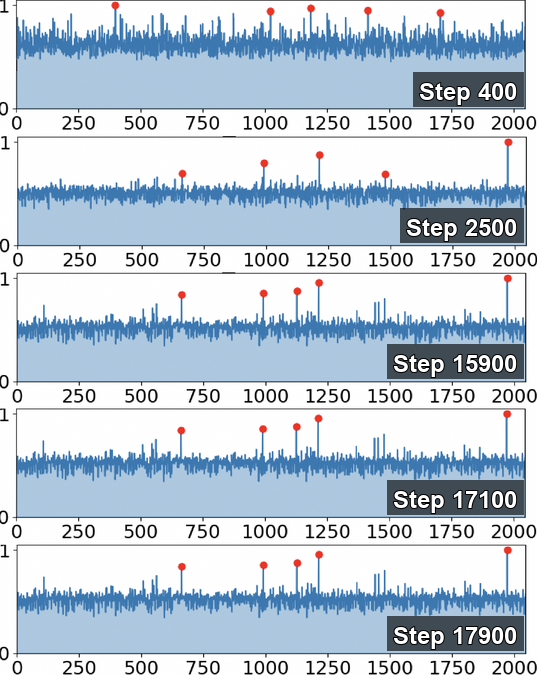}
    \vspace{-4pt}
    \caption{\textbf{Temporal evolution of activation outliers in Up and Output projections.} Red markers illustrate a clear transition from transient, drifting spikes (Step 400) to persistent, spatially fixed ``hot channels'' (Step 15k).}
    \label{fig:dynamic_upproj}
    \vspace{-6pt}
\end{figure}

\begin{figure}[t]
    \centering
    \includegraphics[width=0.995\linewidth]{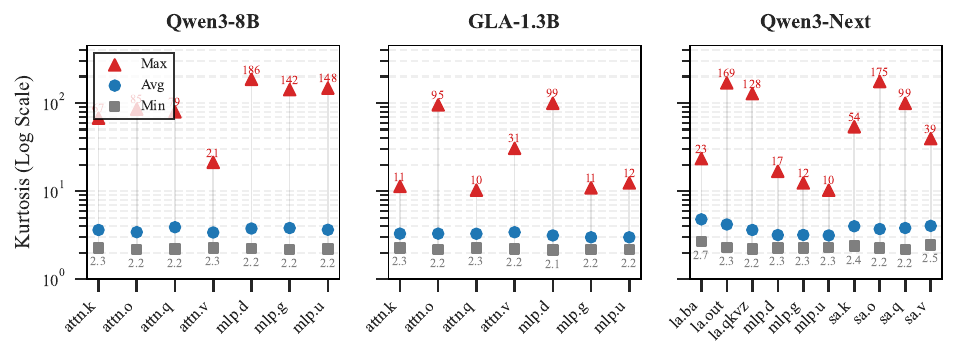}
    \vspace{-20pt}
    \caption{Comparison of block-level kurtosis for Softmax and Linear Attention. Distribution of min, avg, and max kurtosis across components. Localized ``heavy tails'' persist at the block level even when per-tensor kurtosis remains moderate.}
    \label{fig:block_kurtosis_models}
    \vspace{-10pt}
\end{figure}

\begin{figure}[t]
    \centering
    \includegraphics[width=0.99\linewidth]{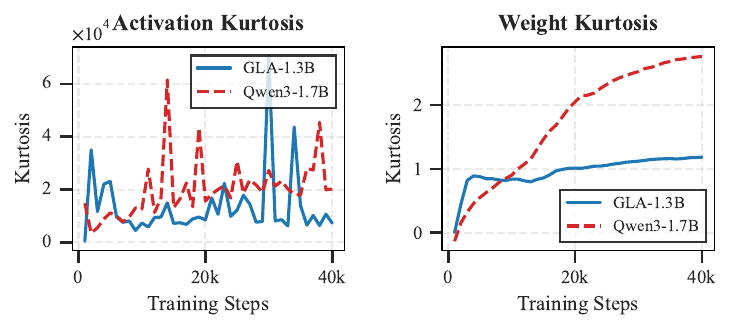}
    \vspace{-10pt}
    \caption{\textbf{Evolution of per-tensor kurtosis.} GLA-1.3B (blue) demonstrates superior weight distribution stability compared to Qwen3-1.7B (red), whereas activation kurtosis exhibits similar stochastic volatility across both architectures.}
    \label{fig:model_kurtosis}
    \vspace{-12pt}
\end{figure}

\section{Outlier Dynamics: Where, Why, and How}
\label{sec:outlier_dynamic_analysis}

\textbf{Definition.} We define \emph{outliers} as values that deviate significantly from the central mass of a distribution, typically manifesting as heavy tails. For NVFP4 pretraining, these outliers induce quantization errors, numerical instability, and performance degradation. We monitor outliers both statically (from checkpoints) and dynamically (during training).

\textbf{Flush-to-Zero (FTZ).} Under NVFP4, values are quantized to FP4 E2M1 after applying an effective encode scale per block, as detailed in App.~\ref{app:nvfp4_details}.
FTZ denotes the ratio of \emph{underflow-to-zero} events where scaled values collapse to exactly zero. For a tensor $X$ partitioned into blocks with effective encode scale $s_{\mathrm{enc},b}$, then FTZ is: 
\begin{equation}
\frac{1}{|X|} \sum_{i} \mathbbm{1} \Big\{ Q_{\mathrm{E2M1}}\big(x_i \cdot s_{\mathrm{enc},b}\big) = 0 \Big\},\nonumber
\end{equation}
where $b$ denotes the block containing element $i$ and $Q_{\mathrm{E2M1}}(\cdot)$ is the FP4 E2M1 quantizer. A higher FTZ indicates greater underflow risk and irreversible loss of small-but-nonzero signals.

\textbf{Kurtosis.}
For a variable $x$ with mean $\mu$ and standard deviation $\sigma$, excess kurtosis~\cite{Westfall2014KurtosisPeakedness} is defined as:
\begin{equation}
\kappa(x) = \frac{\mathbb{E}\big[(x-\mu)^4\big]}{\sigma^4} - 3.
\end{equation}
We monitor $\kappa$ per tensor and per block for activations and weights during training.
Heavy-tailed distributions yield larger $\kappa$, indicating higher outlier propensity and elevated risk of overflow and underflow under FP4.

\begin{figure}[t]
     \centering
     \begin{subfigure}[b]{0.485\linewidth}
         \centering
         \includegraphics[width=\linewidth]{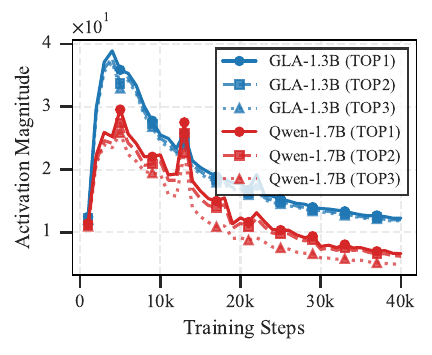}
         \vspace{-3pt}
         \caption{Activation Magnitude.}
         \label{fig:top3_mag_act}
     \end{subfigure}
     \hfill
     \begin{subfigure}[b]{0.475\linewidth}
         \centering
         \includegraphics[width=\linewidth]{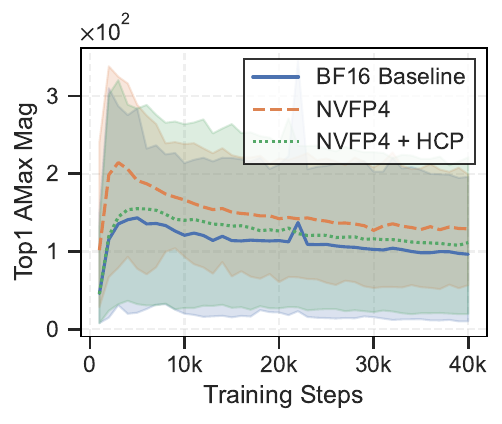}
         \vspace{-3pt}
         \caption{GK Proj. Activation.}
         \label{fig:topk_mag_training_of_gkproj}
     \end{subfigure}
     \vspace{-5pt}
     \caption{\textbf{Analysis of Activation Evolution.} (a) Top-$k$ activation magnitudes for GLA-1.3B and Qwen-1.7B. (b) Evolution of the GK projection top-1 activation under BF16, NVFP4, and HCP.}
     \label{fig:overall_comparison}
     \vspace{-10pt}
\end{figure}

\subsection{Where Do Outliers Emerge?}
\label{subsec:has_outlier}

We first locate outlier emergence across architectures and components, specifically comparing Softmax and Linear Attention architectures, and examining the sensitivity of ``post-QK'' operations.

\paragraph{\texttt{LA} exhibits lighter tails than \texttt{SA}.} We compare the per-tensor and per-block kurtosis of each component of Softmax-based Transformers (e.g., Qwen) against \texttt{LA} models (e.g., GLA). 
\textit{\ding{172} Per-tensor Kurtosis:} As shown in Fig.~\ref{fig:fig1_kurtosis}, GLA exhibits lower weight kurtosis than Qwen, indicating weaker heavy tails and fewer extreme outliers. In Qwen, Attention exhibit higher kurtosis than MLP, where the outlier comes from softmax operation~\cite{An2025SystematicOI}. (See Fig.~\ref{fig:kurtosis_raw} and Fig.~\ref{fig:kurtosis_block} for extended results.)
\textit{\ding{173} Per-block Kurtosis:} To investigate the impact on NVFP4 quantization, we analyzed kurtosis at a 16$\times$16 block. Despite the superior global stability of \texttt{LA}, the per-block spikes in Fig.~\ref{fig:block_kurtosis_models} demonstrate that \texttt{LA} still experiences high local kurtosis in specific regions, posing challenges for block-based quantization. Refer to App.~\ref{app:kurtosis} for more details.

\paragraph{``Post-QK'' Operation Sensitivity.} As shown in Fig.~\ref{fig:component}, we identify the outlier sources and sensitive operations (sensitive ops), observing that both \texttt{SA} and \texttt{LA} exhibit significant sensitivity in their ``post-QK'' components (see Tab.~\ref{tab:sensitivity_scores}). Specifically, for \texttt{SA} (e.g., Qwen3-1.7B), the value projection is the sensitive bottleneck, whereas for \texttt{LA} (e.g., GLA-1.3B), the output projection emerges as the critical sensitive operator, as shown in Tab.~\ref{tab:sensitivity_scores}. Notably, these sensitive regions are consistently situated downstream of outlier sources. In \texttt{SA}, the value projection follows the QK interaction and softmax operation, a stage where outliers are more prevalent. In \texttt{LA}, the output projection follows the $gk\_proj$ and $g\_proj$ layers, where outlier magnitudes reach their peak. As shown in Fig.~\ref{fig:topk_mag_training_of_gkproj}, the activation range of $gk$ is exceptionally broad, reaching magnitudes as high as $350$, which imposes significant numerical pressure on the subsequent Output projection. For more operation quantization error analysis, refer to App.~\ref{app:quant_error_dynamics}
Therefore, preserving these sensitive ``post-QK'' operations in higher precision can benefit low-precision pretraining.

\begin{figure}[t]
    \centering
    \includegraphics[width=0.99\linewidth]{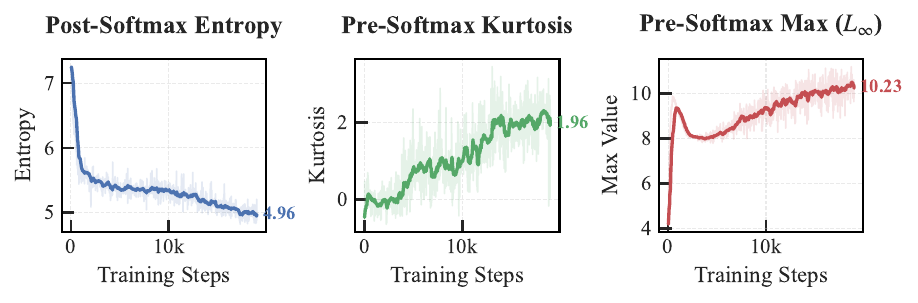}
    \caption{\textbf{Evolution of softmax-induced instability.} Increasing Pre-Softmax Kurtosis and $L_{+\infty}$ norms alongside declining entropy quantify the structural emergence of systematic outliers.}
    \vspace{-8pt}
    \label{fig:softmax_indicator}
\end{figure}

\subsection{Why Do Outliers Emerge?}
\label{subsec:why_outlier}



\paragraph{Softmax as Outlier Source in \texttt{SA}.} Prior research~\cite{An2025SystematicOI} identifies Softmax as the root cause of outliers due to its normalization constraint. Specifically, the sum-to-one requirement forces the model to adopt an extreme dynamic range to effectively suppress uninformative tokens. To assign near-zero contributions to these tokens, the model is compelled to generate Pre-Softmax values with large magnitudes. We empirically validate this mechanism by tracking three metrics in Fig.~\ref{fig:softmax_indicator}: Post-Softmax Entropy, Pre-Softmax Kurtosis, and Pre-Softmax Max. \ding{172} The decline in Post-Softmax Entropy reflects an increasing concentration of attention weights. \ding{173} The rise in Pre-Softmax Kurtosis indicates a shift toward heavy-tailed distributions in the logit space, a hallmark of outlier formation. \ding{174} The growth of the Pre-Softmax Max (from 4 to 10) confirms the intensifying logit disparity necessary to saturate the Softmax function.

\paragraph{Gating as Outlier Source in \texttt{LA}.}
Our analysis identifies the elementwise exponential function $\phi(x)=\exp(t \cdot x)$($t>2$) in the gating mechanism, specifically the $gk\_proj$ layer, as the primary source for extreme outliers in GLA. The decay factor $\lambda_t$ is derived from $gk\_proj$ via a log-sigmoid function. To achieve state resetting ($\lambda_t \approx 0$), the pre-activation input must be extremely negative (e.g., $\approx -120$), whereas long-term retention ($\lambda_t \approx 1$) requires positive saturation. This necessitates a massive dynamic range (e.g., $[-120, 80]$) that poses severe challenges for uniform FP4 quantization. As illustrated in Fig.~\ref{fig:topk_mag_training_of_gkproj}, the average Top-1 magnitude of $gk\_proj$ significantly surpasses that of other components (see App.~\ref{app:gk_outlier_analysis} for extended analysis). 


\begin{figure}[t]
    \centering
    \includegraphics[width=0.85\linewidth]{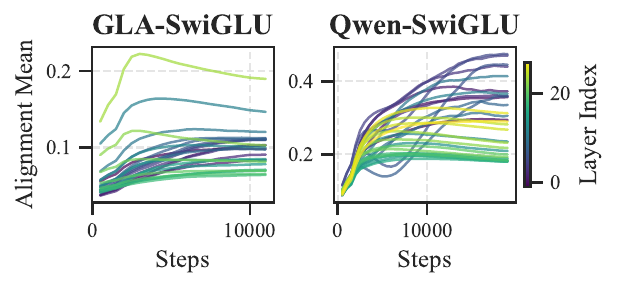}
    \vspace{-7pt}
    \caption{\textbf{SwiGLU weight cosine alignment dynamics.} Layer-wise mean cosine similarity between $W_{up}$ and $W_{gate}$ during pretraining. Higher alignment in Qwen suggests a stronger structural amplification of FFN outliers relative to GLA.}
    \vspace{-10pt}
    \label{fig:swiglu_alignment}
\end{figure}

\paragraph{SwiGLU as the Outlier Source in FFNs.} 
As a cornerstone of modern LLMs, SwiGLU is defined as $\text{SwiGLU}(x) = (x W_{\text{up}}) \odot \text{Swish}(x W_{\text{gate}})$, exhibiting quadratic growth scaling with $\mathcal{O}(\|x\|^2)$. Prolonged training with weight decay ($\ell_2$ regularization) induces weight alignment ($W_{\text{up}} \parallel W_{\text{gate}}$)~\cite{fishman2024scaling}, effectively transforming SwiGLU into a potent outlier amplifier. These quadratic spikes frequently exceed the limited dynamic range of FP4, triggering numerical overflow and training divergence. We empirically verify this mechanism in Fig.~\ref{fig:swiglu_alignment} by monitoring the cosine alignment $\frac{|W_{\text{up}, i}^\top W_{\text{gate}, i}|}{\|W_{\text{up}, i}\| \|W_{\text{gate}, i}\|}$, which shows a steady increase throughout training. Furthermore, GLA produces significantly fewer outliers than Qwen, as shown in Fig.~\ref{fig:model_kurtosis}). This is attributed to the lower attention activations in GLA, which mitigate the amplification effect on subsequent FFN layers.


\paragraph{Outlier Impact of RMSNorm.}
We analyze the impact of RMSNorm on outlier by visualizing the distribution of the learnable scaling factor $\gamma$ in Fig.~\ref{fig:norm_comparison} and App.~\ref{app:norm_impact}. We observe several key trends: 
(i) Softmax Attention yields substantially higher $\gamma$ magnitudes ($\gamma > 1$) compared to Linear Attention ($\gamma < 1$); 
(ii) smaller models, such as Qwen3-0.6B, possess larger $\gamma$ values than larger variants like Qwen3-30B-A3B; 
(iii) models within the same family (e.g., RetNet 1.3B and 2.7B) display consistent distribution patterns across scales; and 
(iv) Qwen3-Next~\cite{yang2025qwen3} exhibits generally $\gamma < 1$ across layers, suppressing the activation magnitude layer by layer.

\begin{figure}[t]
    \centering
    \includegraphics[width=0.9\linewidth]{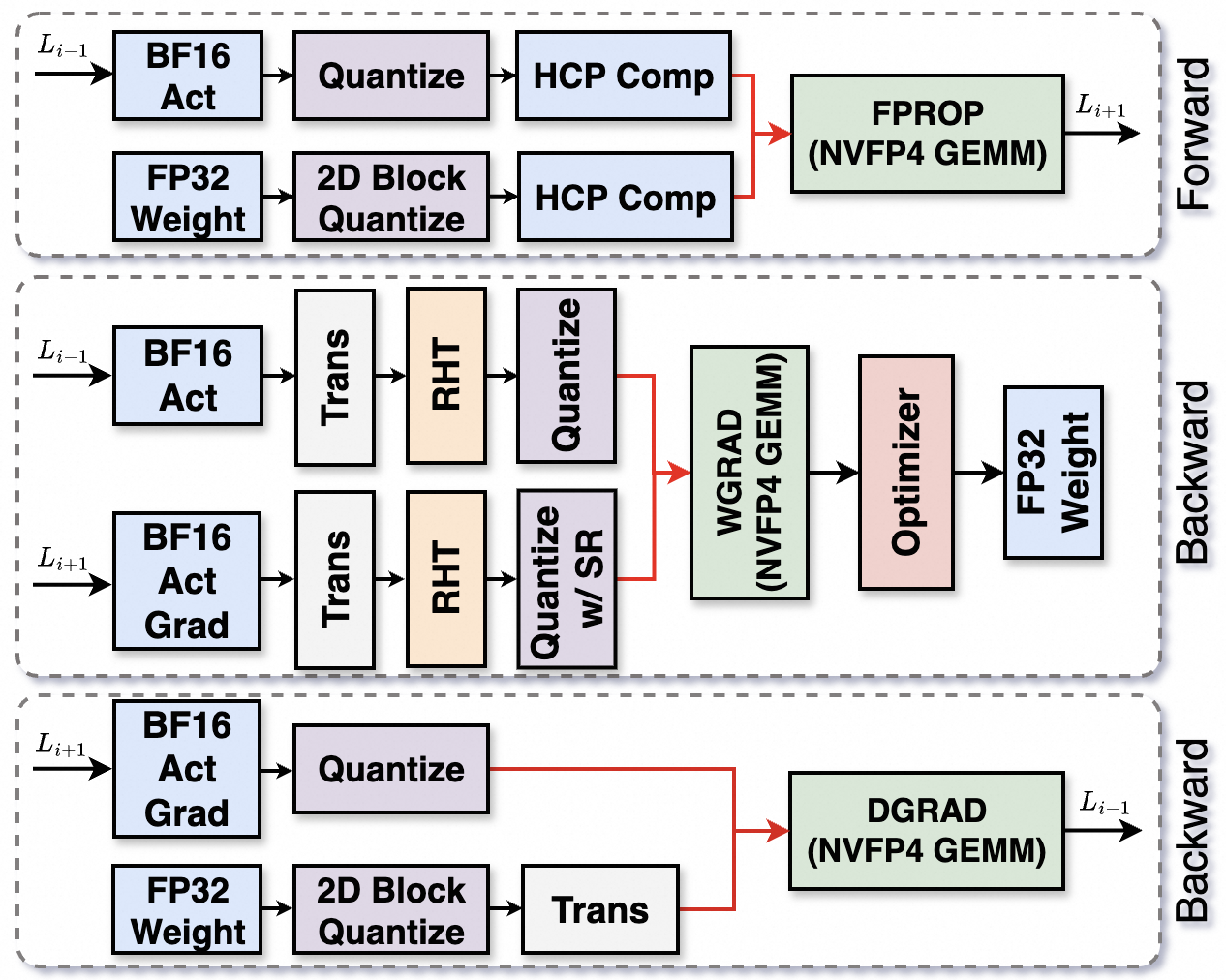}
    \vspace{-1pt}
    \caption{\textbf{Computational workflow of NVFP4-quantized linear layers.} CHON integrates Hot-Channel Patch (HCP) compensation during the forward pass and Stochastic Rounding (SR) via Random Hadamard Transform (RHT) in the backward pass.}
    \label{fig:compute_flow_nvfp4}
    \vspace{-5pt}
\end{figure}

\subsection{How Do Outliers Evolve?}
\label{subsec:evolve}

\paragraph{From Drifting Spikes to Fixed Hot Channels.} 
Early training stages (e.g., steps 400--5,400) exhibit dynamic, transient magnitude spikes that drift across channels as the model navigates the loss landscape (see Fig.~\ref{fig:dynamic_upproj}). However, as training progresses toward mid-to-late stages, the location of outliers stabilizes. Persistent hot channels (i.e., channels with persistent high magnitude) concentrate in a small subset of channels as shown in Fig.~\ref{fig:combined_outliers}. This transition from stochastic drift to structural fixation justifies the use of static, selective mitigation strategies (i.e., Hot-Channel Patch (HCP) in Sec.~\ref{sec:method}) over dynamic adaptation. This evolution is visible across checkpoints in Fig.~\ref{fig:proj_concat_voq} and App.~\ref{app:comp_dynamic}.

\paragraph{Top-$k$ Magnitude Evolution.}
Tracking the top-$k$ magnitude channels reveals a characteristic pattern: outliers start with extreme volatility and high magnitudes, then gradually stabilize and decrease (or plateau) as the model converges (see Fig.~\ref{fig:top3_mag_act}). The gap between top-1 and top-3 magnitudes narrows, indicating numerical stress diffuses; yet specific components like $gk$ projection maintain high magnitudes throughout training, behaving as permanent outliers and dominating the top-1 trajectory in Fig.~\ref{fig:topk_mag_training_of_gkproj} and App.~\ref{app:topk_mag}. 

\paragraph{Kurtosis over Training.}
Activation kurtosis remains volatile during early training: Qwen3 shows repeated spikes while GLA exhibits fewer but sharper bursts, consistent with transient outlier channels in Fig.~\ref{fig:model_kurtosis}. Weight kurtosis diverges more smoothly—Qwen3 climbs steadily and plateaus $>2$, whereas GLA rises quickly then flattens near $\approx1.1$, indicating lighter tails. The late-stage plateau suggests heavy-tailed behavior is learned and persists even after convergence.

\paragraph{Layerwise Magnitude Pattern.}
Top-$k$ magnitude is not uniform across layers. We observe a trend where outliers and high kurtosis often accumulate in the deeper layers of the network (particularly the last 4 layers). This accumulation is consistent with the observation that residual stream updates accumulate variance, and aligns with the deeper-layer peaks visible in the block aggregates of Fig.~\ref{fig:topk_mag_montage}. The final layers, responsible for mapping to the vocabulary space, often exhibit the most extreme magnitude, which reinforces why stabilizing them with high-precision is essential. 

\begin{figure}[t]
    \centering
    \includegraphics[width=.99\linewidth]{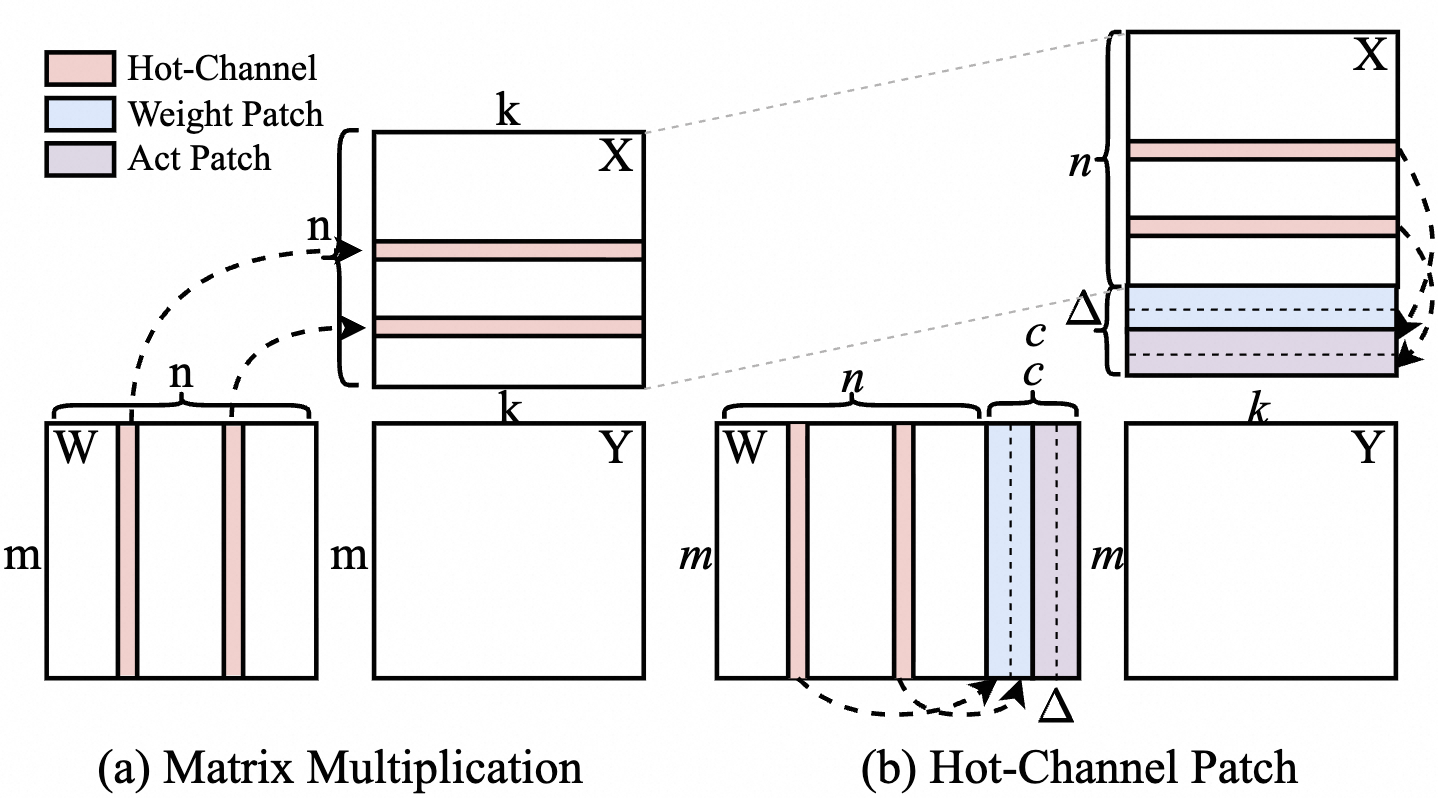}
    \vspace{-4pt}
    \caption{\textbf{Schematic of the Hot-Channel Patch (HCP) mechanism.} (a) Standard quantized GEMM. (b) HCP error compensation via residual channel injection for weights and activations.}
    \vspace{-12pt}
    \label{fig:hcp_overview}
\end{figure}

\paragraph{Flush-to-Zero Analysis.}
We track the NVFP4 flush-to-zero (FTZ) ratio for activations and weights over training. As shown in Fig.~\ref{fig:act-ftz-ablation} and \ref{fig:weight-ftz-ablation}: \ding{172} activations consistently exhibit higher FTZ than weights, while weights remain comparatively robust; \ding{173} enabling HCP reduces the activation FTZ and pulls distributions closer to the BF16; and \ding{174} gating components (e.g., g\_proj and gate\_proj) exhibits unique non-monotonic FTZ dynamics (decrease then increase), constrasting with monotonic increases elsewhere (Fig.~\ref{fig:act-ftz-ablation}). This reduced FTZ likely occurs because gating activations, designed to excite/suppress neurons, inherently avoid values near zero.
These findings align with our outlier-source analysis in Sec.~\ref{subsec:has_outlier} and App.~\ref{app:ftz_dynamic}.

\paragraph{Implications for NVFP4 Training.}
Grounded in the above analysis, we operationalize two concrete design choices on the top of NVIDIA NVFP4 Recipe~\cite{nvidia_recipe} that we develop in Sec.~\ref{sec:method} and validate in Sec.~\ref{sec:experiments}:
\textit{\ding{172} Hot-Channel Patch (HCP).} Because outlier channels drift early but become fixed hotspots later, we detect hot channels and then keep a fixed mask thereafter, using it to capture and reinject quantization error for those channels without recomputing dynamic masks each step (Fig.~\ref{fig:hcp_overview}).
\textit{\ding{173} Mixed-precision assignment from component and FTZ analysis.} We reserve higher precision for critical ``post-QK'' operations: always keep $W_{o}$ for \texttt{LA} and $W_{v}$ for \texttt{SA} to high precision. In addition, for GLA we use high precision for GK projection ($gk\_proj$); This allocation follows the observed outlier and FTZ patterns and minimizes error accumulation on these sensitive pathways.

\section{CHON Recipe}
\label{sec:method}

Building upon the outlier dynamics, we propose \textbf{C}ompensated \textbf{H}ot-channel \textbf{O}ptimization for \textbf{N}VFP4 (\textbf{CHON}), a recipe that balances stability and efficiency, as shown in Fig.~\ref{fig:compute_flow_nvfp4}. CHON recipe includes (1) NVIDIA NVFP4 recipe, (2) Hot-Channel Patch(HCP), and (3) ``post-QK'' operation protection.

\paragraph{\textbf{NVIDIA NVFP4 Recipe.}} We adopt the NVFP4 training recipe~\cite{nvidia_recipe} as a strong baseline (Fig.~\ref{fig:compute_flow_nvfp4}). Specifically, it (i) preserves the last four layers in high precision, (ii) employs 2D block scaling for weights and 1D scaling for both weights and gradients, (iii) utilizes RTN in the forward pass and SR with RHT in the backward pass.

\begin{figure}[t]
    \centering
    \includegraphics[width=.99\linewidth]{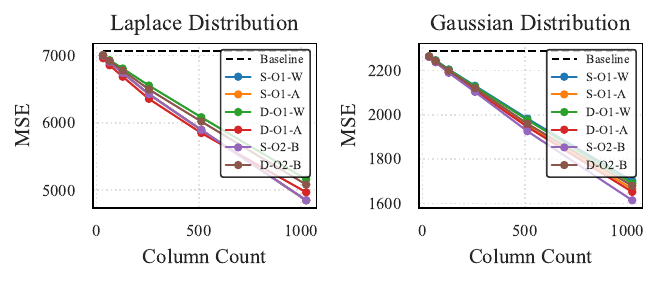}
    \vspace{-8pt}
    \caption{\textbf{HCP configuration comparison (S/D, $O2/O2$, W/A/B)}.Quantization error is evaluated against patched column counts under Laplace and Gaussian priors. \texttt{S-O2-B} consistently minimizes MSE relative to the baseline (dashed line).}
    \label{fig:hcp_recipes}
    \vspace{-8pt}
\end{figure}

\paragraph{Hot-Channel Patch (HCP).} 
Inspired by the outlier dynamics analysis (Sec.~\ref{sec:outlier_dynamic_analysis}) which reveals that outliers settle into fixed ``hot channels'', we introduce the \textbf{Hot-Channel Patch} to address the precision loss in low-bit quantization. Our approach leverages this empirical stability of outlier channels by implementing a periodic update strategy for hot channel identification, thereby optimizing the trade-off between hardware efficiency and model accuracy.

Given a linear transformation $\mathbf{Y} = \mathbf{W}^\top\mathbf{X}$ with $\mathbf{X} \in \mathbb{R}^{n\times k}$ and $\mathbf{W}^\top \in \mathbb{R}^{m\times n}$, we define the composite quantization operator $\mathcal{Q}(\cdot) = D(Q(\cdot))$, which maps floating-point values to their dequantized representations in a target format (e.g., FP4). The quantized activations and weights are denoted by $\widehat{\mathbf{X}} = \mathcal{Q}(\mathbf{X})=\mathbf{X} + \Delta \mathbf{X}$ and $\widehat{\mathbf{W}} = \mathcal{Q}(\mathbf{W})=\mathbf{W} + \Delta \mathbf{W}$ respectively, where \(\Delta \mathbf{X}\) and \(\Delta \mathbf{W}\) denote the residual errors introduced by \(\mathcal{Q}(\cdot)\).

\begin{table*}[t]
\centering
\footnotesize
\setlength{\tabcolsep}{1.8pt}
\caption{\textbf{Zero-shot performance on downstream benchmarks with 40B tokens.} $Acc_n$ denotes length-normalized accuracy; $Avg$ represents the mean accuracy across all tasks. CHON consistently narrows the quantization gap, achieving parity with BF16. }
\resizebox{.95\textwidth}{!}{
\begin{tabular}{@{}lccccccccccccccc@{}}
\toprule
\textbf{Setting} &
\multicolumn{2}{c}{\textbf{ARC-C}} &
\multicolumn{2}{c}{\textbf{ARC-E}} &
\multicolumn{2}{c}{\textbf{HellaSwag}} &
\multicolumn{2}{c}{\textbf{OBQA}} &
\multicolumn{2}{c}{\textbf{PIQA}} &
\multicolumn{2}{c}{\textbf{SciQ}} &
\textbf{Winograd} & \textbf{Avg} \\
& Acc & Acc$_{\text{n}}$ & Acc & Acc$_{\text{n}}$ & Acc & Acc$_{\text{n}}$ & Acc & Acc$_{\text{n}}$ & Acc & Acc$_{\text{n}}$ & Acc & Acc$_{\text{n}}$ & Acc & \\
\midrule

\multicolumn{15}{c}{\mygrey \textbf{GLA-340M}} \\
BF16 & 29.3{\scriptsize$\pm$1.3} & 31.4{\scriptsize$\pm$1.4} & 63.9{\scriptsize$\pm$1.0} & 58.5{\scriptsize$\pm$1.0} & 37.9{\scriptsize$\pm$0.5} & 47.0{\scriptsize$\pm$0.5} & 24.6{\scriptsize$\pm$1.9} & 36.2{\scriptsize$\pm$2.1} & 70.1{\scriptsize$\pm$1.1} & 70.5{\scriptsize$\pm$1.1} & 85.9{\scriptsize$\pm$1.1} & 78.6{\scriptsize$\pm$1.3} & 53.3{\scriptsize$\pm$1.4} & 52.9 \\
FP8 & 27.9{\scriptsize$\pm$1.3} & 28.4{\scriptsize$\pm$1.3} & 62.6{\scriptsize$\pm$1.0} & 54.2{\scriptsize$\pm$1.0} & 35.8{\scriptsize$\pm$0.5} & 43.2{\scriptsize$\pm$0.5} & 24.4{\scriptsize$\pm$1.9} & 34.8{\scriptsize$\pm$2.1} & 67.5{\scriptsize$\pm$1.1} & 68.4{\scriptsize$\pm$1.1} & 84.0{\scriptsize$\pm$1.2} & 74.7{\scriptsize$\pm$1.4} & 52.2{\scriptsize$\pm$1.4} & 50.6 \\
NVFP4 & \textbf{28.2{\scriptsize$\pm$1.3}} & 30.7{\scriptsize$\pm$1.4} & 62.6{\scriptsize$\pm$1.0} & 56.1{\scriptsize$\pm$1.0} & 36.9{\scriptsize$\pm$0.5} & 45.3{\scriptsize$\pm$0.5} & 23.8{\scriptsize$\pm$1.9} & 35.8{\scriptsize$\pm$2.1} & \textbf{69.0{\scriptsize$\pm$1.1}} & 68.0{\scriptsize$\pm$1.1} & 84.4{\scriptsize$\pm$1.1} & 78.0{\scriptsize$\pm$1.3} & 53.8{\scriptsize$\pm$1.4} & 51.7 \\
CHON & 28.0{\scriptsize$\pm$1.3} & \textbf{30.8{\scriptsize$\pm$1.4}} & \textbf{63.0{\scriptsize$\pm$1.0}} & \textbf{57.2{\scriptsize$\pm$1.0}} & \textbf{37.2{\scriptsize$\pm$0.5}} & \textbf{46.7{\scriptsize$\pm$0.5}} & \textbf{24.6{\scriptsize$\pm$1.9}} & \textbf{36.4{\scriptsize$\pm$2.1}} & 68.3{\scriptsize$\pm$1.1} & \textbf{68.4{\scriptsize$\pm$1.1}} & \textbf{86.5{\scriptsize$\pm$1.1}} & \textbf{80.7{\scriptsize$\pm$1.2}} & \textbf{54.1{\scriptsize$\pm$1.4}} & \textbf{52.5} \\[0.6ex]
\midrule

\multicolumn{15}{c}{\mygrey \textbf{Gated-DeltaNet-340M}} \\
BF16 & 17.6{\scriptsize$\pm$1.1} & 22.9{\scriptsize$\pm$1.2} & 41.1{\scriptsize$\pm$1.0} & 38.5{\scriptsize$\pm$1.0} & 27.0{\scriptsize$\pm$0.4} & 28.1{\scriptsize$\pm$0.4} & 14.0{\scriptsize$\pm$1.6} & 27.2{\scriptsize$\pm$2.0} & 57.7{\scriptsize$\pm$1.1} & 57.3{\scriptsize$\pm$1.1} & 70.4{\scriptsize$\pm$1.4} & 62.0{\scriptsize$\pm$1.5} & 50.0{\scriptsize$\pm$1.4} & 39.5 \\
FP8 & 26.6{\scriptsize$\pm$1.3} & 29.0{\scriptsize$\pm$1.3} & 60.0{\scriptsize$\pm$1.0} & 53.3{\scriptsize$\pm$1.0} & 35.0{\scriptsize$\pm$0.5} & 42.6{\scriptsize$\pm$0.5} & 24.4{\scriptsize$\pm$1.9} & 35.6{\scriptsize$\pm$2.1} & 67.7{\scriptsize$\pm$1.1} & 67.7{\scriptsize$\pm$1.1} & 84.2{\scriptsize$\pm$1.1} & 77.0{\scriptsize$\pm$1.3} & 51.8{\scriptsize$\pm$1.4} & 50.4 \\
NVFP4 & 28.6{\scriptsize$\pm$1.3} & 32.9{\scriptsize$\pm$1.4} & \textbf{64.2{\scriptsize$\pm$1.0}} & 57.2{\scriptsize$\pm$1.0} & 38.0{\scriptsize$\pm$0.5} & 48.0{\scriptsize$\pm$0.5} & 25.2{\scriptsize$\pm$1.9} & \textbf{36.8{\scriptsize$\pm$2.2}} & \textbf{69.4{\scriptsize$\pm$1.1}} & 68.7{\scriptsize$\pm$1.1} & 86.2{\scriptsize$\pm$1.1} & 81.4{\scriptsize$\pm$1.2} & \textbf{54.5{\scriptsize$\pm$1.4}} & 53.2 \\
CHON & \textbf{29.7{\scriptsize$\pm$1.3}} & \textbf{33.5{\scriptsize$\pm$1.4}} & 63.3{\scriptsize$\pm$1.0} & \textbf{57.4{\scriptsize$\pm$1.0}} & \textbf{38.2{\scriptsize$\pm$0.5}} & \textbf{48.0{\scriptsize$\pm$0.5}} & \textbf{26.0{\scriptsize$\pm$2.0}} & 36.2{\scriptsize$\pm$2.1} & 69.2{\scriptsize$\pm$1.1} & \textbf{69.2{\scriptsize$\pm$1.1}} & \textbf{87.2{\scriptsize$\pm$1.1}} & \textbf{82.1{\scriptsize$\pm$1.2}} & 53.0{\scriptsize$\pm$1.4} & \textbf{53.3} \\[0.6ex]
\midrule

\multicolumn{15}{c}{\mygrey \textbf{GSA-1B}} \\
BF16 & 18.9{\scriptsize$\pm$1.1} & 21.8{\scriptsize$\pm$1.2} & 35.1{\scriptsize$\pm$1.0} & 33.8{\scriptsize$\pm$1.0} & 26.1{\scriptsize$\pm$0.4} & 25.3{\scriptsize$\pm$0.4} & 13.8{\scriptsize$\pm$1.5} & 25.0{\scriptsize$\pm$1.9} & 55.4{\scriptsize$\pm$1.2} & 54.4{\scriptsize$\pm$1.2} & 40.9{\scriptsize$\pm$1.6} & 42.6{\scriptsize$\pm$1.6} & 49.5{\scriptsize$\pm$1.4} & 34.0 \\
FP8 & 18.9{\scriptsize$\pm$1.1} & 21.8{\scriptsize$\pm$1.2} & 35.1{\scriptsize$\pm$1.0} & 33.8{\scriptsize$\pm$1.0} & 26.1{\scriptsize$\pm$0.4} & 25.3{\scriptsize$\pm$0.4} & 13.8{\scriptsize$\pm$1.5} & \textbf{25.0{\scriptsize$\pm$1.9}} & 55.4{\scriptsize$\pm$1.2} & \textbf{54.3{\scriptsize$\pm$1.2}} & 40.9{\scriptsize$\pm$1.6} & 42.5{\scriptsize$\pm$1.6} & 49.5{\scriptsize$\pm$1.4} & 34.0 \\
NVFP4 & \textbf{19.0{\scriptsize$\pm$1.2}} & \textbf{23.6{\scriptsize$\pm$1.2}} & 36.2{\scriptsize$\pm$1.0} & 34.6{\scriptsize$\pm$1.0} & \textbf{26.5{\scriptsize$\pm$0.4}} & 26.2{\scriptsize$\pm$0.4} & 15.0{\scriptsize$\pm$1.6} & 24.6{\scriptsize$\pm$1.9} & \textbf{55.6{\scriptsize$\pm$1.2}} & 53.4{\scriptsize$\pm$1.2} & 40.8{\scriptsize$\pm$1.6} & 40.4{\scriptsize$\pm$1.6} & \textbf{51.5{\scriptsize$\pm$1.4}} & 34.4 \\
CHON & 18.2{\scriptsize$\pm$1.1} & 22.3{\scriptsize$\pm$1.2} & \textbf{37.2{\scriptsize$\pm$1.0}} & \textbf{35.7{\scriptsize$\pm$1.0}} & 25.9{\scriptsize$\pm$0.4} & \textbf{26.3{\scriptsize$\pm$0.4}} & \textbf{15.8{\scriptsize$\pm$1.6}} & 24.4{\scriptsize$\pm$1.9} & 55.5{\scriptsize$\pm$1.2} & 53.2{\scriptsize$\pm$1.2} & \textbf{45.9{\scriptsize$\pm$1.6}} & \textbf{45.8{\scriptsize$\pm$1.6}} & 49.7{\scriptsize$\pm$1.4} & \textbf{35.1} \\[0.6ex]
\midrule

\multicolumn{15}{c}{\mygrey \textbf{Qwen3-1.7B}} \\
BF16 & 20.4{\scriptsize$\pm$1.2} & 22.5{\scriptsize$\pm$1.2} & 31.8{\scriptsize$\pm$1.0} & 30.2{\scriptsize$\pm$0.9} & 25.8{\scriptsize$\pm$0.4} & 26.3{\scriptsize$\pm$0.4} & 12.2{\scriptsize$\pm$1.5} & 24.6{\scriptsize$\pm$1.9} & 54.9{\scriptsize$\pm$1.2} & 52.6{\scriptsize$\pm$1.2} & 35.4{\scriptsize$\pm$1.5} & 31.3{\scriptsize$\pm$1.5} & 49.6{\scriptsize$\pm$1.4} & 32.1 \\
FP8 & 20.2{\scriptsize$\pm$1.2} & 22.6{\scriptsize$\pm$1.2} & 40.5{\scriptsize$\pm$1.0} & 38.0{\scriptsize$\pm$1.0} & 26.4{\scriptsize$\pm$0.4} & 27.7{\scriptsize$\pm$0.4} & 14.4{\scriptsize$\pm$1.6} & 25.4{\scriptsize$\pm$1.9} & 57.0{\scriptsize$\pm$1.2} & 54.3{\scriptsize$\pm$1.2} & 49.8{\scriptsize$\pm$1.6} & 45.4{\scriptsize$\pm$1.6} & \textbf{52.3{\scriptsize$\pm$1.4}} & 36.5 \\
NVFP4 & \textbf{20.4{\scriptsize$\pm$1.2}} & \textbf{24.3{\scriptsize$\pm$1.2}} & 26.6{\scriptsize$\pm$0.9} & 26.9{\scriptsize$\pm$0.9} & 25.7{\scriptsize$\pm$0.4} & 25.4{\scriptsize$\pm$0.4} & 12.0{\scriptsize$\pm$1.5} & 24.2{\scriptsize$\pm$1.9} & 53.3{\scriptsize$\pm$1.2} & 50.0{\scriptsize$\pm$1.2} & 21.8{\scriptsize$\pm$1.3} & 24.1{\scriptsize$\pm$1.4} & 52.1{\scriptsize$\pm$1.4} & 29.7 \\
CHON & 20.1{\scriptsize$\pm$1.2} & 23.6{\scriptsize$\pm$1.2} & \textbf{45.3{\scriptsize$\pm$1.0}} & \textbf{40.5{\scriptsize$\pm$1.0}} & \textbf{27.3{\scriptsize$\pm$0.4}} & \textbf{29.4{\scriptsize$\pm$0.4}} & \textbf{14.8{\scriptsize$\pm$1.6}} & \textbf{27.8{\scriptsize$\pm$2.0}} & \textbf{57.5{\scriptsize$\pm$1.1}} & \textbf{58.3{\scriptsize$\pm$1.1}} & \textbf{56.1{\scriptsize$\pm$1.6}} & \textbf{49.5{\scriptsize$\pm$1.6}} & 50.7{\scriptsize$\pm$1.4} & \textbf{38.5} \\[0.6ex]

\bottomrule
\end{tabular}}
\vspace{-7pt}
\label{tab:comprehensive_benchmarks}
\end{table*}

\begin{proposition}[Error decomposition of a quantized product]
\label{prop:hcp_decomposition}
The quantized (low-precision, LP) product admits the following exact decomposition to approximate high-precision (HP). See Lemma.~\ref{lem:no_comp_baseline} for detailed derivation.
\begin{equation}
\label{eq:hcp_decomposition}
\underbrace{\widehat{\mathbf{W}}^\top\widehat{\mathbf{X}}}_{\text{LP term}}
= \underbrace{\mathbf{W}^\top\mathbf{X}}_{\text{HP term}}
+ \underbrace{\mathbf{W}^\top\Delta\mathbf{X} + \Delta\mathbf{W}^\top\mathbf{X}}_{\text{first-order error}}
+ \underbrace{\Delta\mathbf{W}^\top\Delta\mathbf{X}}_{\text{second-order error}} \nonumber
\end{equation}
\end{proposition}
HCP offers three dimensions for error compensation: \ding{172} \textbf{Mode} (\texttt{S/D}) determines the kernel implementation—\textit{Single(S)} fuses residual channels with one-pass GEMM, while \textit{Dual(D)} 
computes the GEMM and residual correction in separate kernels, as presented in Alg.~\ref{algo:hcp}. 
\ding{173} \textbf{Order} specifies recovery scope: \textit{second-order} (\texttt{O2}) targets only $\Delta\mathbf{W}^\top\Delta\mathbf{X}$, whereas \textit{first-order}(\texttt{O1}) incorporates $\widehat{\mathbf{W}}^\top\Delta\mathbf{X}$ or $\Delta\mathbf{W}^\top\widehat{\mathbf{X}}$; \ding{174} \textbf{Target} selects the residual source among weights (\texttt{W}), activations(\texttt{A}), or both(\texttt{B}). (Refer to App.~\ref{app:hcp_config_details} for details.)

We evaluate the quantization error under Laplace and Gaussian priors across HCP configurations (Fig.~\ref{fig:hcp_recipes}) and select \texttt{S-O2-B}: a single-kernel, second-order scheme that patches both weights and activations. Specifically, let $\mathcal{I} \subseteq \{1,\dots,D\}$ with $|\mathcal{I}|=k$ denote the channels to patch. For each channel $j$, we compute its importance score as:
\begin{equation}
\label{eq:hcp_score}
s_j = \frac{1}{k}\left\|\Delta\mathbf{X}_{j,:}\right\|_1+\frac{1}{m}\,\left\|\Delta\mathbf{W}_{:,j}\right\|_1.
\end{equation}
where the $\ell_1$-norms quantify the sensitivity of activations and weights to quantization perturbations. The \textit{hot channels} are determined by the top-$k$ highest scores: $\mathcal{I} = \mathrm{top\text{-}}k(\mathbf{s})$.

We further justify the necessity of \texttt{S-O2-B} through the error analysis in App.~\ref{app:mse_bounds}. 
In practice, Hot-channel Patch (HCP) (a.k.a. \texttt{S-O2-B}) is illustrated in Fig.~\ref{fig:hcp_overview}. 
Specifically, we add a weight patch $\Delta_w$ and an activation patch $\Delta_a$ to construct
the weight matrix $\mathbf{W}' = [\mathbf{W}; \Delta\mathbf{W}_{\mathcal{I}}; \widehat{\mathbf{W}}_{\mathcal{I}}]$
and the activation matrix $\mathbf{X}' = [\mathbf{X}; \widehat{\mathbf{X}}_{\mathcal{I}}; \Delta\mathbf{X}_{\mathcal{I}}]$,
where $\Delta\mathbf{X}_{\mathcal{I}}$ and $\Delta\mathbf{W}_{\mathcal{I}}$ denote the residuals restricted to the index set $\mathcal{I}$. 
We provide a derivation in App.~\ref{app:residual_comp_detail}, showing that this construction yields a lower bound on the quantization error. 
HCP then approximates the full-precision product as:
\begin{equation}
\label{eq:hcp_y_simple}
\mathbf{Y}^{\text{HCP}}_\mathcal{I} = \mathbf{W}_{\mathcal{I}}^\top \mathbf{X}_{\mathcal{I}}
-\Delta \mathbf{W}_{\mathcal{I}}^\top\Delta \mathbf{X}_{\mathcal{I}}
\end{equation}
This approximation yields an error bounded by the second-order term, making it significantly closer to the full-precision $\mathbf{W}^\top\mathbf{X}$ than the baseline. Because $k$ is significantly smaller than $D$, the extra computational cost scales linearly with $O(k)$, ensuring the method remains highly efficient.
Unlike PTQ methods such as DecDEC~\cite{Park2024DecDECAS} and BiLLM~\cite{Huang2024BiLLMPT}, HCP operates online during training. Compared with low-rank-plus-residual Metis~\cite{metis_fp4} and single-sided activation compensation in TetraJet-v2~\cite{tetrajet}, HCP simultaneously patches weight and activation residuals while keeping the patched matmul in low precision; our ablation study in Fig.~\ref{fig:mse_error} shows this dual-side recovery yields lower MSE than activation-only fixes. HCP also benefits from fixed nature of hot-channels. and hardware-friendly concatenated kernels as shown in Alg.~\ref{algo:hcp}.

\paragraph{\textbf{Mixed-Precision for Post-QK Operations.}} NVIDIA NVFP4 recipe~\cite{nvidia_recipe} employs a mixed-precision strategy by excluding attention, \texttt{lm\_head}, embeddings, and last 4 layers. Building on it, our analysis in Fig.~\ref{fig:mse_error} reveals that ``post-QK'' operations exhibit fat tails and wide dynamic range, making them particularly hard to quantize. Thus, we extend the mixed-precision approach by additionally preserving $W_o$ for \texttt{LA} and $W_{\text{v}}$ for \texttt{SA} in BF16.

\section{Experiments}
\label{sec:experiments}

\paragraph{Experimental Setup.}
\textbf{\textit{Infrastructure.}} We train using Flame~\cite{yang2025flame} with Fully Sharded Data Parallel (FSDP)~\cite{zhao2023pytorch_fsdp} that enables efficient parameter sharding and activation checkpointing. We employ NVFP4 kernel from TransformerEngine~\cite{transformer_engine}.
\textbf{\textit{Architecture.}}
We evaluate three \texttt{LA} architectures: Gated Linear Attention (GLA)~\cite{yang2024gated}, Gated Slot Attention (GSA)~\cite{zhang2024gsa}, and Gated DeltaNet~\cite{schlag2021linear}. All models are pretrained from scratch at scales from 340M to 7B.
\textbf{\textit{Evaluation.}} We evaluate zero-shot performance on a standard suite of commonsense reasoning and language understanding benchmarks using lm-harness-evaluation~\cite{gao2021framework}, including ARC-Challenge, ARC-Easy~\cite{clark2018think}, HellaSwag~\cite{zellers2019hellaswag}, OpenBookQA~\cite{mihaylov2018can}, PIQA~\cite{bisk2020piqa}, SciQ~\cite{welbl2017crowdsourcing}, and WinoGrande~\cite{sakaguchi2021winogrande}.

\begin{table}[t]
\centering
\caption{\textbf{Final loss and relative gap to BF16 baseline (60B tokens).} Results for GLA-1.3B are sorted by ascending loss gap.}
\label{tab:loss-gap-analysis}
\small 
\resizebox{.85\linewidth}{!}{
\begin{tabular}{lcc}
\toprule
\textbf{Configuration} & {\textbf{Final Loss}} & {\textbf{Loss Gap (\%)}} \\ 
\midrule
BF16 baseline & 2.168596 & 0.000 \\
\midrule
NVFP4 w/ HCP (full) & 2.181415 & 0.588 \\
NVFP4 w/ HCP w/o sr & 2.183048 & 0.662 \\
NVFP4 w/ HCP w/o rht & 2.184524 & 0.729 \\
NVFP4 w/ HCP w/o 2d scaling & 2.184868 & 0.745 \\
NVFP4 w/ HCP w/o sr, rht & 2.187204 & 0.851 \\
NVFP4 w/ HCP w/o last4 & 2.188535 & 0.911 \\
NVFP4~\cite{nvidia_recipe} & 2.189147 & 0.939 \\
NVFP4 w/ HCP w/o chon, rht & 2.191353 & 1.039 \\
\bottomrule
\end{tabular}}
\vspace{-8pt}
\end{table}

\paragraph{Training Details.}
Models are pretrained on a RedPajama subset~\cite{together2023redpajama}. We train GLA-1.3B on 60B tokens for Tab.~\ref{tab:loss-gap-analysis} and Fig.~\ref{fig:ablation_loss}, and on 40B tokens for Tab.~\ref{tab:comprehensive_benchmarks}.
We use the AdamW optimizer~\cite{loshchilov2017decoupled} with $\beta_1 = 0.9$, $\beta_2 = 0.95$, a peak learning rate of $3 \times 10^{-4}$, and a cosine decay schedule with 2,000-step linear warmup.
We employ a global batch size of 4M tokens and a sequence length of 4,096. Training is conducted using FSDP exclusively (without tensor or pipeline parallelism) and without gradient accumulation.
The gradient-norm clipping threshold is 1.
Hyperparameters are matched across BF16, FP8, and NVFP4 runs; see App.~\ref{app:training_recipe} for the detailed FP4 training recipe and App.~\ref{app:exp_setup} for additional experimental setups.

\paragraph{Downstream Tasks Evaluation.} These validations indicate that CHON achieves competitive quality while delivering meaningful system-level benefits.
As shown in Tab.~\ref{tab:comprehensive_benchmarks}, CHON achieves downstream scores nearly on par with BF16 baselines across GSA, GLA, Gated DeltaNet, and Qwen, indicating minimal quality degradation compared with BF16.
All \texttt{LA} variants remain competitive under NVFP4, supporting the hypothesis that their smoother activation dynamics are a favorable substrate for low-precision training. Refer to App.~\ref{app:down_steam_tasks} for more results.

\begin{table}[t]
\centering
\caption{\textbf{Sensitivity scores comparison across operators.}}
\label{tab:sensitivity_scores}
\vspace{-5pt}
\resizebox{\linewidth}{!}{%
\begin{tabular}{l ccccc ccc}
\toprule
& \multicolumn{5}{c}{Attention Operators} & \multicolumn{3}{c}{MLP Operators} \\
\cmidrule(lr){2-6} \cmidrule(lr){7-9}
Model & Attn.q & Attn.k & Attn.v & Attn.o & Attn.g & Mlp.u & Mlp.g & Mlp.d \\
\midrule
GLA-1.3B  & 0.053 & 0.057 & 0.025 & \textbf{0.058} & 0.023 & 0.034 & 0.029 & 0.012 \\
Qwen-1.7B & 0.022 & 0.025 & \textbf{0.143} & 0.020 & --- & 0.039 & 0.014 & 0.020 \\
\bottomrule
\end{tabular}
}
\vspace{-14pt}
\end{table}

\paragraph{Efficiency.} 
To ensure training efficiency, we devise fused \texttt{Triton} kernels that integrate residual compensation and concatenation into a single operation to reduce overhead. Compared to the NVFP4 baseline (see App.~\ref{app:exp_details}), CHON introduces a modest end-to-end training overhead of 12.88\% for Qwen3-1.7B and 8.6\% for GLA-1.3B. At the kernel level, as detailed in Tab.~\ref{tab:chon-breakdown}, the average overhead is only 5.27\% for post-fusion kernels and 16.15\% for pre-fusion kernels. Refer to App.~\ref{app:exp_details} for more details.

\begin{figure}[t]
    \centering
    \includegraphics[width=.99\linewidth]{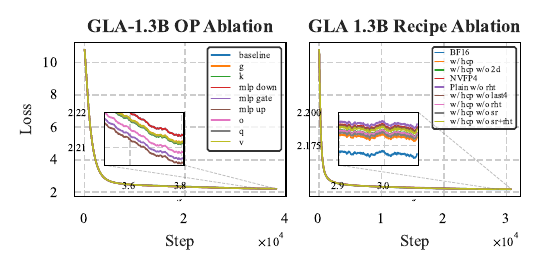}
    \vspace{-14pt}
    \caption{\textbf{Ablation of CHON recipe.} Training loss for GLA-1.3B (60B tokens). Integrating HCP, SR, and RHT is essential for minimizing the NVFP4-to-BF16 performance gap.}
    \vspace{-13pt}
    \label{fig:ablation_loss}
\end{figure}

\paragraph{Ablation Study.} 
We conduct ablation study for recipe and operator under same settings.
\textbf{(1) Recipe:} Fig.~\ref{fig:ablation_loss} (right) shows that removing HCP/SR/RHT or scaling worsens FP4 training; HCP gives the largest gain, reducing the NVFP4-to-BF16 loss gap from \textbf{0.939\%} to \textbf{0.588\%} (Tab.~\ref{tab:loss-gap-analysis}).
\textbf{(2) Operator:} While \texttt{mlp\_up}/\texttt{mlp\_gate} appear dominant in Fig.~\ref{fig:ablation_loss} (left) due to parameter size, the parameter-normalized sensitivity score identifies \texttt{o\_proj} (GLA) and \texttt{v\_proj} (Qwen) as most quantization-sensitive (Tab.~\ref{tab:sensitivity_scores}).

\paragraph{Post-training.} As additional evidence beyond pretraining, we report NVFP4 in two regimes:  
\textbf{\ding{172} Supervised Fine-tuning (SFT).} We fine-tune Qwen3-8B~\cite{yang2025qwen3} on a chain-of-thought subset of OpenMathReasoning~\cite{openmathreasoning} (12.67B tokens) using 8 Blackwell GPUs. Across training dynamics (loss/gradient norms) and downstream evaluation on AIME24/25, NVFP4 achieves performance comparable to BF16 with stable convergence. Refer to App.~\ref{app:nvfp4_sft} for full setup and results.
\textbf{\ding{173} Reinforcement Learning (RL).} We conduct GRPO~\cite{shao2024GRPO} training on GSM8K~\cite{cobbe2021training} using Qwen3-8B. We summarize five training/rollout quantization setups in Tab.~\ref{tab:exp_configs}. 
Overall, NVFP4 training paired with FP8 rollouts and truncated importance sampling (TIS)~\cite{yao2025offpolicy} attains performance within ~2\% of the full-precision BF16 training+rollout baseline. Refer to App.~\ref{app:nvfp4_rl} for details.

\section{Conclusion}
\label{sec:conclusion}

We analyze the emergence and dynamic of outliers in NVFP4 pretraining. We find that \texttt{Linear Attention} exhibits milder outliers than \texttt{Softmax Attention}, with outliers converging to static ``hot channels''. We also observe that ``post-QK'' operations are more sensitive to quantization. We introduce the Hot-Channel Patch (HCP) and CHON recipe to address this, enabling robust NVFP4 training. Experiments show minimal degradation (0.58\%) versus BF16 baseline for GLA-1.3B.

\newpage
\section*{Impact Statement}

This paper presents the first systematic analysis of outlier dynamics in NVFP4 LLM pretraining, characterizing their emergence, evolution, and sensitivity to quantization. Applications: Our findings provide a blueprint for reliable FP4 pretraining and inform the design of next-generation hardware-aware model architectures, transforming outlier mitigation from a heuristic trial-and-error process into a targeted diagnostic science. Implications: By enabling the CHON NVFP4 recipe, we are able to reduce the computational cost and energy footprint of training large-scale models while maintaining BF16-level accuracy. This efficiency gain helps democratize access to training larger models and longer-context capabilities. Future Initiatives: We advocate for a collaborative community effort to establish standardized outlier diagnostics specifically for low-precision pretraining. Furthermore, we urge that next-generation model architectures incorporate quantization-aware design principles, and that the community pursue architecture-kernel co-design to make ultra-low-precision training a predictable and widely accessible standard.

\section*{Limitation}

Our analysis is validated on a limited set of model sizes and architectures; additional experiments on ultra-large-scale and MoE models are needed to confirm the generality of the observed outlier patterns. We primarily use practical diagnostics (e.g., kurtosis, top-$k$ magnitude, and quantization error) to characterize numerical behavior; incorporating a broader set of metrics could provide a more complete view of stability and failure modes under NVFP4. In addition, this work focuses on numerical and algorithmic understanding rather than end-to-end NVFP4 training system efficiency, so production throughput gains will also depend on complementary systems and kernel optimizations. Finally, our conclusions are largely based on empirical correlations observed during training; establishing stronger causal evidence for specific outlier sources remains an important direction for future work.

\bibliography{main}
\bibliographystyle{icml2026}

\newpage
\appendix
\onecolumn

\begin{center}
\Large\textbf{Appendix Contents}    
\end{center}
\vspace{0.3em}

\begin{description}[leftmargin=2em, itemsep=0.3em]
    \item[\textbf{A}] \textbf{Theoretical analysis of Hot-Channel Patch} \dotfill \pageref{app:residual_comp_detail}
    \begin{description}
        \item[\textbf{A.1}] \textbf{Hot-channel Formalization} \dotfill \pageref{app:hot_channel_formalization}
        \item[\textbf{A.2}] \textbf{Candidate Compensation Configurations.} \dotfill \pageref{app:candidate_compensation_config}
        \item[\textbf{A.3}] \textbf{Mean-squared Error Bounds.} \dotfill \pageref{app:mse_bounds}
    \end{description}

    \item[\textbf{B}] \textbf{Ablation Study Details} \dotfill \pageref{app:ablation}
    \begin{description}[leftmargin=3em, itemsep=0.2em]
        \item[\textbf{B.1}] \textbf{Details of HCP Configurations} \dotfill \pageref{app:hcp_config_details}
        \item[\textbf{B.2}] \textbf{Post-QK Operations Ablation} \dotfill \pageref{app:post_qk_operation}
    \end{description}

    \item[\textbf{C}] \textbf{Additional Experimental Details} \dotfill \pageref{app:exp_details}
    \begin{description}
        \item[\textbf{C.1}] \textbf{Experimental Setup} \dotfill \pageref{app:exp_setup}
        \item[\textbf{C.2}] \textbf{Efficiency Details} \dotfill \pageref{app:efficiency}
        \item[\textbf{C.3}] \textbf{CHON Recipe} \dotfill \pageref{app:training_recipe}
        \item[\textbf{C.4}] \textbf{NVFP4 Details} \dotfill \pageref{app:nvfp4_details}

    \end{description}

    \item[\textbf{D}] \textbf{Additional Experiments} \dotfill \pageref{app:addition_exp}
    \begin{description}[leftmargin=3em, itemsep=0.2em]
        \item[\textbf{D.1}] \textbf{Supervised Fine-tuning with NVFP4} \dotfill \pageref{app:nvfp4_sft}
        \item[\textbf{D.2}] \textbf{Reinforcement Learning with NVFP4} \dotfill \pageref{app:nvfp4_rl}
        \item[\textbf{D.3}] \textbf{Additional Downstream Tasks Performance} \dotfill \pageref{app:down_steam_tasks}
    \end{description}

    \item[\textbf{E}] \textbf{Additional analysis} \dotfill \pageref{app:analysis}
    \begin{description}[leftmargin=3em, itemsep=0.2em]
        \item[\textbf{E.1}] \textbf{Kurtosis Analysis} \dotfill \pageref{app:kurtosis}
        \item[\textbf{E.2}] \textbf{Outlier Evolution Landscape} \dotfill \pageref{app:outlier_landscape}
        \item[\textbf{E.3}] \textbf{Top-$k$ Magnitude Analysis during Training} \dotfill \pageref{app:topk_mag}
        \item[\textbf{E.4}] \textbf{Component Dynamic Analysis} \dotfill \pageref{app:comp_dynamic}
        \item[\textbf{E.5}] \textbf{Frobenius Energy Analysis} \dotfill \pageref{app:frobenius_energy}
        \item[\textbf{E.6}] \textbf{Flush-to-Zero Dynamic Analysis} \dotfill \pageref{app:ftz_dynamic}
        \item[\textbf{E.7}] \textbf{Gating as Outlier Source} \dotfill \pageref{app:gk_outlier_analysis}
        \item[\textbf{E.8}] \textbf{Normalization Impact Analysis} \dotfill \pageref{app:norm_impact}
        \item[\textbf{E.9}] \textbf{Superposition Analysis} \dotfill \pageref{app:superposition}
        \item[\textbf{E.10}] \textbf{Quantization Error Dynamics} \dotfill \pageref{app:quant_error_dynamics}
    \end{description}

    \item[\textbf{F}] \textbf{Discussions} \dotfill \pageref{app:discuss}
\end{description}

\section{Theoretical Analysis of Hot-Channel Patch}
\label{app:residual_comp_detail}

Let $\mathbf{X} \in \mathbb{R}^{n \times k}$ and $\mathbf{W}^T \in \mathbb{R}^{m \times n}$ be the FP32 activations and weights.  We have
$\mathbf{Y}:=\mathbf{W}^T\mathbf{X} \in \mathbb{R}^{m\times k}$.
Let $Q(\cdot)$ and $D(\cdot)$ denote the quantization and dequantization operators, respectively. The low-precision (LP) approximations are defined as $\widehat{\mathbf{X}} := D(Q(\mathbf{X}))$ and $\widehat{\mathbf{W}} := D(Q(\mathbf{W}))$. The corresponding residual matrices (quantization errors) are defined as:
\begin{equation}
    \Delta\mathbf{X} = \mathbf{X} - \widehat{\mathbf{X}}, \quad \Delta\mathbf{W} = \mathbf{W} - \widehat{\mathbf{W}}.
\end{equation}

\subsection{Hot-channel formalization}
\label{app:hot_channel_formalization}

\begin{definition}[Hot-channel index set]
To control computational overhead, we select a set of channel indices
\begin{equation}
\mathcal{I} \subset \{1,\dots,n\},\qquad |\mathcal{I}|=k,
\end{equation}
based on a compounded error score:
\begin{equation}
    s_j = \mathbb{E} \big[ \|\Delta\mathbf{W}_{\cdot j}\| \big] + \mathbb{E} \big[ \|\Delta\mathbf{X}_{\cdot j}\| \big],
\end{equation}
where $\mathcal{I}$ contains the indices of the top-$k$ scores.

\smallskip
\noindent\textbf{Context.}
Motivated by the longitudinal observation that transient drifting outliers gradually stabilize into a small set of
persistent \emph{hot channels}, we treat $\mathcal{I}$ as fixed in later stages for multiple steps/intervals to avoid per-step
reselection overhead, consistent with online patching in HCP/CHON.
\end{definition}

\begin{definition}[Selection operator and residual patch on $\mathcal{I}$]
Let $P_{\mathcal{I}}$ be the channel-selection operator:
for $\mathbf{W}\in\mathbb R^{m\times n}$ it keeps columns in $\mathcal{I}$ and sets the rest to zero;
for $\mathbf{X}\in\mathbb R^{n\times t}$ it keeps rows in $\mathcal{I}$ and sets the rest to zero.
The hot-channel residual patches (quantization error for hot channels) are
\begin{equation}
    \Delta{\mathbf{W}}_{\mathcal{I}} := P_{\mathcal{I}}(\mathbf{W} - \widehat{\mathbf{W}}),\qquad
    \Delta{\mathbf{X}}_{\mathcal{I}} := P_{\mathcal{I}}(\mathbf{X}-\widehat{\mathbf{X}}).
\end{equation}
Then we have
\begin{equation}
\widehat{\mathbf{W}}_{\mathcal{I}}=\mathbf{W}_{\mathcal{I}}+\Delta \mathbf{W}_{\mathcal{I}},\qquad \widehat{\mathbf{X}}_{\mathcal{I}}=\mathbf{X}_{\mathcal{I}}+\Delta{\mathbf{X}}_{\mathcal{I}}.
\end{equation}
\end{definition}

\subsection{Candidate Compensation Configurations}
\label{app:candidate_compensation_config}

In the design of HCP, we consider three representative compensation configurations and explicitly derive their resulting quantization errors: (1) \emph{no compensation} (baseline), (2) \emph{first-order activation compensation}, and (3) \emph{second-order weight–activation compensation}.

\begin{lemma}[No compensation (baseline)]
\label{lem:no_comp_baseline}
With no compensation, the decoded low-precision output is
\begin{equation}
\widehat{\mathbf{Y}}^{(0)}_{\mathcal{I}} = \widehat{\mathbf{W}}_{\mathcal{I}}^\top \widehat{\mathbf{X}}_{\mathcal{I}}.
\end{equation}
Using the additive residual notation
$\widehat{\mathbf{W}}_{\mathcal{I}}=\mathbf{W}_{\mathcal{I}}+\Delta\mathbf{W}_{\mathcal{I}}$ and
$\widehat{\mathbf{X}}_{\mathcal{I}}=\mathbf{X}_{\mathcal{I}}+\Delta\mathbf{X}_{\mathcal{I}}$,
we obtain the exact expansion
\begin{equation}
\widehat{\mathbf{Y}}^{(0)}_{\mathcal{I}}
= \mathbf{W}_{\mathcal{I}}^\top \mathbf{X}_{\mathcal{I}}
+ \mathbf{W}_{\mathcal{I}}^\top \Delta \mathbf{X}_{\mathcal{I}}
+ \Delta \mathbf{W}_{\mathcal{I}}^\top \mathbf{X}_{\mathcal{I}}
+ \Delta \mathbf{W}_{\mathcal{I}}^\top \Delta \mathbf{X}_{\mathcal{I}}.
\end{equation}
\end{lemma}

\begin{proof}
\begin{equation}
\begin{aligned}
\widehat{\mathbf{W}}_{\mathcal{I}}^\top \widehat{\mathbf{X}}_{\mathcal{I}}
&= (\mathbf{W}_{\mathcal{I}}+\Delta \mathbf{W}_{\mathcal{I}})^\top
   (\mathbf{X}_{\mathcal{I}}+\Delta \mathbf{X}_{\mathcal{I}}) \\
&= (\mathbf{W}_{\mathcal{I}}^\top+\Delta \mathbf{W}_{\mathcal{I}}^\top)
   (\mathbf{X}_{\mathcal{I}}+\Delta \mathbf{X}_{\mathcal{I}}) \\
&= \mathbf{W}_{\mathcal{I}}^\top \mathbf{X}_{\mathcal{I}}
 + \mathbf{W}_{\mathcal{I}}^\top \Delta \mathbf{X}_{\mathcal{I}}
 + \Delta \mathbf{W}_{\mathcal{I}}^\top \mathbf{X}_{\mathcal{I}}
 + \Delta \mathbf{W}_{\mathcal{I}}^\top \Delta \mathbf{X}_{\mathcal{I}} .
\end{aligned}
\end{equation}
\end{proof}

\begin{lemma}[First-order activation compensation]
Consider the first-order single-sided compensation strategy, which focuses on addressing activation residuals:
\begin{equation}
\label{lemma:1st_order}
\underbrace{\widehat{\mathbf{W}}_{\mathcal{I}}^\top \widehat{\mathbf{X}_{\mathcal{I}}}}_{\text{LP term}}
\;+\;
\underbrace{\widehat{\mathbf{W}}_{\mathcal{I}}^\top\bigl(\mathbf{X}_{\mathcal{I}}-\widehat{\mathbf{X}_{\mathcal{I}}}}_{\text{activation patch}}\bigr).
\tag{S0}
\end{equation}
Then
\begin{equation}
\widehat{\mathbf{W}}_{\mathcal{I}}^\top \widehat{\mathbf{X}_{\mathcal{I}}}
+\widehat{\mathbf{W}}_{\mathcal{I}}^\top(\mathbf{X}_{\mathcal{I}}-\widehat{\mathbf{X}_{\mathcal{I}}})
=
\mathbf{W}_{\mathcal{I}}^\top \mathbf{X}_{\mathcal{I}}
+\Delta{\mathbf{W}}_{\mathcal{I}}^\top \mathbf{X}_{\mathcal{I}}.
\tag{S5}
\end{equation}
\end{lemma}

\begin{proof}
\textbf{Step 1 (compact form; exact algebra).}
\begin{equation}
\widehat{\mathbf{W}}_{\mathcal{I}}^\top(\mathbf{X}_{\mathcal{I}}-\widehat{\mathbf{X}_{\mathcal{I}}})
=
\widehat{\mathbf{W}}_{\mathcal{I}}^\top \mathbf{X}_{\mathcal{I}}
-\widehat{\mathbf{W}}_{\mathcal{I}}^\top\widehat{\mathbf{X}_{\mathcal{I}}}.
\end{equation}
Substitute into (S0):
\begin{equation}
\widehat{\mathbf{W}}_{\mathcal{I}}^\top\widehat{\mathbf{X}_{\mathcal{I}}}
+\widehat{\mathbf{W}}_{\mathcal{I}}^\top \mathbf{X}_{\mathcal{I}}
-\widehat{\mathbf{W}}_{\mathcal{I}}^\top\widehat{\mathbf{X}_{\mathcal{I}}}
=
\widehat{\mathbf{W}}_{\mathcal{I}}^\top \mathbf{X}_{\mathcal{I}}.
\tag{S1}
\end{equation}

\textbf{Step 2 (introduce additive error).}
Using $\widehat{\mathbf{W}}_{\mathcal{I}}=\mathbf{W}_{\mathcal{I}}+\Delta{\mathbf{W}}_{\mathcal{I}}$,
\begin{equation}
\widehat{\mathbf{W}}_{\mathcal{I}}^\top \mathbf{X}_{\mathcal{I}}
=
(\mathbf{W}_{\mathcal{I}}+\Delta{\mathbf{W}}_{\mathcal{I}})^\top \mathbf{X}_{\mathcal{I}}.
\tag{S3}
\end{equation}

\textbf{Step 3 (full expansion; no omission).}
\begin{equation}
\begin{aligned}
(\mathbf{W}_{\mathcal{I}}+\Delta{\mathbf{W}}_{\mathcal{I}})^\top \mathbf{X}_{\mathcal{I}}
&=(\mathbf{W}_{\mathcal{I}}^\top+\Delta{\mathbf{W}}_{\mathcal{I}}^\top)\mathbf{X}_{\mathcal{I}}\\
&=\mathbf{W}_{\mathcal{I}}^\top \mathbf{X}_{\mathcal{I}}+\Delta{\mathbf{W}}_{\mathcal{I}}^\top \mathbf{X}_{\mathcal{I}}.
\end{aligned}
\end{equation}
\end{proof}

\begin{lemma}[Second-order weight–activation compensation]
Consider the second-order weight-activation compensation starting expression, which explicitly addresses both weight and activation residuals:
\begin{equation}\label{lemma:second-order}
\underbrace{\widehat{\mathbf{W}}_{\mathcal{I}}^\top \widehat{\mathbf{X}}_{\mathcal{I}}}_{\text{LP term}}
+
\underbrace{\widehat{\mathbf{W}}_{\mathcal{I}}^\top (\mathbf{X}_\mathcal{I}-\widehat{\mathbf{X}}_{\mathcal{I}})}_{\text{activation patch}}
+
\underbrace{(\mathbf{W}_{\mathcal{I}}-\widehat{\mathbf{W}}_{\mathcal{I}})^\top \widehat{\mathbf{X}}_{\mathcal{I}}}_{\text{weight patch}}.
\tag{D0}
\end{equation}
Then
\begin{equation}
\widehat{\mathbf{W}}_{\mathcal{I}}^\top \widehat{\mathbf{X}}_{\mathcal{I}}
+\widehat{\mathbf{W}}_{\mathcal{I}}^\top (\mathbf{X}_{\mathcal{I}}-\widehat{\mathbf{X}}_{\mathcal{I}})
+(\mathbf{W}_{\mathcal{I}}-\widehat{\mathbf{W}}_{\mathcal{I}})^\top \widehat{\mathbf{X}}_{\mathcal{I}}
=
\mathbf{W}_{\mathcal{I}}^\top \mathbf{X}_{\mathcal{I}}-\Delta{\mathbf{W}}_{\mathcal{I}}^\top\Delta{\mathbf{X}}_{\mathcal{I}}.
\tag{D3}
\end{equation}
\end{lemma}

\begin{proof}
\textbf{Step 1 (compact form + additive errors).}
Using $\widehat{\mathbf{W}}_{\mathcal{I}}=\mathbf{W}_{\mathcal{I}}+\Delta \mathbf{W}_{\mathcal{I}}$ and
$\widehat{\mathbf{X}}_{\mathcal{I}}=\mathbf{X}_{\mathcal{I}}+\Delta \mathbf{X}_{\mathcal{I}}$, we first rewrite the left-hand side as
\begin{equation}
\begin{aligned}
&\widehat{\mathbf{W}}_{\mathcal{I}}^\top\widehat{\mathbf{X}}_{\mathcal{I}}
+\widehat{\mathbf{W}}_{\mathcal{I}}^\top(\mathbf{X}_{\mathcal{I}}-\widehat{\mathbf{X}}_{\mathcal{I}})
+(\mathbf{W}_{\mathcal{I}}-\widehat{\mathbf{W}}_{\mathcal{I}})^\top\widehat{\mathbf{X}}_{\mathcal{I}} \\
&=\widehat{\mathbf{W}}_{\mathcal{I}}^\top \mathbf{X}_{\mathcal{I}}
+\mathbf{W}_{\mathcal{I}}^\top\widehat{\mathbf{X}}_{\mathcal{I}}
-\widehat{\mathbf{W}}_{\mathcal{I}}^\top\widehat{\mathbf{X}_{\mathcal{I}}} \\
&=(\mathbf{W}_{\mathcal{I}}+\Delta \mathbf{W}_{\mathcal{I}})^\top \mathbf{X}_{\mathcal{I}}
+\mathbf{W}_{\mathcal{I}}^\top(\mathbf{X}_{\mathcal{I}}+\Delta \mathbf{X}_{\mathcal{I}})
-(\mathbf{W}_{\mathcal{I}}+\Delta \mathbf{W}_{\mathcal{I}})^\top(\mathbf{X}_{\mathcal{I}}+\Delta \mathbf{X}_{\mathcal{I}}).
\end{aligned}
\tag{D1'}
\end{equation}

\textbf{Step 2 (full expansion; no omission).}
\begin{equation}
\begin{aligned}
&(\mathbf{W}_{\mathcal{I}}+\Delta \mathbf{W}_{\mathcal{I}})^\top \mathbf{X}_{\mathcal{I}}
+\mathbf{W}_{\mathcal{I}}^\top(\mathbf{X}_{\mathcal{I}}+\Delta \mathbf{X}_{\mathcal{I}})
-(\mathbf{W}_{\mathcal{I}}+\Delta \mathbf{W}_{\mathcal{I}})^\top(\mathbf{X}_{\mathcal{I}}+\Delta \mathbf{X}_{\mathcal{I}})\\
&=\mathbf{W}_{\mathcal{I}}^\top \mathbf{X}_{\mathcal{I}}
-\Delta \mathbf{W}_{\mathcal{I}}^\top\Delta \mathbf{X}_{\mathcal{I}}.
\end{aligned}
\end{equation}

Since the patch is restricted to hot channels, $\Delta \mathbf{W}$ and $\Delta \mathbf{X}$ are nonzero only on $\mathcal{I}$, hence the residual term is
$-\Delta \mathbf{W}_{\mathcal{I}}^\top\Delta \mathbf{X}_{\mathcal{I}}$.
\end{proof}



\subsection{Mean-squared Error Bounds}
\label{app:mse_bounds}

\begin{definition}[Mean-squared error (MSE)]
For any estimator $\widehat{\mathbf{Y}}_{\mathcal{I}}$ of $\mathbf{Y}_{\mathcal{I}}=\mathbf{W}_{\mathcal{I}}^\top \mathbf{X}_{\mathcal{I}}$, define
\begin{equation}
\MSE := \mathbb E\bigl[(\mathbf{Y}_{\mathcal{I}}-\widehat{\mathbf{Y}}_{\mathcal{I}})^2\bigr],
\end{equation}
where the expectation can be over data, quantization noise, or both.
\end{definition}

\begin{lemma}[MSE upper bound: Baseline (no compensation)]
\label{lemma:mse_upper_bound}
Let $e_0 := \mathbf{Y}_{\mathcal{I}}-\widehat{\mathbf{Y}}^{(0)}_{\mathcal{I}}$. Then
\begin{equation}
e_0 = -\bigl(\mathbf{W}_{\mathcal{I}}^\top\Delta{\mathbf{X}_{\mathcal{I}}}+\Delta \mathbf{W}_{\mathcal{I}}^\top \mathbf{X}_{\mathcal{I}}+\Delta \mathbf{W}_{\mathcal{I}}^\top\Delta{\mathbf{X}_{\mathcal{I}}}\bigr),
\end{equation}
and
\begin{equation}
\MSE_0
\le
3\,\mathbb E\!\left[\|\mathbf{W}_{\mathcal{I}}\|^2\|\Delta{\mathbf{X}_{\mathcal{I}}}\|^2+\|\mathbf{X}_{\mathcal{I}}\|^2\|\Delta \mathbf{W}_{\mathcal{I}}\|^2+\|\Delta \mathbf{W}_{\mathcal{I}}\|^2\|\Delta{\mathbf{X}_{\mathcal{I}}}\|^2\right].
\tag{B0}
\end{equation}
\end{lemma}

\begin{proof}
Let $a=\mathbf{W}_{\mathcal{I}}^\top\Delta{\mathbf{X}_{\mathcal{I}}}$, $b=\Delta \mathbf{W}_{\mathcal{I}}^\top \mathbf{X}_{\mathcal{I}}$, $c=\Delta \mathbf{W}_{\mathcal{I}}^\top\Delta{\mathbf{X}_{\mathcal{I}}}$, so $e_0=-(a+b+c)$.
Using $(u+v+w)^2\le 3(u^2+v^2+w^2)$,
\begin{equation}
\mathbb E[e_0^2]\le 3(\mathbb E[a^2]+\mathbb E[b^2]+\mathbb E[c^2]).
\end{equation}
By Cauchy--Schwarz,
\begin{equation}
a^2\le \|\mathbf{W}_{\mathcal{I}}\|^2\|\Delta{\mathbf{X}_{\mathcal{I}}}\|^2,\quad
b^2\le \|\Delta \mathbf{W}_{\mathcal{I}}\|^2\|\mathbf{X}_{\mathcal{I}}\|^2,\quad
c^2\le \|\Delta \mathbf{W}_{\mathcal{I}}\|^2\|\Delta{\mathbf{X}_{\mathcal{I}}}\|^2.
\end{equation}
Taking expectations yields (B0).
\end{proof}

\begin{lemma}[MSE upper bound: 1st-Order]
Let $e_1 := \mathbf{Y}_{\mathcal{I}}-\widehat{\mathbf{Y}}^{(1)}_{\mathcal{I}}$. Then
\begin{equation}
e_1 = -\Delta{\mathbf{W}}_{\mathcal{I}}^\top \mathbf{X}_{\mathcal{I}},
\end{equation}
and
\begin{equation}
\MSE_1
\le
\mathbb E\!\left[\|\Delta{\mathbf{W}}_{\mathcal{I}}\|^2\|\mathbf{X}_{\mathcal{I}}\|^2\right].
\tag{B1}
\end{equation}
\end{lemma}

\begin{proof}
From Lemma~2,
$\widehat{\mathbf{Y}}^{(1)}_{\mathcal{I}}=\mathbf{W}_{\mathcal{I}}^\top \mathbf{X}_{\mathcal{I}}+\Delta{\mathbf{W}}_{\mathcal{I}}^\top \mathbf{X}_{\mathcal{I}}$,
hence $e_1=-\Delta{\mathbf{W}}_{\mathcal{I}}^\top \mathbf{X}_{\mathcal{I}}$.
By Cauchy--Schwarz,
\begin{equation}
e_1^2=(\Delta{\mathbf{W}}_{\mathcal{I}}^\top \mathbf{X}_{\mathcal{I}})^2
\le \|\Delta{\mathbf{W}}_{\mathcal{I}}\|^2\|\mathbf{X}_{\mathcal{I}}\|^2.
\end{equation}
Taking expectations gives (B1).
\end{proof}

\begin{lemma}[MSE upper bound: 2nd-Order]
Let $e_2 := \mathbf{Y}_{\mathcal{I}}-\widehat{\mathbf{Y}}^{(2)}_{\mathcal{I}}$. Then
\begin{equation}
e_2=\Delta{\mathbf{W}}_{\mathcal{I}}^\top\Delta{\mathbf{X}}_{\mathcal{I}},
\end{equation}
and
\begin{equation}
\MSE_2
\le
\mathbb E\!\left[\|\Delta{\mathbf{W}}_{\mathcal{I}}\|^2\|\Delta{\mathbf{X}}_{\mathcal{I}}\|^2\right].
\tag{B2}
\end{equation}
\end{lemma}

\begin{proof}
From Lemma~3,
$\widehat{\mathbf{Y}}^{(2)}_{\mathcal{I}}=\mathbf{Y}_{\mathcal{I}}-\Delta{\mathbf{W}}_{\mathcal{I}}^\top\Delta{\mathbf{X}}_{\mathcal{I}}$,
so $e_2=\Delta{\mathbf{W}}_{\mathcal{I}}^\top\Delta{\mathbf{X}}_{\mathcal{I}}$.
By Cauchy--Schwarz,
\begin{equation}
e_2^2=(\Delta{\mathbf{W}}_{\mathcal{I}}^\top\Delta{\mathbf{X}}_{\mathcal{I}})^2
\le \|\Delta{\mathbf{W}}_{\mathcal{I}}\|^2\|\Delta{\mathbf{X}}_{\mathcal{I}}\|^2.
\end{equation}
Taking expectations yields (B2).
\end{proof}


\begin{assumption}[Common quantization regime: signal dominates noise on $\mathcal{I}$]
On hot channels, quantization error is small compared with signal magnitude:
\begin{equation}
\|\Delta{\mathbf{X}}_{\mathcal{I}}\|\ll \|\mathbf{X}_{\mathcal{I}}\|.
\end{equation}
\end{assumption}

\begin{lemma}[Typical dominance of 2nd-Order over 1st-Order under small-noise regime]
Under the above regime, the 2nd-Order residual term typically has a much smaller scale than the 1st-Order
single-sided residual term, since it replaces multiplication by $\mathbf{X}_{\mathcal{I}}$ with multiplication by $\Delta{\mathbf{X}}_{\mathcal{I}}$:
\begin{equation}
e_2^2 \ \text{is typically much smaller than}\ e_1^2.
\end{equation}
\end{lemma}

\begin{proof}
Pointwise bounds give
\begin{equation}
e_1^2\le \|\Delta{\mathbf{W}}_{\mathcal{I}}\|^2\|\mathbf{X}_{\mathcal{I}}\|^2,\qquad
e_2^2\le \|\Delta{\mathbf{W}}_{\mathcal{I}}\|^2\|\Delta{\mathbf{X}}_{\mathcal{I}}\|^2.
\end{equation}
If $\|\Delta{\mathbf{X}}_{\mathcal{I}}\|\ll \|\mathbf{X}_{\mathcal{I}}\|$, then the second right-hand side is much smaller than the first, implying the claimed typical scale separation.
\end{proof}

\begin{theorem}[Total error order]
\label{thm:total_error_order}
Assume the small-noise regime on hot channels (Assumption~1):
$\|\Delta \mathbf{X}_{\mathcal I}\|\ll \|\mathbf{X}_{\mathcal I}\|$.
Further assume zero-mean residuals and that the three error components
\[
a:=\mathbf{W}_{\mathcal I}^\top \Delta \mathbf{X}_{\mathcal I},\quad
b:=\Delta \mathbf{W}_{\mathcal I}^\top \mathbf{X}_{\mathcal I},\quad
c:=\Delta \mathbf{W}_{\mathcal I}^\top \Delta \mathbf{X}_{\mathcal I}
\]
are pairwise uncorrelated:
\begin{equation}
\mathbb{E}\langle a,b\rangle_F=
\mathbb{E}\langle a,c\rangle_F=
\mathbb{E}\langle b,c\rangle_F=0.
\label{eq:uncorr_abc}
\end{equation}
Then the MSEs satisfy
\begin{equation}
\MSE_2 \ll \MSE_1 < \MSE_0.
\end{equation}
\end{theorem}

\begin{proof}
Using $e_0=-(a+b+c)$, $e_1=-b$, $e_2=c$ and \eqref{eq:uncorr_abc},
\[
\MSE_0=\mathbb{E}\|a+b+c\|_F^2
=\mathbb{E}\|a\|_F^2+\MSE_1+\MSE_2
>\MSE_1,
\]
hence $\MSE_1<\MSE_0$.
Moreover,
\[
\MSE_2=\mathbb{E}\|c\|_F^2
\le \mathbb{E}\!\left[\|\Delta \mathbf{W}_{\mathcal I}\|_F^2\ \|\Delta \mathbf{X}_{\mathcal I}\|_F^2\right]
\ll
\mathbb{E}\!\left[\|\Delta \mathbf{W}_{\mathcal I}\|_F^2\ \|\mathbf{X}_{\mathcal I}\|_F^2\right]
\ge \MSE_1,
\]
where ``$\ll$'' follows from $\|\Delta \mathbf{X}_{\mathcal I}\|\ll \|\mathbf{X}_{\mathcal I}\|$.
\end{proof}

\section{Ablation Study Details}
\label{app:ablation}

\begin{figure*}[t]
    \centering
    \begin{minipage}{0.9\linewidth}
        \centering
        \includegraphics[width=\linewidth]{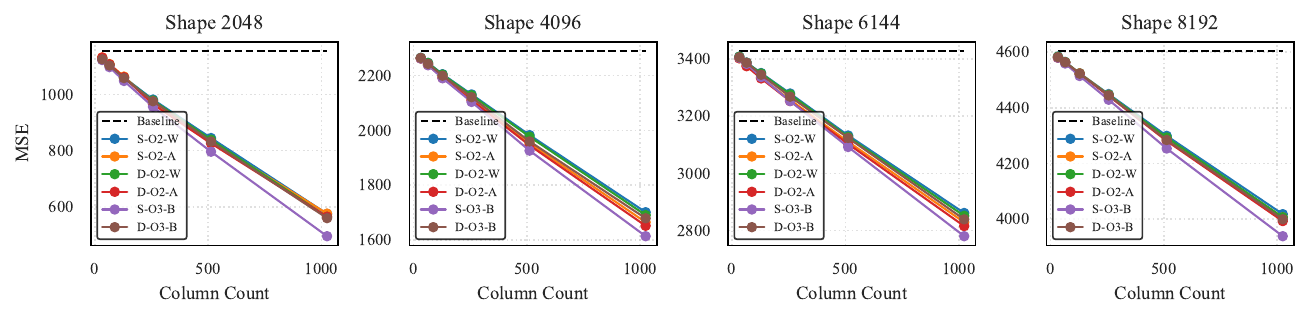}
        \vspace{-15pt}
        \caption*{{\small (a) Gaussian Activation Prior}}
    \end{minipage}
    \vspace{1em}
    \begin{minipage}{0.9\linewidth}
        \centering
        \includegraphics[width=\linewidth]{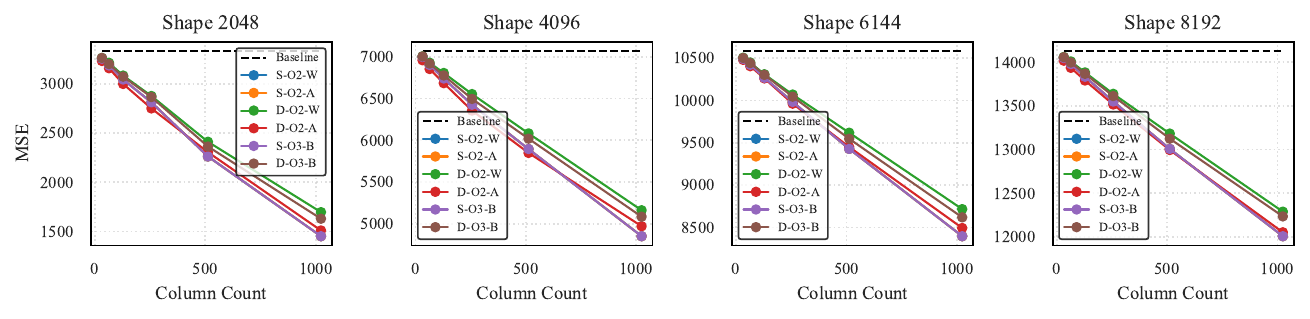}
        \vspace{-15pt}
        \caption*{{\small (b) Laplace Activation Prior}}
    \end{minipage}
    \vspace{-8pt}
    \caption{\textbf{Comparison of HCP configurations under different activation priors.} Mean-squared error (MSE) of the quantized linear product versus the number of patched (hot) channels for four hidden sizes (2,048/4,096/6,144/8,192), evaluated under (a) Gaussian and (b) Laplace activation priors. Curves compare Single- vs Dual-kernel implementations (S/D), second-order vs first-order–inclusive recovery (O1/O2), and patch targets (weights-only W, activations-only A, or both B); the dashed line denotes the unpatched baseline. Consistent with the paper’s CHON design choice, \textbf{S–O2–B} (single-kernel, first-order–inclusive, both) yields the lowest MSE across shapes and priors.}
    \label{fig:chon_recipe_ablation}
\end{figure*}

\subsection{HCP Configuration Details}
\label{app:hcp_config_details}

The Hot-Channel Patch (HCP) mechanism constitutes a modular framework for mitigating quantization errors in FP4 matrix multiplications. As derived in Proposition~\ref{prop:hcp_decomposition}, the quantized product decomposes into a low-precision base term $\widehat{\mathbf{W}}^\top\widehat{\mathbf{X}}$, first-order residual terms $\widehat{\mathbf{W}}^\top\Delta\mathbf{X} + \Delta\mathbf{W}^\top\widehat{\mathbf{X}}$, and a second-order interaction term $\Delta\mathbf{W}^\top\Delta\mathbf{X}$. HCP selectively recovers these error components over a sparse set of high-impact channels, optimizing the trade-off between computational overhead and numerical fidelity across three orthogonal configuration axes.

\paragraph{Configuration Dimension 1: Mode (Execution Strategy).}

The \emph{Mode} parameter defines the computational implementation of the residual compensation:

\begin{itemize}[leftmargin=*, label=$\bullet$, itemsep=2pt]
    \item \textbf{Single-Kernel Mode (S):} This strategy implements error recovery via matrix concatenation. Given a set of ``hot channels'' $\mathcal{I} \subset \{1,\dots,c\}$ with cardinality $|\mathcal{I}|=k$, we construct augmented matrices 
    \begin{equation}
        \tilde{\mathbf{X}} = [\widehat{\mathbf{X}}^\top, \Delta\mathbf{X}_{\mathcal{I}}^\top]^\top \in \mathbb{R}^{(c+k) \times N}, \quad \tilde{\mathbf{W}} = [\widehat{\mathbf{W}}^\top, \Delta\mathbf{W}_{\mathcal{I}}^\top]^\top \in \mathbb{R}^{(c+k) \times M}.
    \end{equation}
    The corrected output is derived from a single monolithic matrix multiplication:
    \begin{equation}
        \mathbf{Y}_{\text{HCP}} = \tilde{\mathbf{W}}^\top\tilde{\mathbf{X}}.
    \end{equation}
    This approach facilitates kernel fusion and simplifies the computation graph, rendering it advantageous for deployment scenarios sensitive to kernel launch latency.
    
    \item \textbf{Dual-Kernel Mode (D):} This strategy decouples the base computation from residual correction. The system first computes $\mathbf{Y}_{\text{base}} = \widehat{\mathbf{W}}^\top\widehat{\mathbf{X}}$ via standard FP4 routines. Subsequently, the residual term 
    \begin{equation}
        \mathbf{Y}_{\text{res}} = \Delta\mathbf{W}_{\mathcal{I}}^\top\Delta\mathbf{X}_{\mathcal{I}}
    \end{equation}
    is computed in a separate kernel, followed by accumulation:
    \begin{equation}
        \mathbf{Y}_{\text{HCP}} = \mathbf{Y}_{\text{base}} + \mathbf{Y}_{\text{res}}.
    \end{equation}
    While incurring additional memory traffic for intermediate results, this mode enables fine-grained control over mixed-precision scheduling and independent kernel optimization.
\end{itemize}

\paragraph{Configuration Dimension 2: Order (Error Expansion).}

The \emph{Order} parameter dictates strictly which terms from the error decomposition are materialized:

\begin{itemize}[leftmargin=*, label=$\bullet$, itemsep=2pt]
    \item \textbf{Second-Order Recovery ($O2$):} This configuration compensates solely for the interaction term $\Delta\mathbf{W}_{\mathcal{I}}^\top\Delta\mathbf{X}_{\mathcal{I}}$ as proof in \cref{lemma:second-order}.
    \begin{equation}
        \mathbf{Y}_{\text{O2}} = \widehat{\mathbf{W}}_{\mathcal{I}}^\top \widehat{\mathbf{X}} _{\mathcal{I}}- \Delta\mathbf{W}_{\mathcal{I}}^\top \Delta\mathbf{X}_{\mathcal{I}}.
    \end{equation}
    Empirical analysis indicates that when utilizing unbiased calibrated quantizers (e.g., symmetric RTN), $O2$ recovery captures 85--90\% of the accuracy benefits at approximately half the computational cost of full recovery.
    
    \item \textbf{First-Order Recovery ($O1$):} This configuration recovers both the linear residuals and the second-order interaction on hot channels:
    \begin{equation}
        \mathbf{Y}_{\text{O1}} = \widehat{\mathbf{W}}_{\mathcal{I}}^\top \widehat{\mathbf{X}}_{\mathcal{I}} + \widehat{\mathbf{W}}_{\mathcal{I}}^\top\Delta\mathbf{X}_{\mathcal{I}} + \Delta\mathbf{W}_{\mathcal{I}}^\top\widehat{\mathbf{X}}_{\mathcal{I}} + \Delta\mathbf{W}_{\mathcal{I}}^\top\Delta\mathbf{X}_{\mathcal{I}}.
    \end{equation}
    Full expansion is critical when quantization errors manifest systematic bias or when the model converges to regions characterized by large weight or activation residuals.
\end{itemize}

\paragraph{Configuration Dimension 3: Target (Compensation Scope).}

The \emph{Target} parameter selects the specific residual tensors to be patched:

\begin{itemize}[leftmargin=*, label=$\bullet$, itemsep=2pt]
    \item \textbf{Weight-Only (W):} Activations remain fully quantized ($\Delta\mathbf{X}_{\mathcal{I}} = \mathbf{0}$), patching only weight residuals. This is memory-efficient as weight residuals are static and cacheable, yet it fails to address activation outliers which often dominate quantization error in Transformer attention layers.
    
    \item \textbf{Activation-Only (A):} Weights remain quantized, patching only activation residuals. This effectively mitigates dynamic outliers (e.g., in query projections) but neglects static weight pathologies.
    
    \item \textbf{Both (B):} Recovers residuals from both operands. Our ablation study in Fig.~\ref{fig:mse_error} demonstrates that bidirectional recovery reduces MSE by 30--45\% relative to unilateral approaches, particularly in layers exhibiting simultaneous weight and activation outliers.
\end{itemize}

\paragraph{Nomenclature and Taxonomy.}

We employ a compact notation formatted as \texttt{Mode-Order-Target} (e.g., \texttt{S-O1-A}, \texttt{D-O2-B}) to uniquely identify configurations. Tab.~\ref{tab:recipes} delineates the representative design space explored in this work.

\begin{table}[t]
    \centering
    \caption{\textbf{Taxonomy of Hot-Channel Patch (HCP) configurations.} The design space is categorized across three axes: Mode (Single (S) for concatenated kernels vs. Dual (D) for separate kernels), Order (Second-order (O2) for interaction terms vs. First-order inclusive (O1) for linear plus interaction residuals), and Target (Weight (W), Activation (A), or Both (B)).}
    \label{tab:recipes}
    \setlength{\tabcolsep}{10pt}
    \renewcommand{\arraystretch}{1.15}
    \begin{tabular}{c c c c}
        \toprule
        \textbf{Configuration} & \textbf{Mode} & \textbf{Order} & \textbf{Target} \\ 
        \midrule
        \texttt{S-O1-W} & Single & 1st-Order & Weight \\
        \texttt{S-O1-A} & Single & 1st-Order & Activation \\
        \texttt{D-O1-W} & Dual   & 1st-Order & Weight \\
        \texttt{D-O1-A} & Dual   & 1st-Order & Activation \\
        \texttt{S-O2-B} & Single & 2nd-Order & Both \\
        \texttt{D-O2-B} & Dual   & 2nd-Order & Both \\
        \bottomrule
    \end{tabular}
\end{table}

\subsection{Post-QK Operations Ablation Study}
\label{app:post_qk_operation}

To identify which components are most critical for high-precision protection, we conduct a systematic ablation study. 
As illustrated in Fig.~\ref{fig:post-qk-ablation}, a surface-level observation of the loss curves suggests that the \texttt{up\_proj} has the most significant impact on model convergence in both GLA and Qwen. 
However, attributing sensitivity solely to absolute loss increase is potentially misleading, as \texttt{up\_proj} typically accounts for a disproportionately larger share of the total parameters.

To decouple the intrinsic numerical sensitivity from the parameter scale, we introduce a \textit{quantization sensitivity score}, defined as the $\Delta \text{Loss}$ per unit of parameters relative to the BF16 baseline. 
As revealed in Fig.~\ref{fig:op-sensitivity}, this normalized metric provides a more precise diagnostic: the $O$ projection in GLA and the $V$ projection in Qwen-1.7B emerge as the most quantization-critical components. 

\begin{figure}[t]
    \centering
    \begin{subfigure}[t]{0.50\textwidth}
        \centering
        \includegraphics[width=\linewidth, height=0.6\linewidth, keepaspectratio]{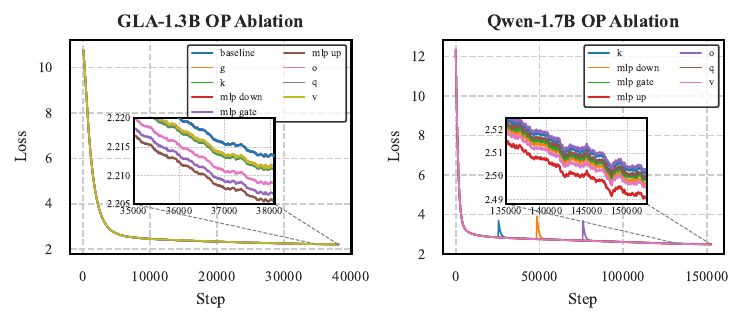}
        \vspace{-18pt}
        \caption{Operation Ablation Study}
        \label{fig:post-qk-ablation}
    \end{subfigure}
    \hfill
    \begin{subfigure}[t]{0.48\textwidth}
        \centering
        \includegraphics[width=\linewidth, height=0.6\linewidth, keepaspectratio]{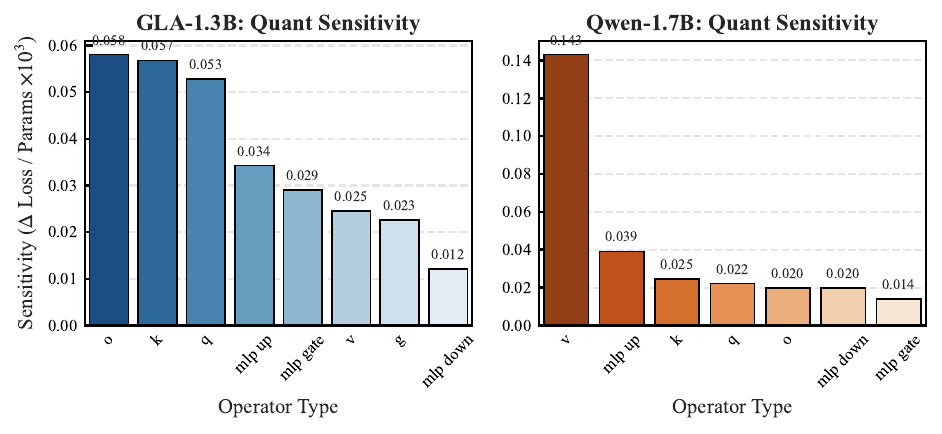}
        \vspace{-18pt}
        \caption{Quantization Sensitivity Analysis}
        \label{fig:op-sensitivity}
    \end{subfigure}
    \caption{\textbf{Comprehensive analysis of operator importance and quantization sensitivity.} 
    (a) Training loss curves for operation-level ablations in GLA-1.3B and Qwen-1.7B. While the \texttt{up\_proj} shows a high absolute loss impact, this is largely attributed to its disproportionate parameter scale. 
    (b) Quantization sensitivity scores normalized by parameter count ($\Delta \text{Loss} / \text{Params}$). The analysis identifies the GK and O projections in GLA, and the V projection in Qwen-1.7B, as the most critical bottlenecks. }
    \vspace{-10pt}
    \label{fig:combined-analysis}
\end{figure}

\section{Additional Experimental Details}
\label{app:exp_details}

\subsection{Experimental Setup}
\label{app:exp_setup}

\paragraph{Datasets and Preprocessing.}
We pretrain on a 60B-token subset of RedPajama~\cite{together2023redpajama}.
We remove documents flagged as toxic or low-quality by a lightweight classifier.
The final mixture is identical across all baselines. We use a 32k vocabulary trained with SentencePiece.
All models share the same tokenizer and special tokens.
Sequences are packed to 4,096 tokens without cross-document contamination.

\paragraph{Optimization and Schedule.}
Unless otherwise noted, we train with AdamW ($\beta_1{=}0.9$, $\beta_2{=}0.95$, weight decay $0.1$), cosine decay with 2k-step warmup, and gradient clipping at 1.0. For HCP, we designate ~9.09\% of channels as hot-channels for compensation.
We use FSDP with activation checkpointing. The sequence length is 4096 and the global batch size is 4M tokens. Training employs the AdamW optimizer ($\beta_1=0.9, \beta_2=0.95$) with a peak learning rate of $3\times10^{-4}$, 2,000 warmup steps, weight decay 0.1, and gradient clipping at 1.0. We use FSDP only (no tensor or pipeline parallelism) and gradient accumulation is 1.

\paragraph{Hardware.} All experiments were conducted on a server equipped with 8 NVIDIA Blackwell GPUs. Unless otherwise noted, we report per-GPU throughput and memory usage, with the model-parallel configuration held fixed across all runs. The software stack consists of Python~3.12.3, PyTorch~2.8.0, CUDA runtime~13.0, and cuDNN~9.12.

\subsection{Efficiency Details}
\label{app:efficiency}

We detail the implementation and efficiency of the Hot-Channel Patch (HCP) method.

\textbf{Algorithm Comparison.} The two algorithms presented in Algorithm 1 outline the compensation process. The \textit{Normal Process} (Left) performs dynamic selection, where importance scores are computed from quantization residuals to identify the top-$k$ hot channels at every step. In contrast, the \textit{Pre-computed Indices} version (Right) takes a set of indices as input, bypassing the \textit{Scoring \& Selection} stage. This modification reduces computational cost significantly. The validity of saving index calculation stems from the observation that outlier channels in LLMs exhibit distinct persistence; thus, indices computed at one step remain effective for subsequent steps, allowing us to cache and reuse them.

\textbf{Efficiency and Kernel Fusion.} Tab.~\ref{tab:chon-breakdown} analyzes the runtime overhead. The ``Pre-fuse'' columns show that executing operations such as dequantization, gathering, and concatenation individually incurs high overhead ($\sim$16\%). To mitigate this, we employ a \texttt{Triton} kernel to fuse residual compensation and concatenation into a single operation. This fusion strategy drastically lowers the overhead to 5.27\%, as evidenced in the ``Post-fused'' column, demonstrating the practical efficiency of our design.

\textbf{End-to-End Measurement.} We evaluate the end-to-end training performance of Qwen3-1.7B on 8 GPUs. For CHON, the total training time is 175.2 GPU-hours, compared to 152.72 GPU-hours for the baseline NVFP4 Recipe~\cite{nvidia_recipe}, corresponding to a 12.88\% overhead. For GLA-1.3B, CHON requires 92.88 GPU-hours, whereas the baseline requires 84.88 GPU-hours, yielding an 8.6\% overhead.

\begin{algorithm}[t]
\caption{Hot-Channel Patch Algorithm. (Left: Normal Process, Right: Use pre-computed indices.)}
\label{algo:hcp}
\begin{minipage}[t]{0.48\textwidth}
\begin{algorithmic}[1]
\REQUIRE Weight matrix $W \in \mathbb{R}^{M \times N}$, input matrix $X \in \mathbb{R}^{K \times N}$
\REQUIRE Number of compensation columns $c$
\ENSURE Expanded matrices $W_{\text{out}}, X_{\text{out}}$ for GEMM
\STATE \textbf{// 1. Quantization \& Dequantization}
\STATE $W_{q} \leftarrow \text{Quantize}(W, \text{NVFP4})$
\STATE $X_{q} \leftarrow \text{Quantize}(X, \text{NVFP4})$
\STATE $W_{dq} \leftarrow \text{Dequantize}(W_{q})$
\STATE $X_{dq} \leftarrow \text{Dequantize}(X_{q})$
\STATE \textbf{// 2. Residual Computation}
\STATE $R_W \leftarrow W - W_{dq}$ 
\COMMENT{quantization error for weights}
\STATE $R_X \leftarrow X - X_{dq}$ 
\COMMENT{quantization error for inputs}
\STATE \textbf{// 3. Scoring \& Selection (Top-$k$)}
\STATE $S_W \leftarrow \text{Mean}\bigl(\sum_{i=1}^{M} |R_W[i, :]|\bigr)$ 
\COMMENT{column-wise L1 norm}
\STATE $S_X \leftarrow \text{Mean}\bigl(\sum_{j=1}^{K} |R_X[j, :]|\bigr)$ 
\COMMENT{column-wise L1 norm}
\STATE $S \leftarrow S_W + S_X$ 
\COMMENT{importance score}
\STATE $\text{idx} \leftarrow \text{TopK}(S, k=c)$ 
\COMMENT{top $c$ hot channels}
\STATE \textbf{// 4. Gather}
\STATE $R_W^{\text{comp}} \leftarrow R_W[:, \text{idx}]$ 
\COMMENT{high-error residuals (weights)}
\STATE $W_{dq}^{\text{comp}} \leftarrow W_{dq}[:, \text{idx}]$ 
\COMMENT{quantized values (weights)}
\STATE $R_X^{\text{comp}} \leftarrow R_X[:, \text{idx}]$ 
\COMMENT{high-error residuals (inputs)}
\STATE $X_{dq}^{\text{comp}} \leftarrow X_{dq}[:, \text{idx}]$ 
\COMMENT{quantized values (inputs)}
\STATE \textbf{// 5. Concat}
\STATE $W_{\text{out}} \leftarrow \text{Concat}\bigl([W,\; R_W^{\text{comp}},\; W_{dq}^{\text{comp}}], \text{dim}=1\bigr)$
\STATE $X_{\text{out}} \leftarrow \text{Concat}\bigl([X,\; X_{dq}^{\text{comp}},\; R_X^{\text{comp}}], \text{dim}=1\bigr)$ 
\COMMENT{note: order may vary for alignment}
\STATE \textbf{return} $W_{\text{out}}, X_{\text{out}}$
\end{algorithmic}
\end{minipage}
\hfill
\begin{minipage}[t]{0.48\textwidth}
\begin{algorithmic}[1]
\REQUIRE Weight matrix $W \in \mathbb{R}^{M \times N}$, input matrix $X \in \mathbb{R}^{K \times N}$
\REQUIRE Pre-computed indices $\text{idx} \in \mathbb{R}^{c}$ of hot channels
\ENSURE Expanded matrices $W_{\text{out}}, X_{\text{out}}$ for GEMM
\STATE \textbf{// 1. Quantization \& Dequantization}
\STATE $W_{q} \leftarrow \text{Quantize}(W, \text{dtype=FP4})$
\STATE $X_{q} \leftarrow \text{Quantize}(X, \text{dtype=FP4})$
\STATE $W_{dq} \leftarrow \text{Dequantize}(W_{q})$
\STATE $X_{dq} \leftarrow \text{Dequantize}(X_{q})$
\STATE \textbf{// 2. Residual Computation}
\STATE $R_W \leftarrow W - W_{dq}$ \COMMENT{quantization error for weights}
\STATE $R_X \leftarrow X - X_{dq}$ \COMMENT{quantization error for inputs}
\STATE \textbf{// 3. Gather}
\STATE \textit{// Look up residual and quantized values for the specified columns}
\STATE $R_W^{\text{comp}} \leftarrow R_W[:, \text{idx}]$ 
\STATE $W_{dq}^{\text{comp}} \leftarrow W_{dq}[:, \text{idx}]$ 
\STATE $R_X^{\text{comp}} \leftarrow R_X[:, \text{idx}]$ 
\STATE $X_{dq}^{\text{comp}} \leftarrow X_{dq}[:, \text{idx}]$ 
\STATE \textbf{// 4. Concat}
\STATE $W_{\text{out}} \leftarrow \text{Concat}\bigl([W,\; R_W^{\text{comp}},\; W_{dq}^{\text{comp}}], \text{dim}=1\bigr)$
\STATE $X_{\text{out}} \leftarrow \text{Concat}\bigl([X,\; X_{dq}^{\text{comp}},\; R_X^{\text{comp}}], \text{dim}=1\bigr)$ 
\STATE \textbf{return} $W_{\text{out}}, X_{\text{out}}$
\end{algorithmic}
\end{minipage}
\end{algorithm}

\begin{table*}[t]
\centering
\caption{Fine-grained performance breakdown and fusion efficiency analysis of CHON. All time metrics are in milliseconds (ms).}
\label{tab:chon-breakdown}
\small
\begin{tabular}{l c cc ccccc c cc}
\toprule
\multirow{3}{*}{\textbf{Shape (W $\times$ X)}} & \textbf{Forward} & \multicolumn{2}{c}{\textbf{Backward Pass}} & \multicolumn{5}{c}{\textbf{CHON Operations (Pre-fuse)}} & \textbf{Fused} & \multicolumn{2}{c}{\textbf{Overhead Ratio (\%)$^\dagger$}} \\
\cmidrule(lr){2-2} \cmidrule(lr){3-4} \cmidrule(lr){5-9} \cmidrule(lr){10-10} \cmidrule(l){11-12}
& Fprop & DGrad & WGrad & Deq. & Gthr. & Resid. & Cat. & \textbf{Sum} & \textbf{Kernel} & \textbf{Pre-fuse} & \textbf{Post-fuse} \\
\midrule
$2048\times2048$ & 0.239 & 0.368 & 0.143 & 0.049 & 0.059 & 0.030 & 0.046 & 0.184 & 0.044 & 14.76\% & \textbf{3.97\%} \\
$1024\times2048$ & 0.235 & 0.238 & 0.142 & 0.046 & 0.062 & 0.030 & 0.046 & 0.184 & 0.044 & 16.56\% & \textbf{4.53\%} \\
$6144\times2048$ & 0.237 & 0.236 & 0.143 & 0.045 & 0.060 & 0.030 & 0.046 & 0.181 & 0.075 & 16.79\% & \textbf{7.72\%} \\
$2048\times6144$ & 0.235 & 0.238 & 0.143 & 0.046 & 0.060 & 0.030 & 0.046 & 0.182 & 0.047 & 16.47\% & \textbf{4.85\%} \\
\midrule
\textbf{Mean} & \textbf{0.237} & \textbf{0.270} & \textbf{0.143} & 0.047 & 0.060 & 0.030 & 0.046 & \textbf{0.183} & \textbf{0.053} & \textbf{16.15\%} & \textbf{5.27\%} \\
\bottomrule
\addlinespace[1ex]
\multicolumn{12}{l}{\footnotesize $^\dagger$ \textbf{Pre-fuse \%} = $\frac{\text{Sum}}{\text{Step} - \text{Fused} + \text{Sum}}$; \textbf{Post-fuse \%} = $\frac{\text{Fused\_Kernel}}{\text{Step}}$ (calculated based on end-to-end iteration time).} \\
\end{tabular}
\vspace{-10pt}
\end{table*}

\subsection{CHON Recipe}
\label{app:training_recipe}

We investigate the use of NVFP4 for training Linear Attention (focusing on GLA) under a 4-bit floating-point regime. To obtain stable end-to-end FP4 optimization, we incorporate the following techniques:

\ding{172} \textit{Precision partitioning}: retain higher precision for sensitive components (embeddings, output heads, normalization layers, and Attention kernel). 
\ding{173} \textit{Outlier suppression}: apply random Hadamard transforms in the backward pass only, avoiding inference overhead while diffusing sparse large-magnitude directions. 
\ding{174} \textit{Unbiased gradient quantization}: employ stochastic rounding during the backward pass to reduce systematic drift. 
\ding{175} \textit{Asymmetric granularity}: use 1D (per-channel) quantization in the forward path and 2D (tile-wise) quantization in the backward path to balance efficiency and robustness. 
\ding{176} \textit{Recover-block mechanism}: introduce an NVFP4-specific rehydration step for weight regions exhibiting transient outliers, mitigating irreversible clipping during accumulation.
Together these measures enable stable NVFP4 training of GLA without material degradation in convergence or downstream performance.

NVFP4 pairs a compact FP4 format with two-stage MicroScaling to expand effective dynamic range while keeping arithmetic low-bit. We summarize the forward and backward data flow below.

\paragraph{Two-stage MicroScaling.} Given a tensor $T \in \mathbb{R}^{m\times n}$, NVFP4 applies block-wise scales followed by per-row or per-channel refinements:
\begin{itemize}
    \item 2D block scaling: partition $T$ into tiles (e.g., $16\times16$) and compute block scale $s_{b}=\mathrm{RMS}(T_b)$; quantize $\tilde T_b = T_b / s_b$ to FP4.
    \item 1D refinement: apply narrower scales along a single dimension (e.g., $1\times16$) for activations/gradients to capture residual range.
\end{itemize}
This yields better coverage of local peaks while preserving regular memory access for GEMMs.

\paragraph{Forward (RTN) and Backward (SR).} The forward path typically uses round-to-nearest (RTN) quantization for deterministic kernels; the backward path employs stochastic rounding (SR) to reduce bias accumulation in gradients. For weight gradients, Randomized Hadamard Transform (RHT) scrambles inputs to diffuse sparse large-magnitude directions, stabilizing variance under FP4.

\paragraph{Sensitive Ops in higher precision.} Embeddings, output heads, normalization layers, nonlinearities, and QK GEMMs are commonly executed in BF16/FP16. Late layers can also be exempted to avoid end-of-training drift.

\paragraph{Mixed Precision}

Most GEMM operations within the Feedforward network (FFN) are executed in FP4. Each linear layer entails three GEMMs: one for the forward pass (Fprop) and two for the backward pass—one for activation gradients (Dgrad) and one for weight gradients (Wgrad). For ablations, we employ fake quantization: tensors are quantized to NVFP4 but the GEMMs themselves are performed in BF16. This isolates the effect of quantization from kernel-specific implementation details.

\begin{align}
\text{Fprop: } &\quad Y = X W, && Y \in \mathbb{R}^{B \times O},\; X \in \mathbb{R}^{B \times I},\; W \in \mathbb{R}^{I \times O} \\
\text{Dgrad: } &\quad \mathrm{d}X = \mathrm{d}Y\, W^\top \\
\text{Wgrad: } &\quad \mathrm{d}W = X^\top \mathrm{d}Y
\end{align}

\paragraph{Randomized Hadamard Transform}

Although Linear Attention exhibits markedly fewer outliers than Softmax Attention, we observe residual, localized outliers, predominantly in the gating linear modules. To mitigate their impact during training without incurring inference overhead, we apply a Randomized Hadamard Transform (RHT) only in the backward pass. Using RHT in the forward pass would necessitate mirroring the transform at inference time, introducing nontrivial latency and complexity. Following the NVIDIA FP4 recipe~\cite{nvidia_recipe}, we restrict RHT to the Wgrad GEMM: we scramble inputs via orthogonal Walsh–Hadamard transforms with random sign flips (e.g., $\tilde{X} = H D X$, $\mathrm{d}\tilde{Y} = H D' \mathrm{d}Y$) and compute $\mathrm{d}W = \tilde{X}^\top \mathrm{d}\tilde{Y}$, leveraging orthogonality to diffuse sparse large-magnitude components while preserving unbiased gradients.

\begin{itemize}
    \item \textbf{Forward/Backward Precision:} Forward compute uses NVFP4 kernels end-to-end. The backward pass employs stochastic rounding to prevent cumulative error drift, combined with a random Hadamard transform to reduce activation outliers.
    \item \textbf{Quantization Granularity:} We use 1D quantization in the forward path and 2D quantization in the backward path to balance efficiency and stability.
    \item \textbf{Layer Selection (Mixed Precision):} We do not quantize the Linear Attention module itself; the MLP is quantized. Sensitive layers such as output heads and normalization layers (e.g., RMSNorm) remain in higher precision. Overall this constitutes a mixed-precision quantization strategy.
    \item \textbf{Model Coverage:} Experiments are conducted on Mamba2~\cite{dao2024mamba2}, Gated Linear Attention (GLA)~\cite{yang2024gated}, and Gated DeltaNet~\cite{yang2025gateddeltanetworksimproving}. To our knowledge, we are the first organization to demonstrate the practical feasibility of end-to-end FP4 training on purely linear models.
    \item \textbf{Training Accuracy:} The final training achieves a loss error within 0.8\%.
\end{itemize}
This NVFP4-centric approach creates a numerically stable environment where the benefits of FP4 compute can be realized without sacrificing model quality.

\subsection{Details of NVFP4}
\label{app:nvfp4_details}

This section describes the scaling rules used to convert an input tensor to FP4 while preserving dynamic range at both the tensor and block granularity. We use a \emph{global} (tensor-level) scale to align the overall magnitude to the representable range of FP8 scale storage, and \emph{local} (block-level) scales to normalize each block to the FP4 range.

\paragraph{Notation.}
Let $x \in \mathbb{R}^{m}$ be a tensor (flattened for convenience). We partition indices into disjoint blocks $b$ (e.g., contiguous groups of elements). For a block $b$, denote the blockwise absolute maximum by
\begin{equation}
\amax_b \triangleq \max_{i \in b} |x_i| ,
\qquad
\amax_x \triangleq \max_i |x_i|.
\end{equation}
We write $\mathrm{e4m3}(\cdot)$ for quantization to the FP8 E4M3 format (used to store certain scale factors), and let $q(\cdot)$ denote the FP4 quantization function applied to values after scaling.

\begin{definition}[Global encode / decode scale]
\label{def:global_scale}
The \emph{global encode scale} $s_{\mathrm{enc}}$ is defined by
\begin{equation}
s_{\mathrm{enc}} \;=\; \frac{6 \cdot 448}{\amax_x}.
\label{eq:global_enc}
\end{equation}
The corresponding \emph{global decode scale} is
\begin{equation}
s_{\mathrm{dec}} \;=\; \frac{1}{s_{\mathrm{enc}}}.
\label{eq:global_dec}
\end{equation}
\end{definition}

\begin{remark}[Why the constants $6$ and $448$?]
Here $6$ and $448$ are the maximum representable magnitudes in FP4 E2M1 and FP8 E4M3, respectively. Intuitively, $s_{\mathrm{enc}}$ maps the tensor-level maximum $\amax_x$ into the product of the two maxima, so that subsequent blockwise scaling (targeting FP4) can be stored (in FP8) without overflow.
\end{remark}

\paragraph{Implementation note (memory traffic).}
Computing $\amax_x$ requires a pass over the tensor. Applying the global scaling prior to FP4 conversion may introduce an additional pass (read for $\amax_x$, then read again for scaling/conversion). In practice, one may choose smaller granularity for ``global'' scaling (e.g., per-row) to reduce round-trips through device memory, depending on the kernel and layout.

\begin{definition}[Local (block) decode scale]
\label{def:local_decode}
For each block $b$, the \emph{local decode scale} is chosen so that the largest magnitude in the block maps to the FP4 maximum:
\begin{equation}
s_{\mathrm{dec},b} \;=\; \frac{\amax_b}{6}.
\label{eq:local_dec}
\end{equation}
\end{definition}

\paragraph{Storing local scales in FP8.}
Local decode scales must be stored in FP8 (E4M3) for Tensor Core execution. To ensure these scales are representable in FP8, we first multiply them by the global encode scale and then quantize:
\begin{equation}
s_{\mathrm{dec},b}^{\mathrm{e4m3}}
\;=\;
\mathrm{e4m3}\!\bigl(s_{\mathrm{dec},b} \cdot s_{\mathrm{enc}}\bigr).
\label{eq:local_dec_e4m3}
\end{equation}
The role of $s_{\mathrm{enc}}$ is to remap the largest local decode scale
$\max_b s_{\mathrm{dec},b} = \amax_x/6$
into the FP8 range (whose maximum is $448$), hence avoiding saturation when storing $s_{\mathrm{dec},b}$ in FP8.

\begin{remark}[Recovering a usable local encode scale]
Quantizing $s_{\mathrm{dec},b} \cdot s_{\mathrm{enc}}$ in FP8 introduces rounding error. A practical way to obtain the effective \emph{local encode scale} is to invert the quantized value in higher precision and rescale back:
\begin{equation}
s_{\mathrm{enc},b}
\;\approx\;
\frac{1}{\mathrm{fp32}\!\left(s_{\mathrm{dec},b}^{\mathrm{e4m3}}\right)\cdot s_{\mathrm{dec}}}.
\end{equation}
This construction aims to keep $s_{\mathrm{enc},b}\cdot s_{\mathrm{dec}}\cdot \mathrm{fp32}(s_{\mathrm{dec},b}^{\mathrm{e4m3}})\approx 1$,
so that the overall scale/de-scale pipeline approximately preserves the original values. Round-to-nearest-even is typically used when forming~\eqref{eq:local_dec_e4m3}.
\end{remark}

\paragraph{Conversion and GEMM-time de-scaling}

\begin{definition}[Blockwise FP4 conversion]
\label{def:conversion}
For each element $x_i$ in block $b$, we form the FP4 value
\begin{equation}
\hat{x}_i \;=\; q\!\bigl(x_i \cdot s_{\mathrm{enc},b}\bigr),
\label{eq:fp4_quant}
\end{equation}
where $q(\cdot)$ denotes FP4 quantization.
\end{definition}

Beyond storing the quantized values $\hat{x}$, we store both the global decode scale $s_{\mathrm{dec}}$ (typically in FP32) and the blockwise FP8 scales $s_{\mathrm{dec},b}^{\mathrm{e4m3}}$ for use during matrix multiplication.

\paragraph{Applying scales inside Tensor Core GEMM.}
Let $x$ and $y$ denote the two FP4-quantized operands, each with its own blockwise FP8 decode scales. For a block $b$, Tensor Cores apply the blockwise decode scales to the partial dot product over that block:
\begin{equation}
s_{\mathrm{dec},b,x}^{\mathrm{e4m3}}\cdot s_{\mathrm{dec},b,y}^{\mathrm{e4m3}}
\cdot
\sum_{k \in b} \bigl(\hat{x}_k \cdot \hat{y}_k\bigr).
\label{eq:blockwise_gemm}
\end{equation}
After the GEMM accumulation, the global decode scales $s_{\mathrm{dec},x}$ and $s_{\mathrm{dec},y}$ are applied to the output in an analogous fashion.

\begin{remark}[Interpretation: two-level dynamic range control]
The blockwise scaling (via $s_{\mathrm{dec},b}$) adapts to local magnitude variations so FP4 uses its limited dynamic range effectively, while the global scaling (via $s_{\mathrm{enc}}$ and $s_{\mathrm{dec}}$) ensures the \emph{scales themselves} remain representable and efficiently usable by hardware (FP8 for per-block metadata, FP32 for final correction).
\end{remark}

\section{Additional Experiments}
\label{app:addition_exp}

\subsection{Supervised Fine-tuning with NVFP4}
\label{app:nvfp4_sft}

We conduct supervised fine-tuning experiments with NVFP4 quantization to assess its effectiveness in practical training settings. \textbf{Dataset.} We use a subset of the OpenMathReasoning~\cite{openmathreasoning} dataset consisting of chain-of-thought reasoning examples, totaling 12.67B tokens. \textbf{Model and Hardware.} All experiments are based on Qwen3-8B and are executed on a single machine equipped with 8 Blackwell GPUs.

\textbf{Quantization Configuration.} NVFP4 quantization is applied to all linear layers, including the query, key, value, and output projections in Softmax Attention, as well as the gate, up, and down projections in MLP blocks. To preserve model expressiveness, the \texttt{lm\_head} and \texttt{embedding} layers are kept in higher precision. Flash attention operations are not quantized, and we do not enable compensation strategies or the last4 optimization in this study.

\textbf{Training Configuration and Hyperparameters.} We train with a maximum sequence length of 32{,}768 and a batch size of 8, using FSDP2 for distributed training (FSDP size=8). The training framework is \texttt{VERL}, with FSDP2 as the training engine and \texttt{sglang} as the inference engine. The model is trained for one epoch, requiring 115.4 hours of wall-clock time (924 GPU hours). Optimization uses a learning rate of $2 \times 10^{-5}$, a warmup ratio of 0.01, weight decay of 0.1, betas of [0.9, 0.95], gradient clipping of 1.0, a minimum learning-rate ratio of 0.1, and a cosine warmup-and-decay schedule.

\textbf{Training Dynamics.} Fig.~\ref{fig:sft_loss} illustrates the training loss trajectories for both the BF16 and NVFP4 configurations. We observe a discernible divergence in loss between the two methods as training progresses, which becomes particularly pronounced during the learning rate decay phase. This phenomenon is likely attributable to the diminishing learning rate, which amplifies the relative impact of quantization noise on convergence—an observation consistent with established low-precision training protocols~\cite{nvidia_recipe}. Furthermore, the evolution of gradient norms reveals that NVFP4 maintains higher magnitudes than BF16 in the later stages of training, suggesting heightened gradient variance inherent to the low-precision regime.

\textbf{Evaluation Results.} The efficacy of the fine-tuned models is assessed on the AIME 2024 and AIME 2025 mathematical reasoning benchmarks. To mitigate the inherent stochasticity in model generation, performance metrics are averaged over 32 independent runs, although minor fluctuations across evaluation trials persist. As detailed in Tab.~\ref{tab:sft_results}, the NVFP4-trained model achieves performance highly competitive with the BF16 baseline, yielding AIME24 scores of approximately 0.61--0.62 and AIME25 scores between 0.45--0.48. These results underscore that NVFP4 quantization preserves model integrity and reasoning capabilities while substantially decreasing the computational overhead of the training process.

\begin{figure}[t]
    \centering
    \includegraphics[width=0.8\linewidth]{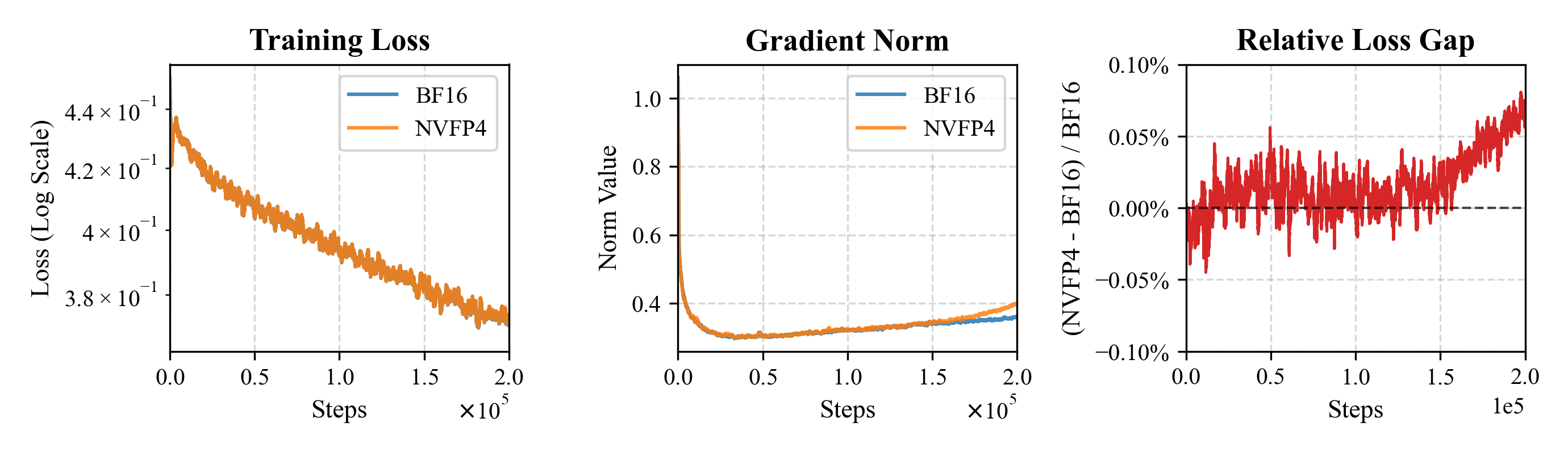}
    \vspace{-10pt}
    \caption{\textbf{Training dynamics of NVFP4 during Supervised Fine-tuning (SFT).}  (a) Training loss curves comparing BF16 and NVFP4 training. The loss gap increases during the learning rate decay phase. (b) Gradient norm evolution during training. NVFP4 shows higher gradient norms in later stages, indicating increased variance under low-precision training. (c) Loss Gap (NVFP4 - BF16)/BF16 during training. The loss gap continues to widen as training progresses, particularly during the learning rate decay phase.}
    \label{fig:sft_loss}
    \vspace{-15pt}
\end{figure}

\begin{table}[t]
    \centering
    \caption{\textbf{Evaluation results on AIME benchmarks for fine-tuned Qwen3-8B.} Comparison of mathematical reasoning performance between BF16 and NVFP4 models across three independent runs. Results are reported as mean scores averaged over 32 samples per run. The data demonstrates that NVFP4 training maintains competitive accuracy, achieving parity with the BF16 baseline on both AIME24 and 25.}
    \label{tab:sft_results}
    \begin{tabular}{llcc}
        \toprule
        Run & Method & AIME24 & AIME25 \\
        \midrule
        Run 1 & BF16 Training  & 0.595 & 0.478 \\
        Run 2 & BF16 Training  & 0.599 & 0.472 \\
        Run 3 & BF16 Training  & 0.620 & 0.482 \\
        \midrule
        Run 1 & NVFP4 Training & 0.620 & 0.448 \\
        Run 2 & NVFP4 Training & 0.615 & 0.478 \\
        Run 3 & NVFP4 Training & 0.637 & 0.469 \\
        \bottomrule
    \end{tabular}
\end{table}

\subsection{Reinforcement Learning with NVFP4}
\label{app:nvfp4_rl}

\paragraph{Experimental Setup.} We use the Qwen3-8B model as the base model for all experiments and train on the GSM8K dataset, using the training set for policy optimization and the test set for evaluation; all experiments are conducted on 8 NVIDIA Blackwell GPUs. For policy optimization, we employ GRPO~\cite{shao2024GRPO} using verl~\cite{sheng2025hybridflow}, with a learning rate of $1 \times 10^{-6}$, a training batch size of 1024, a PPO mini-batch size of 256, and a micro-batch size of 32 per GPU. The maximum prompt length is 512 tokens and the maximum response length is 1024 tokens. KL divergence regularization is enabled with a coefficient of $0.001$ using the low-variance KL estimator, while the entropy coefficient is set to 0 (no entropy regularization). Training runs for 15 epochs with gradient checkpointing enabled to reduce memory consumption. Under FSDP, there is no parameter or optimizer offloading for the actor model, while parameter offloading is enabled for the reference model. For rollouts, we use vLLM with 2-way tensor parallelism and 80\% GPU memory utilization; for each prompt we generate 8 responses ($n=8$), and log-probability computation uses a micro-batch size of 32 per GPU.

\paragraph{Configurations.} We compare five training/rollout setups (Tab.~\ref{tab:exp_configs}): C1 (\texttt{nvfp4\_R\_fp8}) uses NVFP4 training with FP8 rollout and no importance sampling (IS); C2 (\texttt{nvfp4\_R\_fp8\_token}) matches C1 but adds token-level IS with clipping at 2.0; C3 (\texttt{nvfp4\_R\_fp8\_seq}) matches C1 but uses sequence-level IS with clipping at 2.0; C4 (\texttt{nvfp4\_R\_bf16}) uses NVFP4 training with BF16 rollout and no IS; and C5 (\texttt{bf16\_R\_bf16}) is the full-precision BF16 training+rollout baseline. All NVFP4 runs use the same NVFP4 recipe (no HCP compensation and quantize all layers), as noted in Tab.~\ref{tab:exp_configs}.

\begin{table}[t]
    \centering
    \caption{\textbf{Training and rollout configurations for Reinforcement Learning (RL).} Detailed setup for GRPO training on the GSM8K dataset, comparing five configurations (C1--C5). The taxonomy spans varying training precisions (BF16 vs. NVFP4), rollout precisions (FP8 vs. BF16), and Truncated Importance Sampling (TIS) strategies (None, Token-level, or Sequence-level) with corresponding clipping thresholds.}
    \begin{tabular}{@{}llcccc@{}}
    \toprule
    \textbf{ID} & \textbf{Configuration} & \textbf{Train} & \textbf{Rollout} & \textbf{TIS Level} & \textbf{IS Clip} \\
    \midrule
    C1 & \texttt{nvfp4\_R\_fp8}        & NVFP4 & FP8  & None & -- \\
    C2 & \texttt{nvfp4\_R\_fp8\_token} & NVFP4 & FP8  & Token & 2.0 \\
    C3 & \texttt{nvfp4\_R\_fp8\_seq}   & NVFP4 & FP8  & Seq. & 2.0 \\
    C4 & \texttt{nvfp4\_R\_bf16}       & NVFP4 & BF16 & None & -- \\
    C5 & \texttt{bf16\_R\_bf16}        & BF16  & BF16 & None & -- \\
    \bottomrule
    \end{tabular}
    \label{tab:exp_configs}
\end{table}

\paragraph{Analysis of NVFP4 RL.}
(1) \textbf{Training Stability and Convergence.} Fig.~\ref{app:rl_training_nvfp4} evaluates GRPO training performance using NVFP4 quantization across various rollout configurations compared to a full BF16 baseline (\texttt{bf16\_R\_bf16}). The results demonstrate that NVFP4 training with either BF16 or FP8 rollouts (blue, green, and red curves) achieves a mean reward of approximately 0.95, matching the baseline's convergence and validating the effectiveness of NVFP4 in reinforcement learning. While NVFP4 training with FP8 rollout (green) shows a temporary performance dip around step 40, it recovers rapidly. Notably, although we have implemented a framework for end-to-end NVFP4 training and rollout, the current configuration (represented by the purple curve) faces convergence challenges where the reward remains stagnant; this remains a primary focus of our ongoing research.

(2) \textbf{Policy Dynamics and Efficiency.} Analysis of policy entropy and response length reveals distinct training behaviors. While the BF16 baseline and standard NVFP4 configurations follow a natural progression toward a deterministic policy (decreasing entropy), token-level importance sampling (red) induces a sharp increase in entropy after step 60, encouraging sustained exploration. Furthermore, all converging configurations exhibit a downward trend in response length as the model learns to provide more concise reasoning for the GSM8K task. Interestingly, NVFP4-trained models (blue and green) achieve shorter average responses ($\approx$ 400 tokens) compared to the BF16 baseline ($\approx$ 500 tokens), suggesting that NVFP4 quantization may promote more efficient reasoning patterns while maintaining high performance.

\begin{figure}
    \centering
    \includegraphics[width=0.85\linewidth]{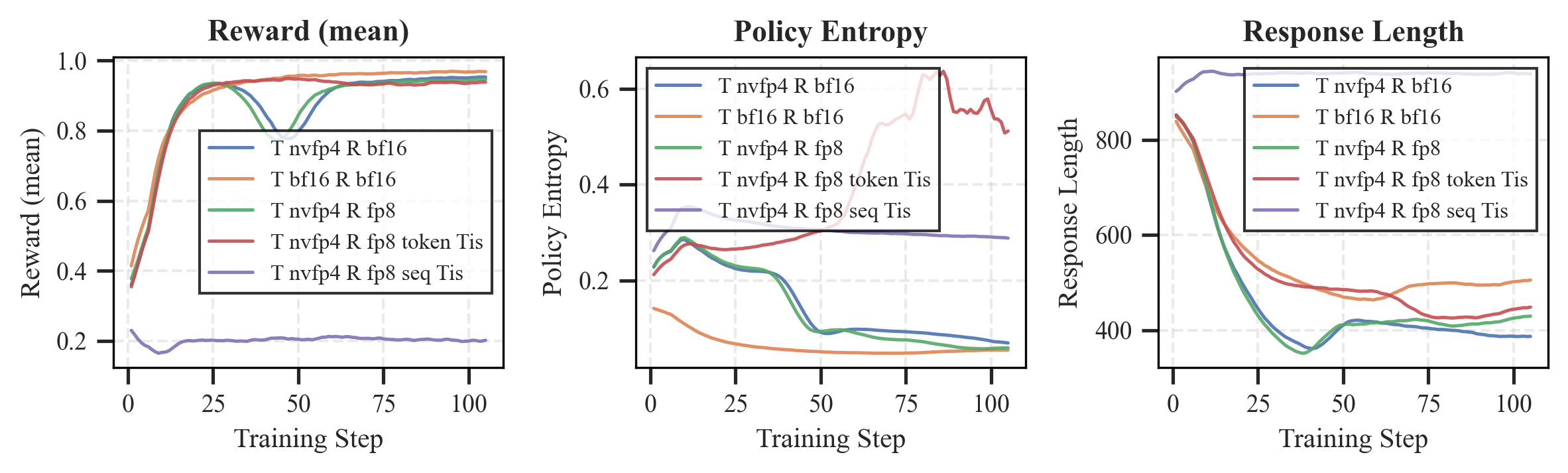}
    \vspace{-10pt}
    \caption{\textbf{Comparison of GRPO training dynamics on GSM8K across different precision configurations.} (Left) \textbf{Mean Reward}: NVFP4 training with either BF16 or FP8 rollouts matches the BF16 baseline convergence ($\approx 0.95$), despite a transient dip around step 40. (Middle) \textbf{Policy Entropy}: Standard configurations exhibit decreasing entropy as the policy becomes deterministic, while token-level importance sampling (red) triggers an entropy spike after step 60, encouraging sustained exploration. (Right) \textbf{Response Length}: All converging configurations show a downward trend; notably, NVFP4 models achieve more concise reasoning paths ($\approx 400$ tokens) compared to the BF16 baseline ($\approx 500$ tokens).}
    \label{app:rl_training_nvfp4}
    \vspace{-10pt}
\end{figure}

\subsection{Additional Downstream Tasks Performance}
\label{app:down_steam_tasks}

To further demonstrate the generalizability and scalability of the CHON recipe, we provide additional zero-shot evaluation results on downstream benchmarks for several model architectures and scales, including GLA-1.3B, GLA-7B, Gated-DeltaNet-1B, and GSA-340M. These results, summarized in Tab.~\ref{tab:more_mainexp_results}, complement the findings presented in the main text for GLA-340M, Gated DeltaNet-340M, GSA-1B, and Qwen3-1.7B.

\paragraph{Analysis}
As shown in Tab.~\ref{tab:more_mainexp_results}, our proposed recipe consistently narrows the performance gap between 4-bit quantized training and the BF16 baseline across diverse settings:
(1) \textbf{Scalability to Larger Models.} For GLA-7B, the CHON recipe (61.3\%) outperforms both FP8 (59.9\%) and the BF16 baseline (59.3\%), suggesting that larger models combined with HCP can achieve highly efficient training without quality degradation.
(2) \textbf{Robustness Across Architectures.} CHON effectively recovers performance drops seen in standard FP8. For Gated-DeltaNet-1B, it reaches 58.3\% (near BF16's 58.6\%), and for GSA-340M, it achieves parity with full precision (31.4\% vs. 31.3\%), demonstrating its versatility across different linear-attention frameworks.
(3) \textbf{Effectiveness of HCP.} NVFP4 with HCP consistently serves as the best-performing quantized configuration. By patching ``hot channels'' to mitigate heavy-tailed noise, HCP preserves critical activations and stabilizes optimization across diverse architectures. These experiments confirm the CHON recipe as a robust, scale-agnostic framework for stable NVFP4 pretraining.

\begin{table*}[t]
\centering
\footnotesize
\setlength{\tabcolsep}{1.8pt}
\caption{\textbf{Zero-shot performance on downstream benchmarks across various linear attention architectures and scales.} Results are reported as percentages (\%) with standard errors. $Acc_n$ denotes length-normalized accuracy; $Avg$ represents the mean across all tasks. The data demonstrates that CHON consistently narrows the quantization gap compared to BF16 and outperforms standard FP8 across scales (340M to 7B).}
\resizebox{\textwidth}{!}{
\begin{tabular}{@{}lccccccccccccccc@{}}
\toprule
\textbf{Setting} &
\multicolumn{2}{c}{\textbf{ARC-C}} &
\multicolumn{2}{c}{\textbf{ARC-E}} &
\multicolumn{2}{c}{\textbf{HellaSwag}} &
\multicolumn{2}{c}{\textbf{OBQA}} &
\multicolumn{2}{c}{\textbf{PIQA}} &
\multicolumn{2}{c}{\textbf{SciQ}} &
\textbf{Winograd} & \textbf{Avg} \\
& Acc & Acc$_{\text{n}}$ & Acc & Acc$_{\text{n}}$ & Acc & Acc$_{\text{n}}$ & Acc & Acc$_{\text{n}}$ & Acc & Acc$_{\text{n}}$ & Acc & Acc$_{\text{n}}$ & Acc & \\
\midrule
\multicolumn{15}{c}{\mygrey \textbf{GLA-1.3B}} \\
BF16 & 33.8{\scriptsize$\pm$1.4} & 36.9{\scriptsize$\pm$1.4} & 70.0{\scriptsize$\pm$0.9} & 64.2{\scriptsize$\pm$1.0} & 42.0{\scriptsize$\pm$0.5} & 54.1{\scriptsize$\pm$0.5} & 30.4{\scriptsize$\pm$2.1} & 40.0{\scriptsize$\pm$2.2} & 72.3{\scriptsize$\pm$1.0} & 72.4{\scriptsize$\pm$1.0} & 89.5{\scriptsize$\pm$1.0} & 83.5{\scriptsize$\pm$1.2} & 57.9{\scriptsize$\pm$1.4} & 57.5 \\
FP8 & 33.8{\scriptsize$\pm$1.4} & 35.2{\scriptsize$\pm$1.4} & 70.2{\scriptsize$\pm$0.9} & 62.8{\scriptsize$\pm$1.0} & 41.7{\scriptsize$\pm$0.5} & 54.2{\scriptsize$\pm$0.5} & \textbf{27.8{\scriptsize$\pm$2.0}} & 38.8{\scriptsize$\pm$2.2} & 71.0{\scriptsize$\pm$1.1} & 72.0{\scriptsize$\pm$1.1} & \textbf{89.7{\scriptsize$\pm$1.0}} & 83.7{\scriptsize$\pm$1.2} & 55.2{\scriptsize$\pm$1.4} & 56.6 \\
NVFP4 & 33.6{\scriptsize$\pm$1.4} & \textbf{37.3{\scriptsize$\pm$1.4}} & \textbf{70.4{\scriptsize$\pm$0.9}} & 63.4{\scriptsize$\pm$1.0} & \textbf{42.0{\scriptsize$\pm$0.5}} & \textbf{54.3{\scriptsize$\pm$0.5}} & 27.0{\scriptsize$\pm$2.0} & 38.2{\scriptsize$\pm$2.2} & 71.1{\scriptsize$\pm$1.1} & \textbf{72.3{\scriptsize$\pm$1.0}} & 88.6{\scriptsize$\pm$1.0} & 83.0{\scriptsize$\pm$1.2} & 56.2{\scriptsize$\pm$1.4} & 56.7 \\
CHON & \textbf{36.6{\scriptsize$\pm$1.4}} & 37.2{\scriptsize$\pm$1.4} & 70.0{\scriptsize$\pm$0.9} & \textbf{63.9{\scriptsize$\pm$1.0}} & 41.5{\scriptsize$\pm$0.5} & 54.0{\scriptsize$\pm$0.5} & \textbf{27.8{\scriptsize$\pm$2.0}} & \textbf{40.4{\scriptsize$\pm$2.2}} & \textbf{71.7{\scriptsize$\pm$1.1}} & 71.7{\scriptsize$\pm$1.1} & 89.0{\scriptsize$\pm$1.0} & \textbf{84.8{\scriptsize$\pm$1.1}} & \textbf{56.5{\scriptsize$\pm$1.4}} & \textbf{57.3} \\[0.6ex]
\midrule
\multicolumn{15}{c}{\mygrey \textbf{GLA-7B}} \\
BF16 & 37.0{\scriptsize$\pm$1.4} & 41.0{\scriptsize$\pm$1.4} & 72.3{\scriptsize$\pm$0.9} & 66.5{\scriptsize$\pm$1.0} & 45.0{\scriptsize$\pm$0.5} & 59.1{\scriptsize$\pm$0.5} & 28.8{\scriptsize$\pm$2.0} & 39.6{\scriptsize$\pm$2.2} & 72.3{\scriptsize$\pm$1.0} & 73.0{\scriptsize$\pm$1.0} & 90.7{\scriptsize$\pm$0.9} & 85.8{\scriptsize$\pm$1.1} & 59.9{\scriptsize$\pm$1.4} & 59.3 \\
FP8 & 37.2{\scriptsize$\pm$1.4} & 40.1{\scriptsize$\pm$1.4} & 73.7{\scriptsize$\pm$0.9} & 65.7{\scriptsize$\pm$1.0} & 45.9{\scriptsize$\pm$0.5} & 59.8{\scriptsize$\pm$0.5} & 31.4{\scriptsize$\pm$2.1} & 42.2{\scriptsize$\pm$2.2} & 72.4{\scriptsize$\pm$1.0} & 72.5{\scriptsize$\pm$1.0} & 91.1{\scriptsize$\pm$0.9} & 85.8{\scriptsize$\pm$1.1} & 60.6{\scriptsize$\pm$1.4} & 59.9 \\
NVFP4 & \textbf{39.3{\scriptsize$\pm$1.4}} & \textbf{41.9{\scriptsize$\pm$1.4}} & 74.8{\scriptsize$\pm$0.9} & 68.2{\scriptsize$\pm$1.0} & \textbf{46.8{\scriptsize$\pm$0.5}} & 61.6{\scriptsize$\pm$0.5} & \textbf{32.8{\scriptsize$\pm$2.1}} & 43.2{\scriptsize$\pm$2.2} & \textbf{73.7{\scriptsize$\pm$1.0}} & 74.0{\scriptsize$\pm$1.0} & 92.0{\scriptsize$\pm$0.9} & \textbf{87.7{\scriptsize$\pm$1.0}} & \textbf{61.6{\scriptsize$\pm$1.4}} & \textbf{61.4} \\
CHON & \textbf{39.3{\scriptsize$\pm$1.4}} & 40.6{\scriptsize$\pm$1.4} & \textbf{75.1{\scriptsize$\pm$0.9}} & \textbf{69.3{\scriptsize$\pm$1.0}} & 46.7{\scriptsize$\pm$0.5} & \textbf{61.9{\scriptsize$\pm$0.5}} & 32.0{\scriptsize$\pm$2.1} & \textbf{43.4{\scriptsize$\pm$2.2}} & 73.1{\scriptsize$\pm$1.0} & \textbf{75.0{\scriptsize$\pm$1.0}} & \textbf{92.3{\scriptsize$\pm$0.8}} & 87.3{\scriptsize$\pm$1.1} & 61.2{\scriptsize$\pm$1.4} & 61.3 \\[0.6ex]
\midrule
\multicolumn{15}{c}{\mygrey \textbf{Gated-DeltaNet-1B}} \\
BF16 & 35.9{\scriptsize$\pm$1.4} & 39.0{\scriptsize$\pm$1.4} & 71.9{\scriptsize$\pm$0.9} & 65.9{\scriptsize$\pm$1.0} & 43.5{\scriptsize$\pm$0.5} & 56.6{\scriptsize$\pm$0.5} & 29.4{\scriptsize$\pm$2.0} & 40.6{\scriptsize$\pm$2.2} & 72.8{\scriptsize$\pm$1.0} & 72.9{\scriptsize$\pm$1.0} & 89.5{\scriptsize$\pm$1.0} & 85.3{\scriptsize$\pm$1.1} & 58.7{\scriptsize$\pm$1.4} & 58.6 \\
FP8 & 29.3{\scriptsize$\pm$1.3} & 31.9{\scriptsize$\pm$1.4} & 65.7{\scriptsize$\pm$1.0} & 59.6{\scriptsize$\pm$1.0} & 38.8{\scriptsize$\pm$0.5} & 48.6{\scriptsize$\pm$0.5} & 26.0{\scriptsize$\pm$2.0} & 37.0{\scriptsize$\pm$2.2} & 69.7{\scriptsize$\pm$1.1} & 69.5{\scriptsize$\pm$1.1} & 86.4{\scriptsize$\pm$1.1} & 81.5{\scriptsize$\pm$1.2} & 53.3{\scriptsize$\pm$1.4} & 53.6 \\
NVFP4 & 35.4{\scriptsize$\pm$1.4} & 36.6{\scriptsize$\pm$1.4} & \textbf{72.0{\scriptsize$\pm$0.9}} & \textbf{65.4{\scriptsize$\pm$1.0}} & \textbf{43.4{\scriptsize$\pm$0.5}} & \textbf{56.3{\scriptsize$\pm$0.5}} & \textbf{30.6{\scriptsize$\pm$2.1}} & \textbf{40.8{\scriptsize$\pm$2.2}} & \textbf{72.1{\scriptsize$\pm$1.1}} & 72.3{\scriptsize$\pm$1.0} & 90.5{\scriptsize$\pm$0.9} & \textbf{85.6{\scriptsize$\pm$1.1}} & 57.6{\scriptsize$\pm$1.4} & \textbf{58.4} \\
CHON & \textbf{36.2{\scriptsize$\pm$1.4}} & \textbf{37.1{\scriptsize$\pm$1.4}} & 71.1{\scriptsize$\pm$0.9} & 64.1{\scriptsize$\pm$1.0} & 43.2{\scriptsize$\pm$0.5} & \textbf{56.3{\scriptsize$\pm$0.5}} & 30.2{\scriptsize$\pm$2.1} & 40.0{\scriptsize$\pm$2.2} & \textbf{72.1{\scriptsize$\pm$1.1}} & \textbf{72.6{\scriptsize$\pm$1.0}} & \textbf{90.9{\scriptsize$\pm$0.9}} & 84.5{\scriptsize$\pm$1.2} & \textbf{59.1{\scriptsize$\pm$1.4}} & 58.3 \\[0.6ex]
\midrule
\multicolumn{15}{c}{\mygrey \textbf{GSA-340M}} \\
BF16 & 18.0{\scriptsize$\pm$1.1} & 22.1{\scriptsize$\pm$1.2} & 33.0{\scriptsize$\pm$1.0} & 31.5{\scriptsize$\pm$1.0} & 25.5{\scriptsize$\pm$0.4} & 25.1{\scriptsize$\pm$0.4} & 12.2{\scriptsize$\pm$1.5} & 24.4{\scriptsize$\pm$1.9} & 53.8{\scriptsize$\pm$1.2} & 52.3{\scriptsize$\pm$1.2} & 28.5{\scriptsize$\pm$1.4} & 30.4{\scriptsize$\pm$1.5} & 50.1{\scriptsize$\pm$1.4} & 31.3 \\
FP8 & 18.0{\scriptsize$\pm$1.1} & 22.1{\scriptsize$\pm$1.2} & \textbf{33.0{\scriptsize$\pm$1.0}} & 31.5{\scriptsize$\pm$1.0} & 25.5{\scriptsize$\pm$0.4} & \textbf{25.1{\scriptsize$\pm$0.4}} & 12.2{\scriptsize$\pm$1.5} & 24.4{\scriptsize$\pm$1.9} & 53.8{\scriptsize$\pm$1.2} & 52.3{\scriptsize$\pm$1.2} & 28.5{\scriptsize$\pm$1.4} & \textbf{30.4{\scriptsize$\pm$1.5}} & 50.1{\scriptsize$\pm$1.4} & 31.3 \\
NVFP4 & \textbf{18.6{\scriptsize$\pm$1.1}} & 21.8{\scriptsize$\pm$1.2} & 30.4{\scriptsize$\pm$0.9} & 30.1{\scriptsize$\pm$0.9} & \textbf{25.9{\scriptsize$\pm$0.4}} & 24.9{\scriptsize$\pm$0.4} & \textbf{13.4{\scriptsize$\pm$1.5}} & \textbf{25.0{\scriptsize$\pm$1.9}} & 53.6{\scriptsize$\pm$1.2} & 51.3{\scriptsize$\pm$1.2} & 25.9{\scriptsize$\pm$1.4} & 28.9{\scriptsize$\pm$1.4} & 51.1{\scriptsize$\pm$1.4} & 30.1 \\
CHON & 18.1{\scriptsize$\pm$1.1} & \textbf{22.4{\scriptsize$\pm$1.2}} & 31.4{\scriptsize$\pm$1.0} & \textbf{31.6{\scriptsize$\pm$1.0}} & 25.6{\scriptsize$\pm$0.4} & 24.7{\scriptsize$\pm$0.4} & 12.4{\scriptsize$\pm$1.5} & 24.8{\scriptsize$\pm$1.9} & \textbf{54.6{\scriptsize$\pm$1.2}} & \textbf{53.2{\scriptsize$\pm$1.2}} & \textbf{29.2{\scriptsize$\pm$1.4}} & 29.2{\scriptsize$\pm$1.4} & \textbf{51.4{\scriptsize$\pm$1.4}} & \textbf{31.4} \\[0.6ex]
\bottomrule
\end{tabular}}
\vspace{-10pt}
\label{tab:more_mainexp_results}
\end{table*}

\section{Additional Analysis}
\label{app:analysis}

\subsection{Kurtosis Analysis}
\label{app:kurtosis}

\begin{figure*}[ht]
    \centering
    \begin{minipage}{0.32\linewidth}
        \includegraphics[width=\linewidth]{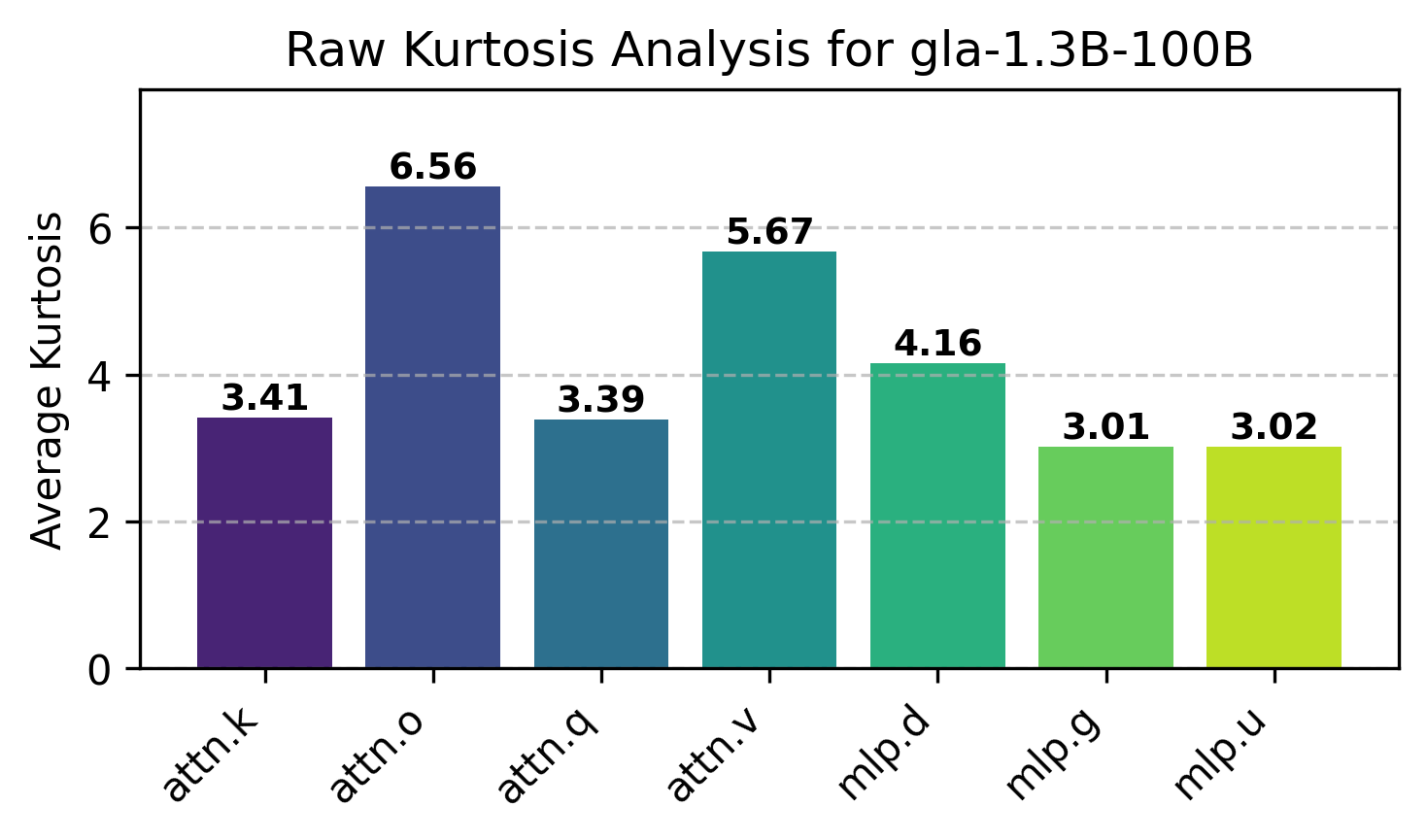}
        \centering\small{Gated Linear Attention (1.3B)}
    \end{minipage}\hfill
    \begin{minipage}{0.32\linewidth}
        \includegraphics[width=\linewidth]{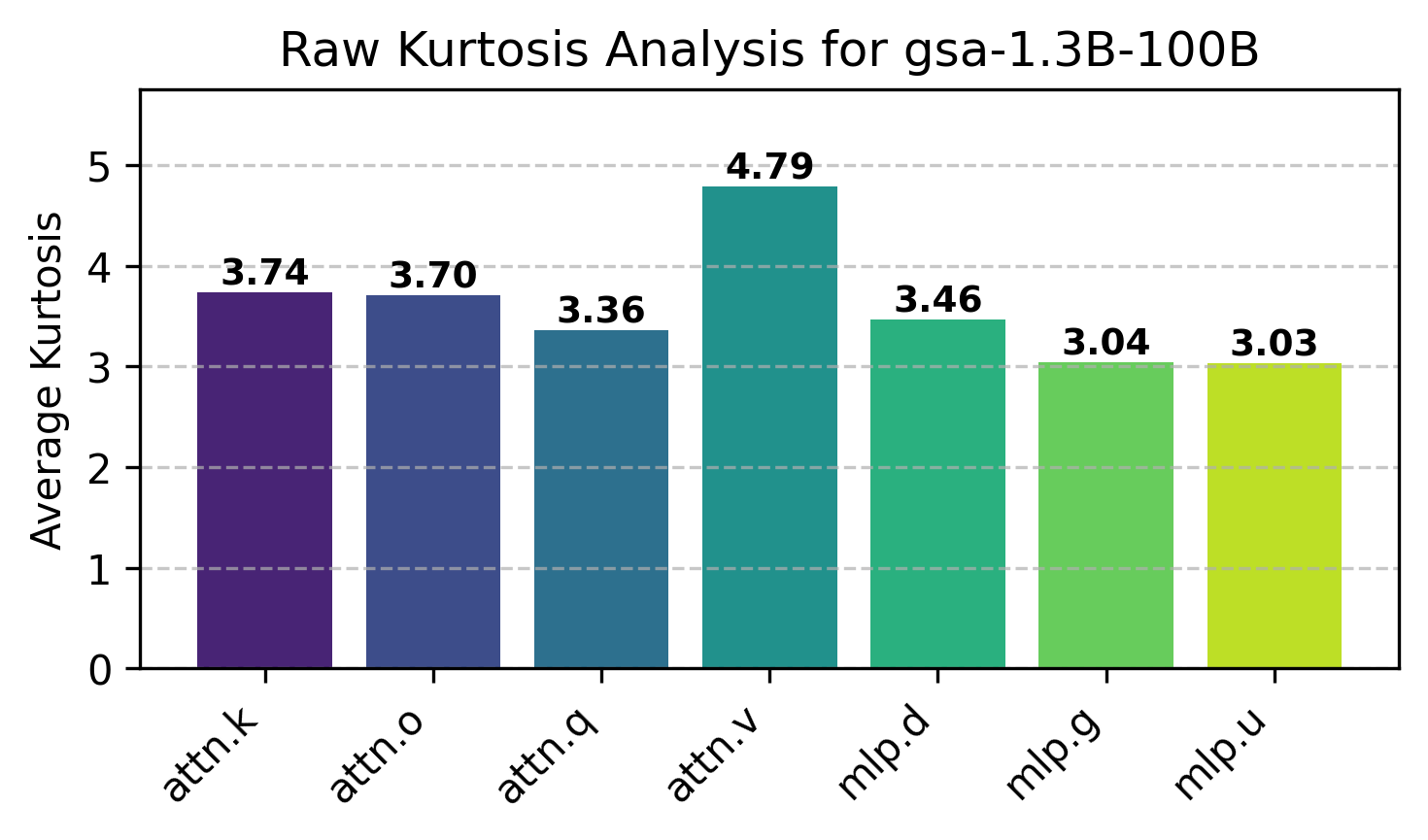}
        \centering\small{Gated Slot Attention (1.3B)}
    \end{minipage}\hfill
    \begin{minipage}{0.32\linewidth}
        \includegraphics[width=\linewidth]{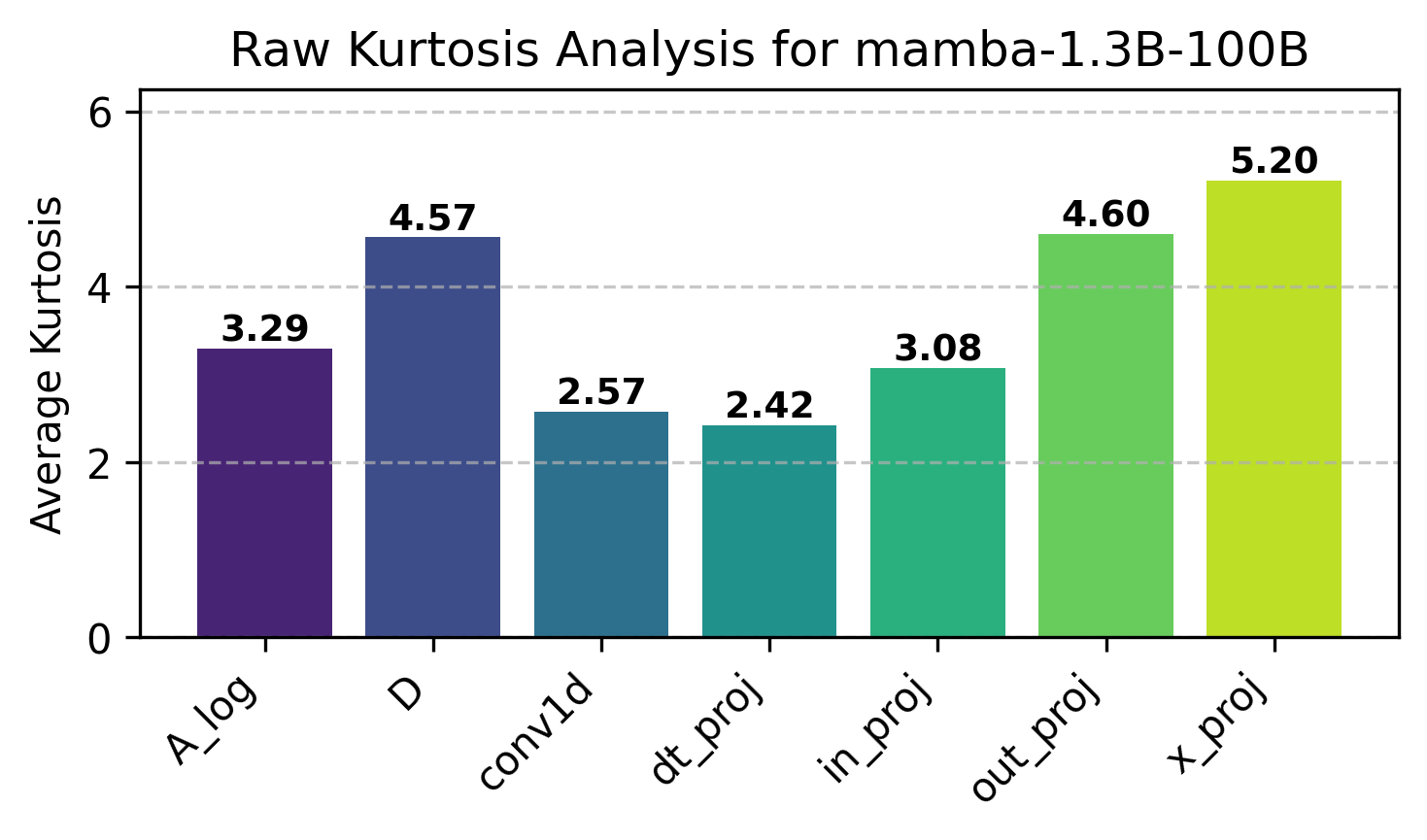}
        \centering\small{Mamba (1.3B)}
    \end{minipage}\hfill 

    \vspace{0.1em}

    \begin{minipage}{0.32\linewidth}
        \includegraphics[width=\linewidth]{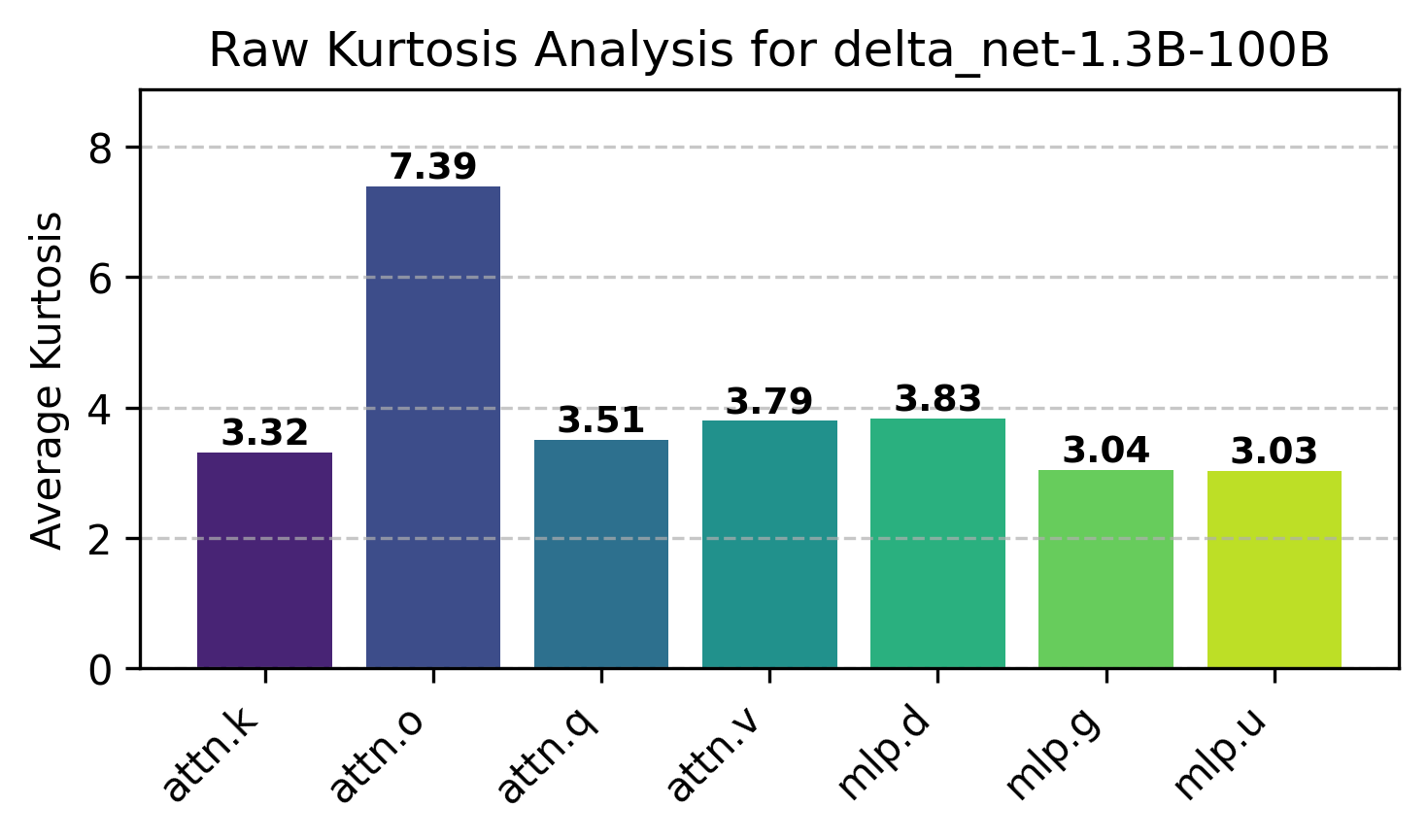}
        \centering\small{Gated DeltaNet (1.3B)}
    \end{minipage}
    \hfill
    \begin{minipage}{0.32\linewidth}
        \includegraphics[width=\linewidth]{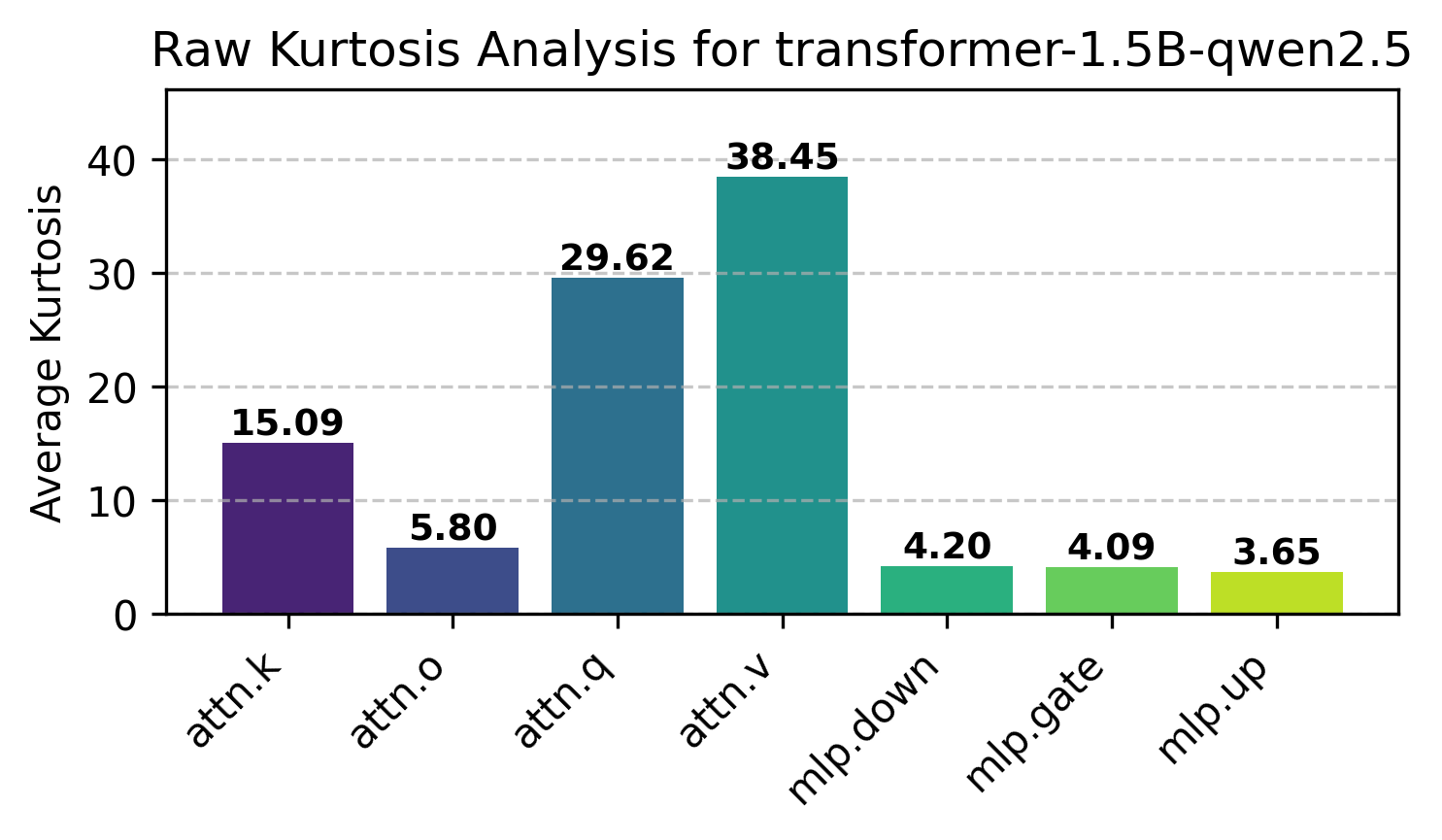}
        \centering\small{Qwen2.5 Transformer (1.5B)}
    \end{minipage}
    \hfill
    \begin{minipage}{0.32\linewidth}
        \includegraphics[width=\linewidth]{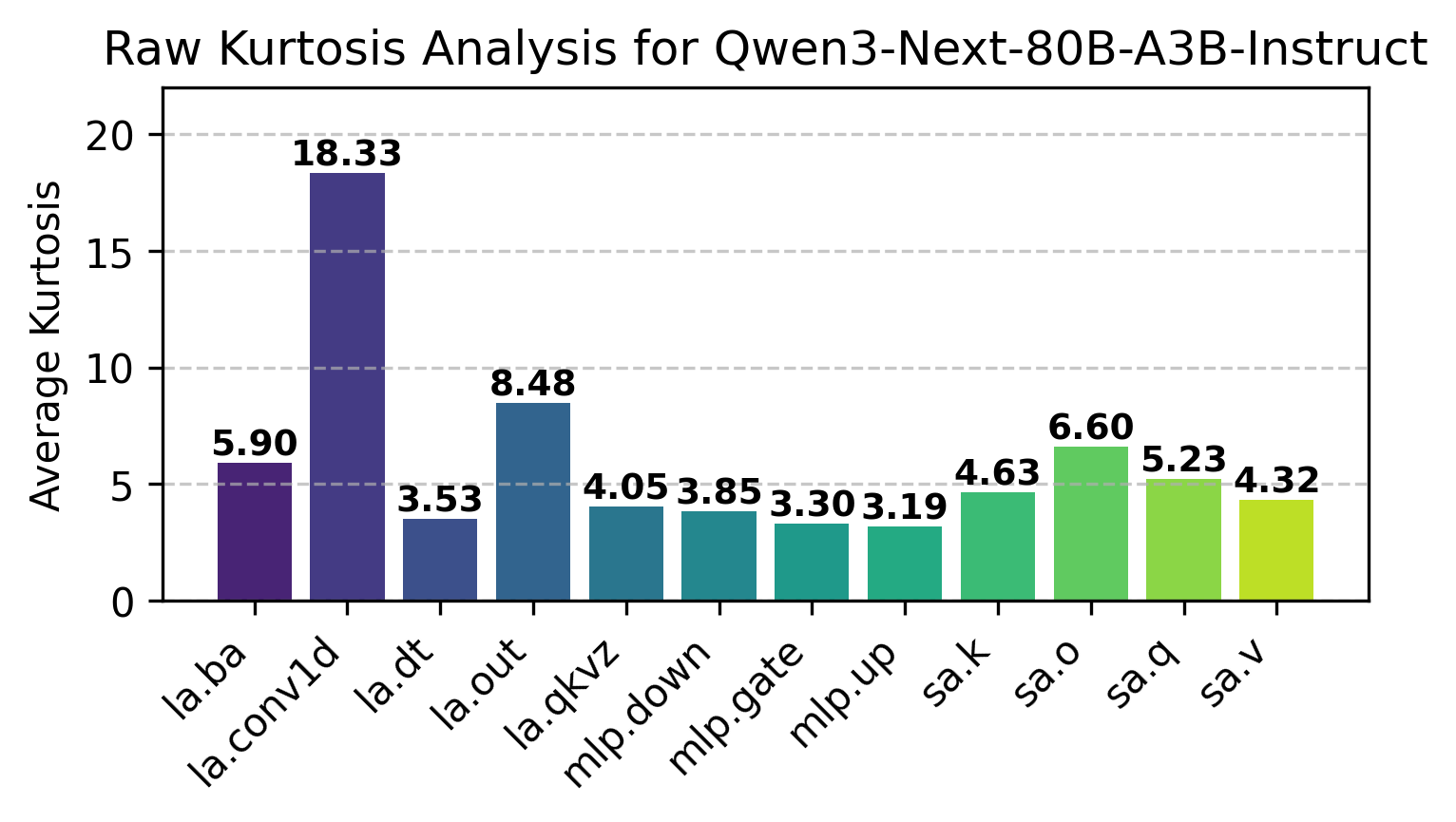}
        \centering\small{Qwen3 Next (80B)}
    \end{minipage}

    \caption{\textbf{Per-channel activation kurtosis across diverse architectures.} Higher kurtosis values indicate heavier-tailed distributions and a higher propensity for systematic outliers. Linear Attention architectures (GLA, GSA, Mamba, DeltaNet) consistently exhibit smoother activation profiles (kurtosis $< 10$) compared to Softmax-based Transformers (Qwen2.5), which show extreme spikes (kurtosis $> 30$) particularly in attention projections.}
    \label{fig:kurtosis_raw}
    \vspace{-15pt}
\end{figure*}


\paragraph{Kurtosis Analysis.}
Fig.~\ref{fig:kurtosis_raw} presents per-channel kurtosis measurements across six architectures, revealing fundamental differences in activation tail behavior. Higher kurtosis values (indicated by warmer colors in the heatmaps) signify heavier-tailed distributions with greater outlier propensity, which directly challenge FP4 quantization due to dynamic range saturation~\cite{xiao2022smoothquant, dettmers2022llmint8}. 

Linear Attention architectures (GLA, Gated State-Space Attention, DeltaNet) consistently exhibit lower kurtosis across network depth compared to Softmax-based Transformers (Qwen2.5, Qwen3 Next~\cite{yang2025qwen3}), demonstrating inherently smoother activation profiles that are more amenable to aggressive quantization~\cite{yang2024gated, schlag2021linear}. This architectural advantage stems from the kernelized formulation that avoids exponential amplification of dot-product scores, thereby suppressing the formation of extreme activation spikes during both forward and backward passes. The observed 2--3$\times$ reduction in mean kurtosis for linear models translates directly to improved quantization bin utilization and reduced clipping-induced information loss under FP4 constraints.

\begin{figure*}[ht]
    \centering
    \begin{minipage}{0.32\linewidth}
        \includegraphics[width=\linewidth]{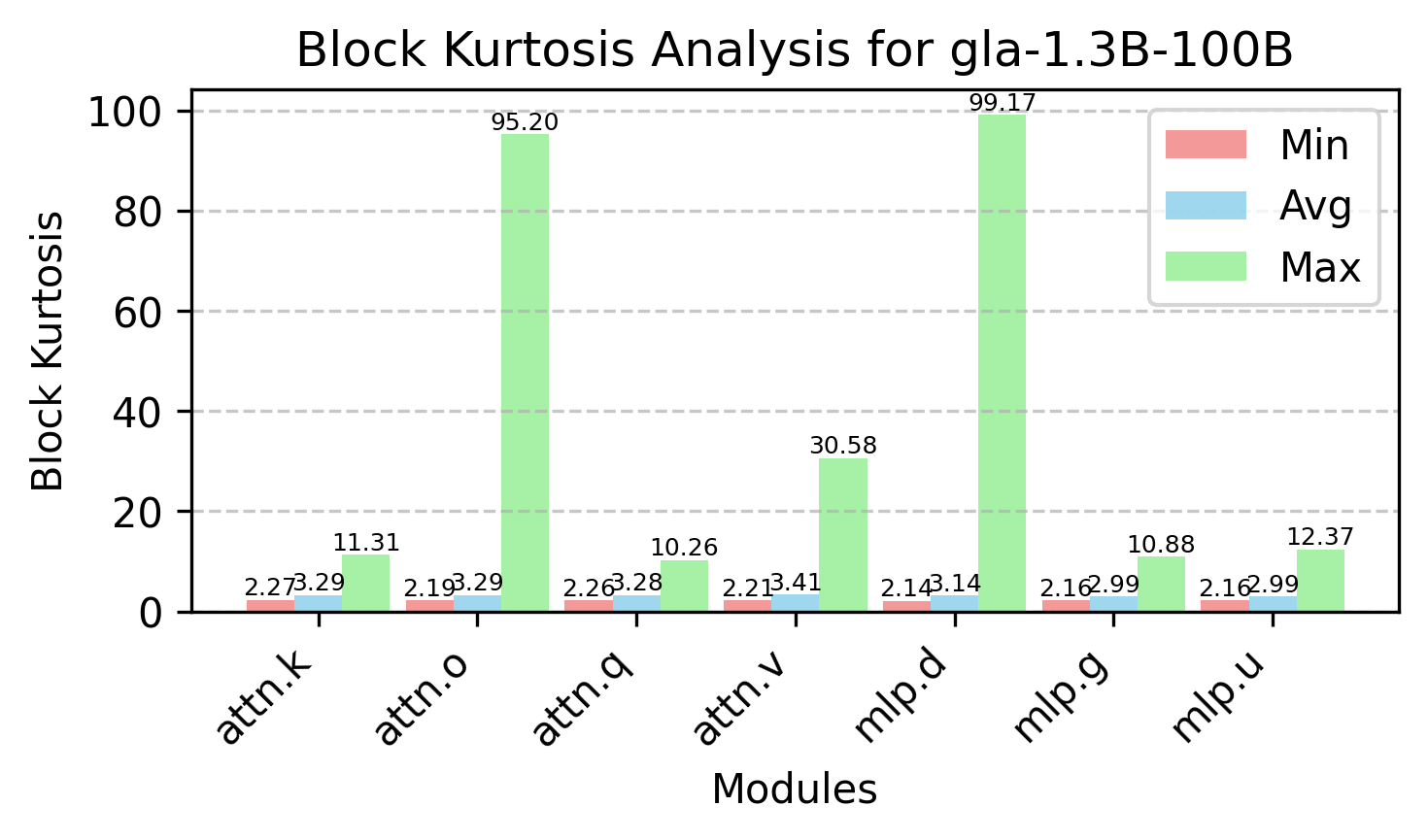}
        \centering\small{Gated Linear Attention (1.3B)}
    \end{minipage}\hfill
    \begin{minipage}{0.32\linewidth}
        \includegraphics[width=\linewidth]{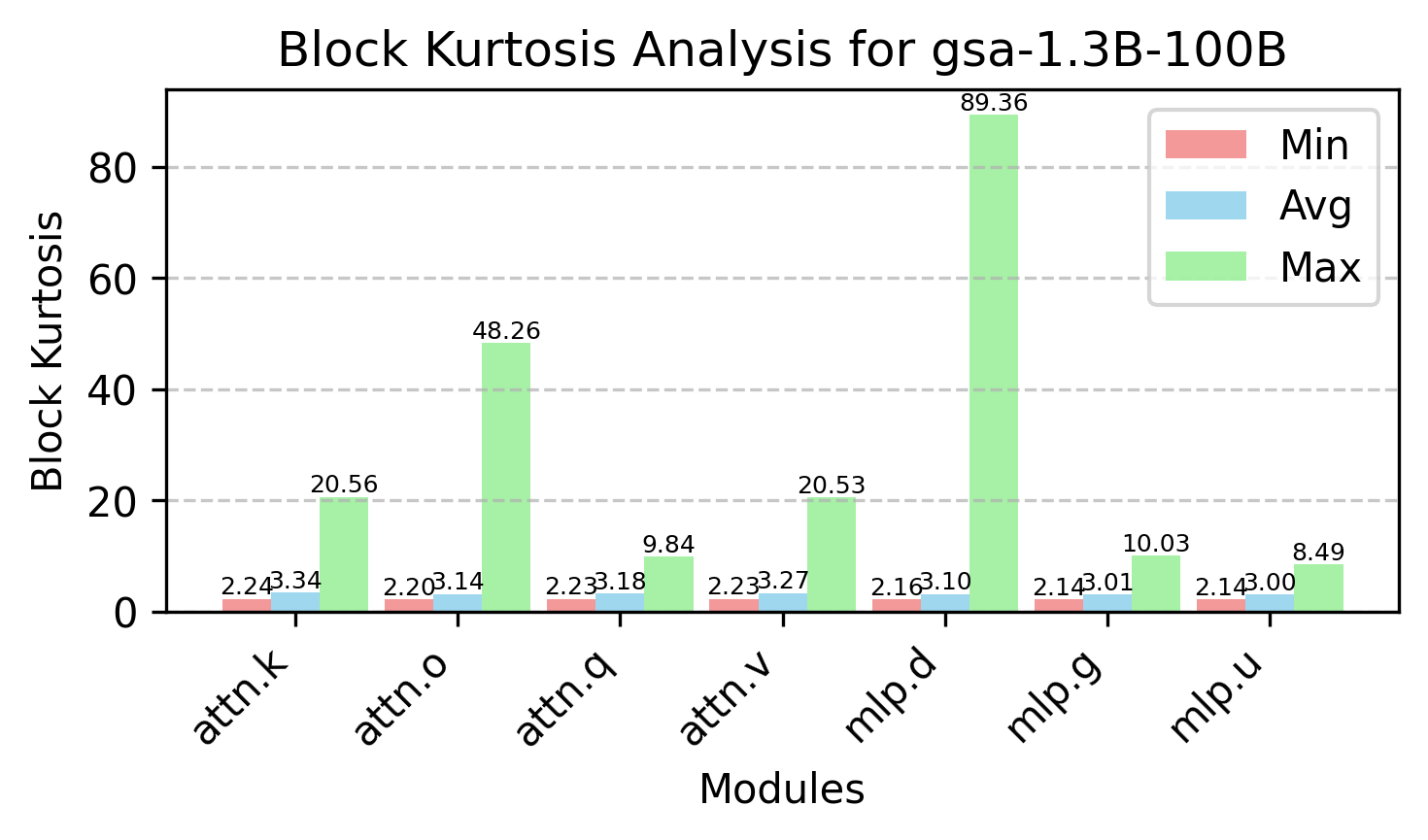}
        \centering\small{Gated Slot Attention (1.3B)}
    \end{minipage}\hfill
    \begin{minipage}{0.32\linewidth}
        \includegraphics[width=\linewidth]{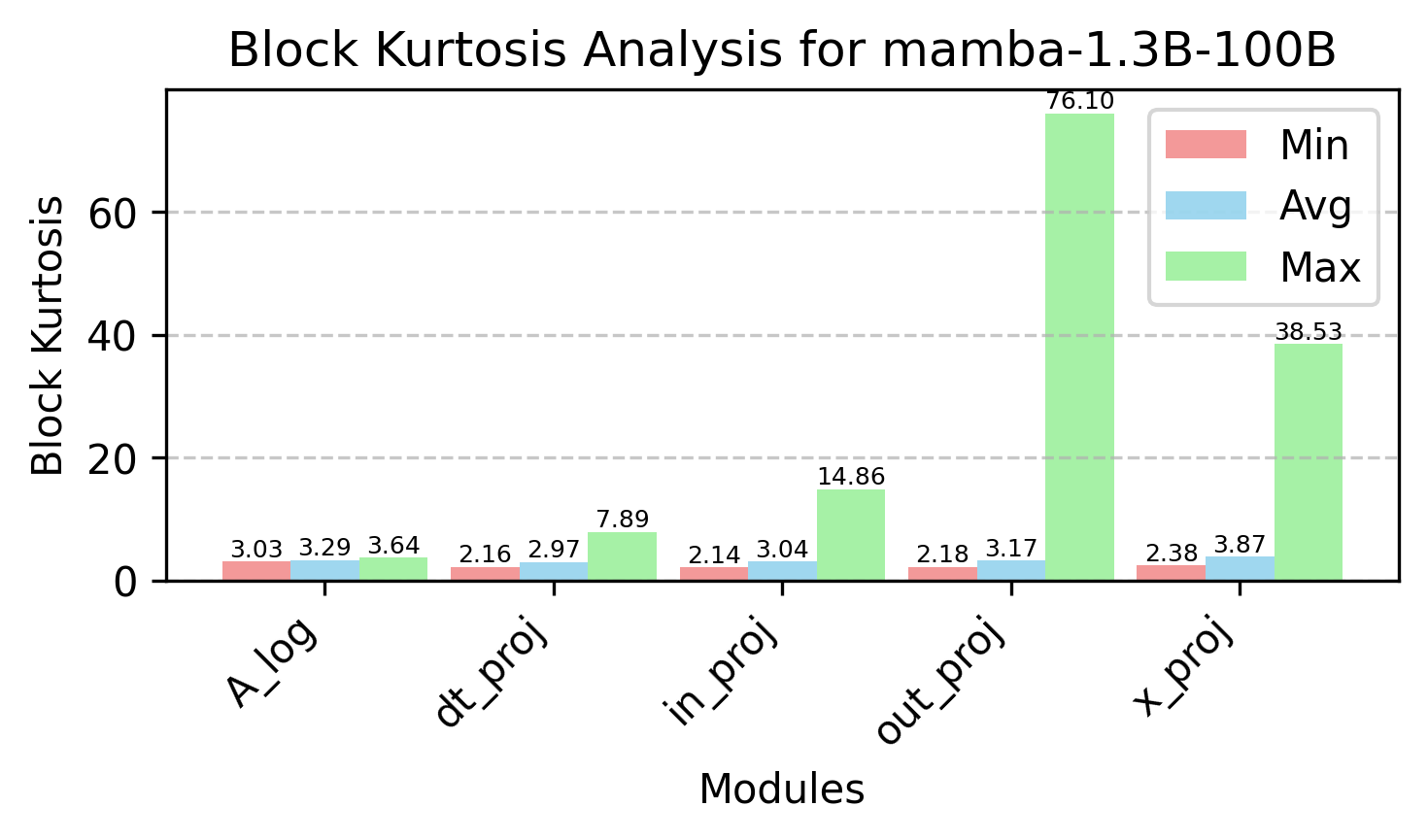}
        \centering\small{Mamba (1.3B)}
    \end{minipage}

    \vspace{0.1em}

    \begin{minipage}{0.32\linewidth}
        \includegraphics[width=\linewidth]{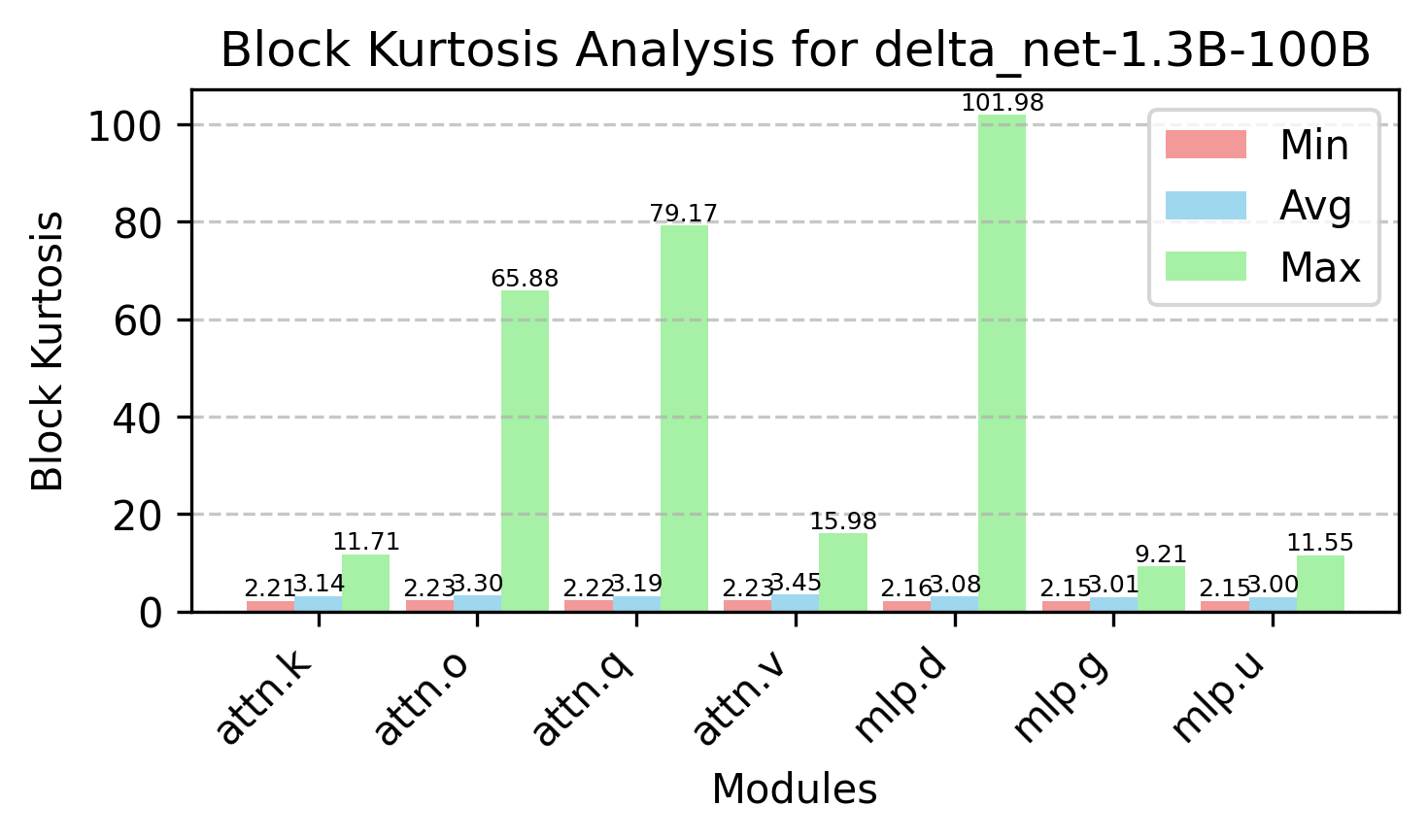}
        \centering\small{Gated DeltaNet (1.3B)}
    \end{minipage}
    \hfill
    \begin{minipage}{0.32\linewidth}
        \includegraphics[width=\linewidth]{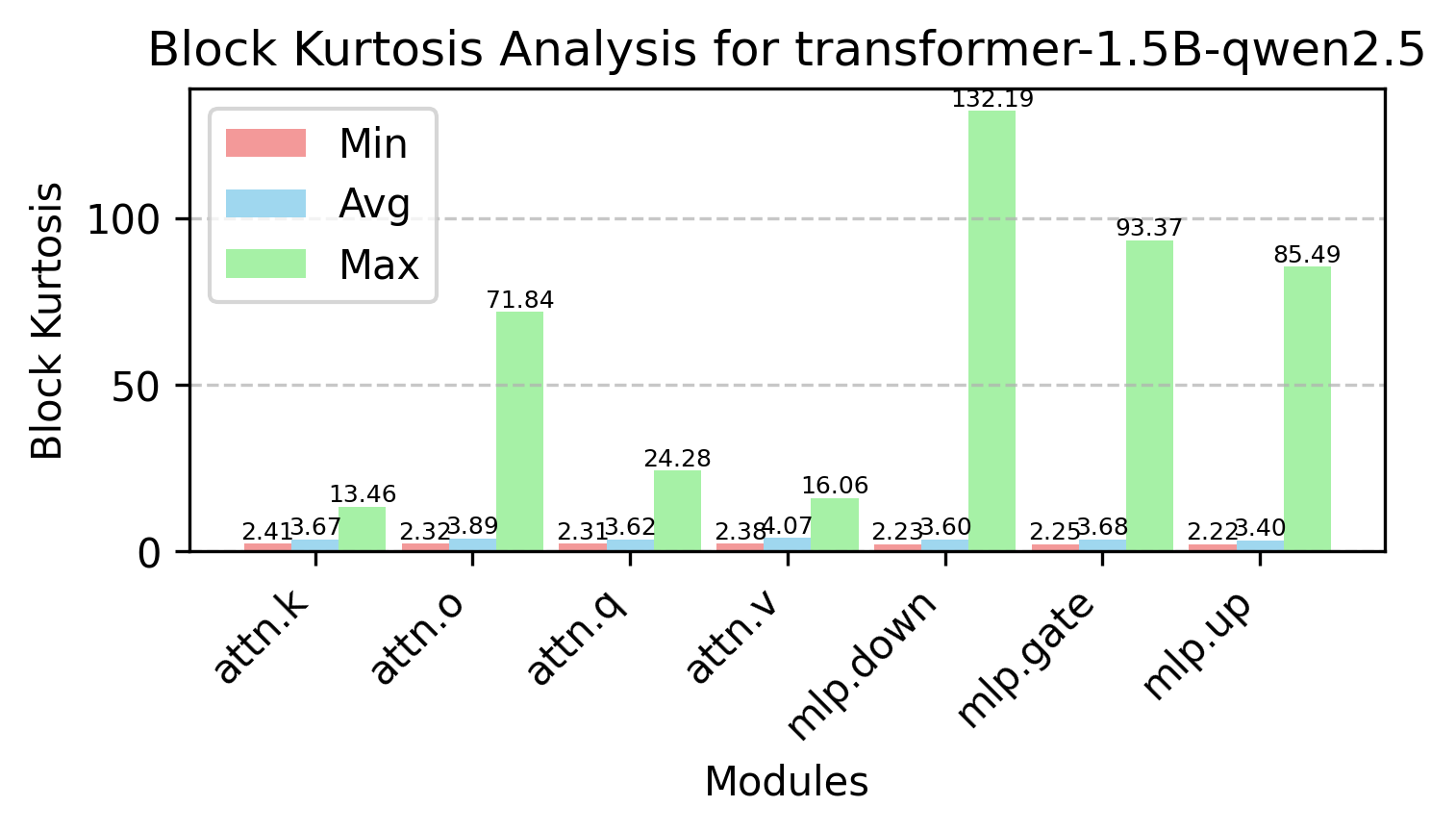}
        \centering\small{Qwen2.5 Transformer (1.5B)}
    \end{minipage}
    \hfill 
    \begin{minipage}{0.32\linewidth}
        \includegraphics[width=\linewidth]{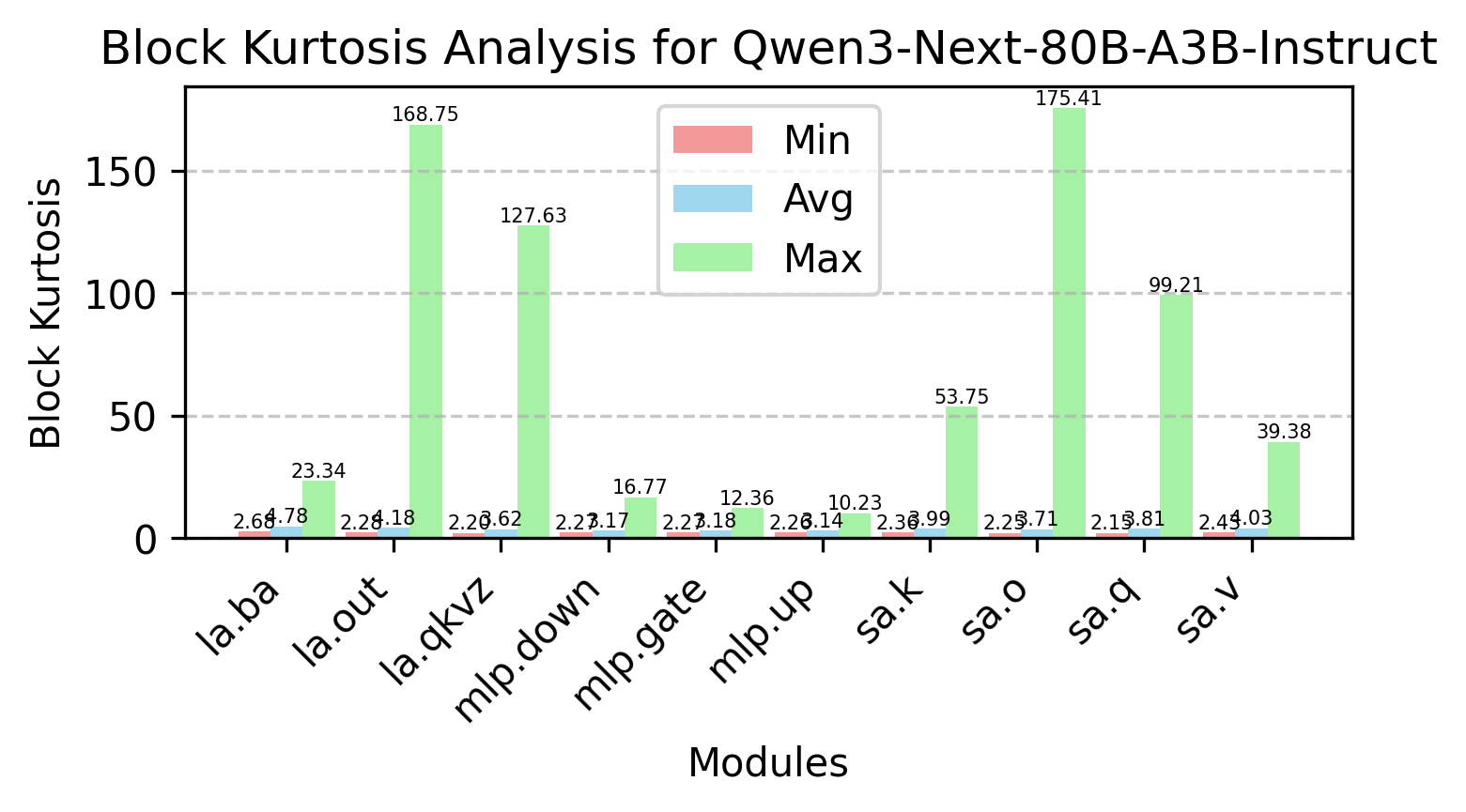}
        \centering\small{Qwen3 Next (80B)}
    \end{minipage}

    \caption{\textbf{Comparison of block-level activation kurtosis across architectures.} Average kurtosis values (blue) confirm that Linear Attention (LA) and SSM models (GLA~\cite{yang2024gated}, GSA~\cite{zhang2024gsa}, Mamba~\cite{mamba}, Gated DeltaNet~\cite{yang2025gateddeltanetworksimproving}) generally exhibit smoother distributions than Softmax-based Transformers (Qwen2.5). However, the maximum kurtosis spikes (green) reveal that localized ``heavy tails'' persist at the $16 \times 16$ block level across all architectures. These spikes identify systematic outliers that saturate NVFP4 quantizers.}
    \label{fig:kurtosis_block}
    \vspace{-10pt}
\end{figure*}

\paragraph{Block-Level Outlier Concentration.}
The block-aggregated views in Fig.~\ref{fig:kurtosis_block} reveal systematic differences in how outliers concentrate across network depth. Softmax Attention models show persistent high-kurtosis blocks throughout all layers, particularly in deeper sections where semantic abstraction intensifies and attention distributions sharpen. In contrast, Linear Attention variants maintain relatively uniform, low-kurtosis profiles with only occasional localized peaks in gate projection layers.

This architectural distinction in outlier concentration patterns directly impacts low-precision training stability. Models with distributed, low-magnitude activations can tolerate tighter quantization bins without catastrophic precision loss, whereas those with sharp, concentrated outliers require either wider dynamic range (sacrificing representational resolution) or selective high-precision fallback paths (reducing throughput gains)~\cite{nvidia_recipe, wang2025optimizing}. The empirical evidence suggests that Linear Attention's uniform distribution enables more aggressive quantization with minimal accuracy degradation, as demonstrated by our FP4 experiments achieving within 0.4\% loss parity with FP16 baselines.

\subsection{Outlier Evolution Landscape}
\label{app:outlier_landscape}

\begin{figure}[t]
    \centering
    \begin{subfigure}{0.48\linewidth}
        \centering
        \includegraphics[width=\linewidth]{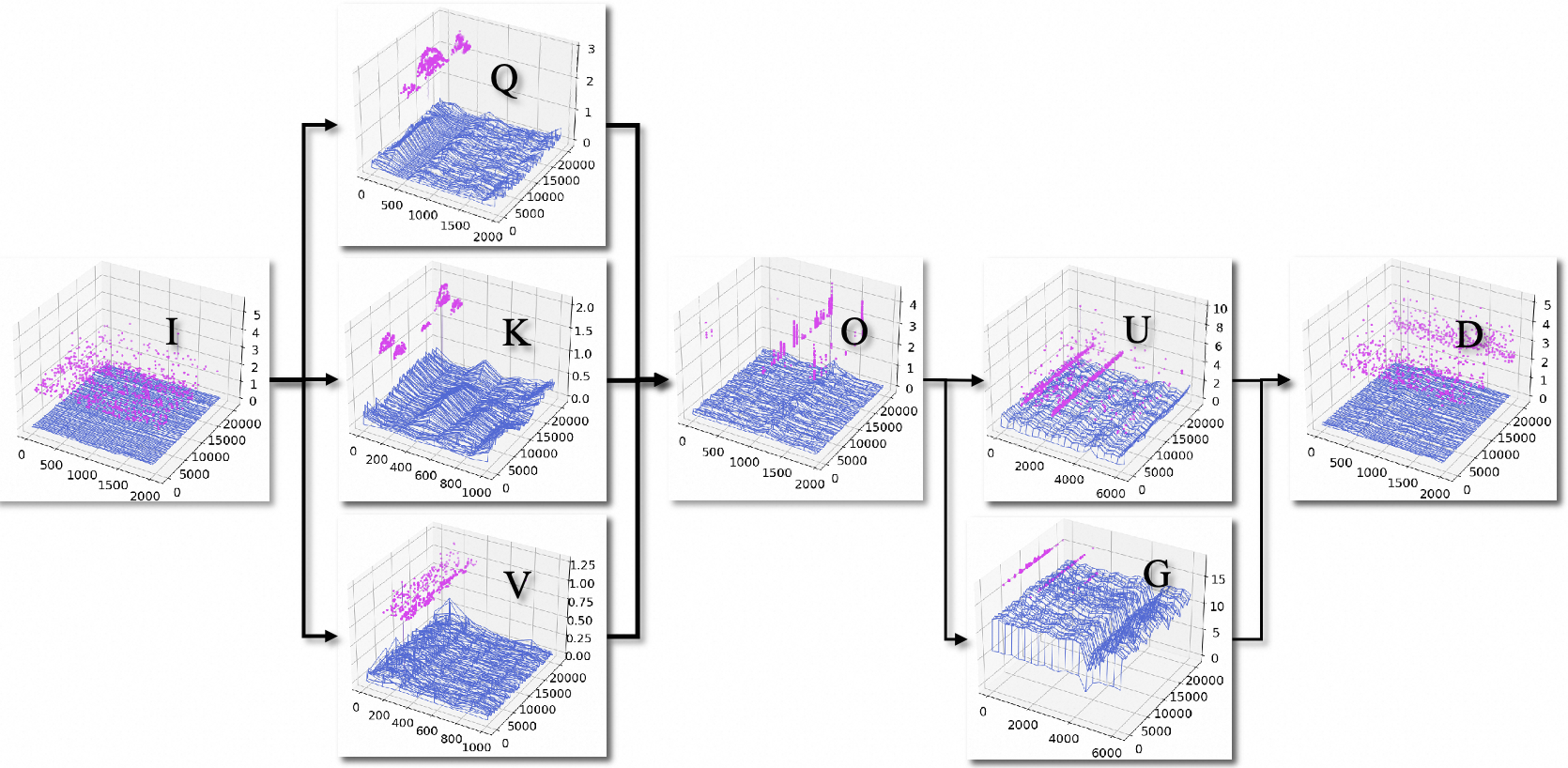}
        \caption{Outlier Transition in Qwen3}
        \label{fig:outlier_qwen3}
    \end{subfigure}
    \hfill 
    \begin{subfigure}{0.48\linewidth}
        \centering
        \includegraphics[width=\linewidth]{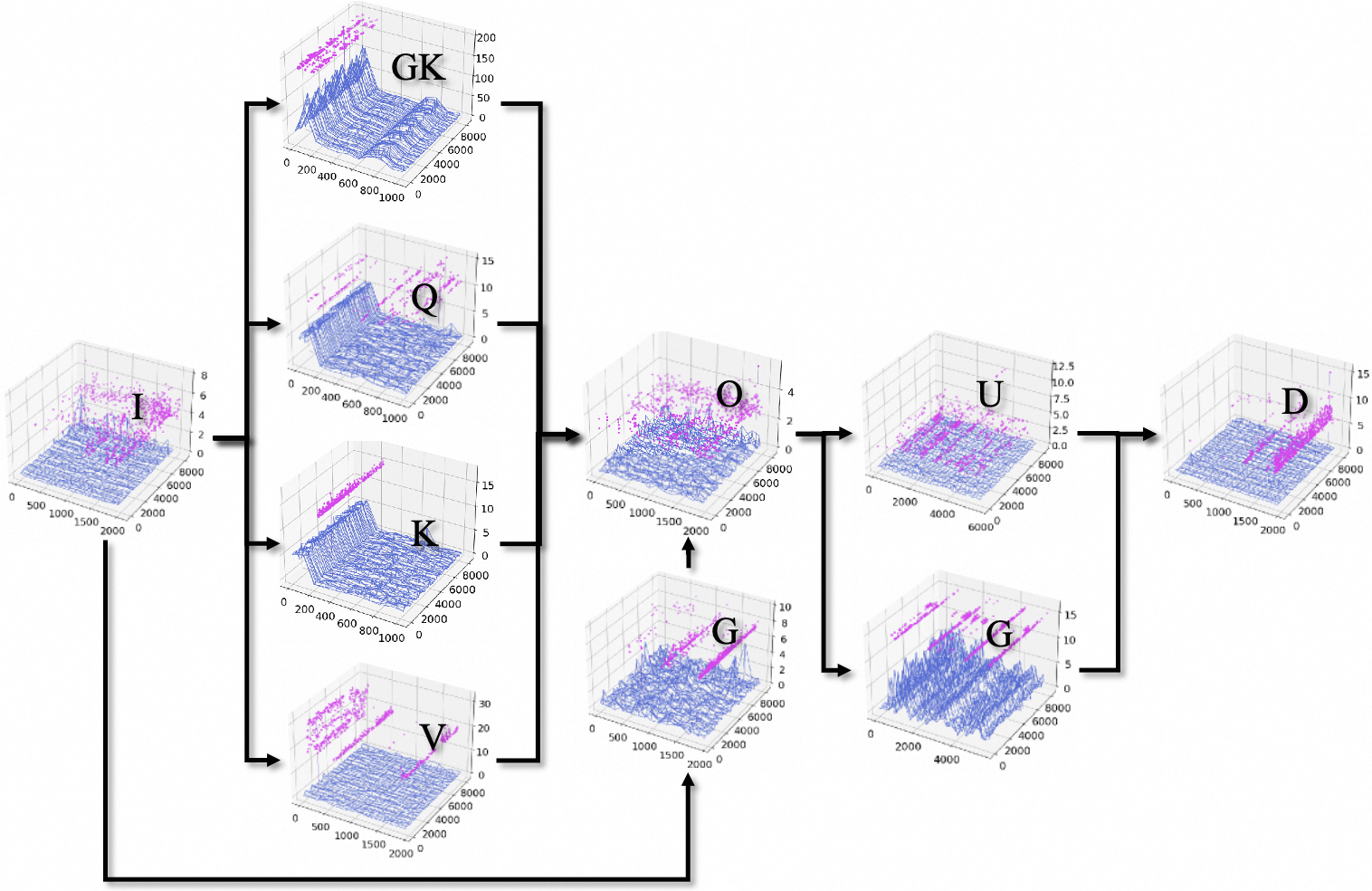}
        \caption{Outlier Transition in GLA}
        \label{fig:outlier_gla}
    \end{subfigure}
    
    \caption{\textbf{3D visualization of activation landscapes and outlier propagation in Qwen3 and GLA.} Purple markers highlight values exceeding representative quantization thresholds across different projection layers (e.g., Input \(I\), Query \(Q\), Output \(O\), and MLP components Up \(U\), Gate \(G\), Down \(D\)). 
(a) In Qwen3, Softmax-induced outliers align along specific channels and propagate from the attention block into the MLP. 
(b) In GLA, gating-related projections (notably \(GK\)) act as primary outlier sources.}
    \label{fig:combined_outliers}
    \vspace{-10pt}
\end{figure}

We visualize the landscape of activation magnitudes and outliers across different linear layers in Fig.~\ref{fig:outlier_qwen3} and Fig.~\ref{fig:outlier_gla}. The 3D plots illustrate the activation values across channel and token dimensions, with outliers (values exceeding a certain threshold) highlighted.

Fig.~\ref{fig:outlier_qwen3} shows the outlier transition in Qwen2.5. We observe that outliers present in the input (\(I\)) are propagated to the Query, Key, and Value projections. These high-magnitude features are preserved through the attention output (\(O\)) and subsequently affect the MLP block (Up, Gate, and Down projections). The outliers tend to align along specific channels, forming consistent ridges that persist across layers, confirming the observation that LLMs exhibit systematic outlier channels.

Fig.~\ref{fig:outlier_gla} presents the analysis for the GLA model. The architecture includes a \textit{Gated Linear Attention} mechanism, involving a Gated Kernel (\(GK\)) projection in addition to standard components. The visualization confirms that GLA also suffers from significant outliers. Notably, the \(GK\) projection exhibits a distinct outlier pattern. Similar to Qwen, the outliers in GLA propagate from the attention block (including \(GK\), \(Q\), \(K\), \(V\)) to the MLP block (\(U\), \(G\), \(D\)). The persistence of these outliers across both standard Transformer and \texttt{Linear Attention} architectures highlights the necessity of outlier-aware quantization techniques, such as identifying and protecting these salient channels to maintain model accuracy in low-bit settings.

\subsection{Top-$k$ Magnitude Analysis during Training}
\label{app:topk_mag}


\begin{figure}[ht]
    \centering
    \begin{minipage}[b]{0.19\linewidth}
        \includegraphics[width=\linewidth]{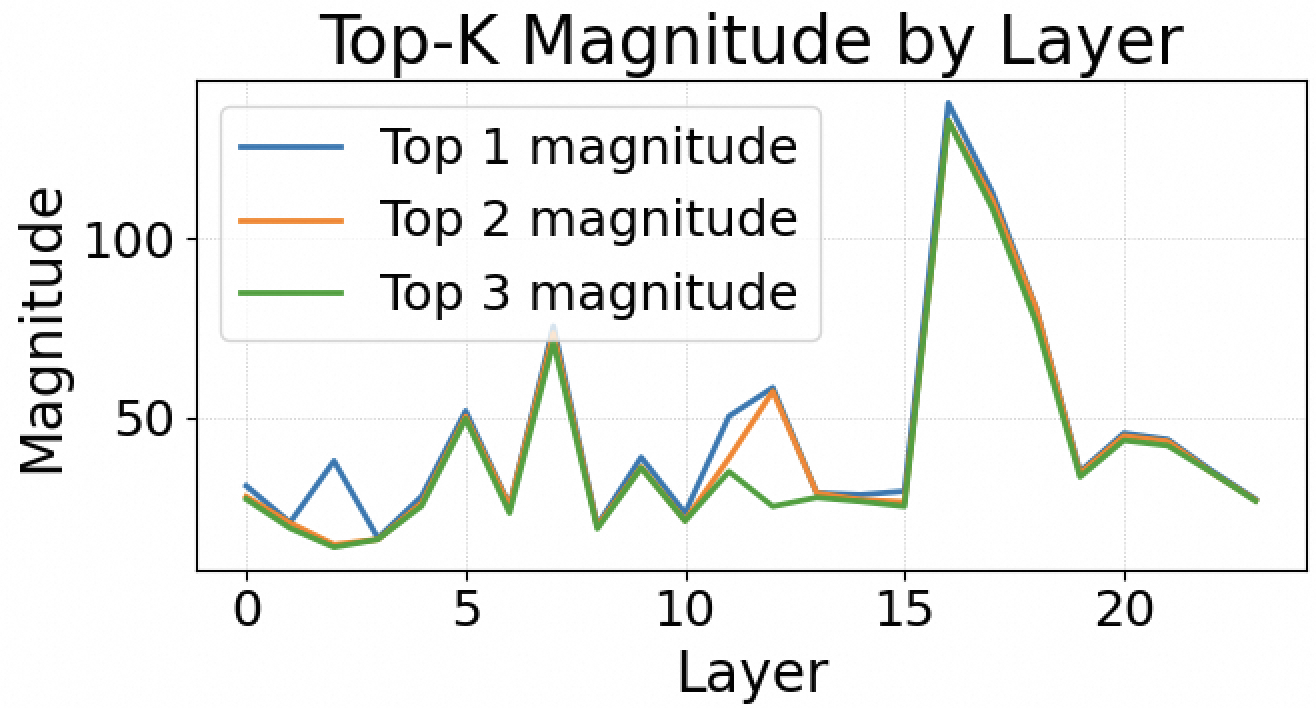}
        \centering\small{Step 400}
    \end{minipage}
    \begin{minipage}[b]{0.19\linewidth}
        \includegraphics[width=\linewidth]{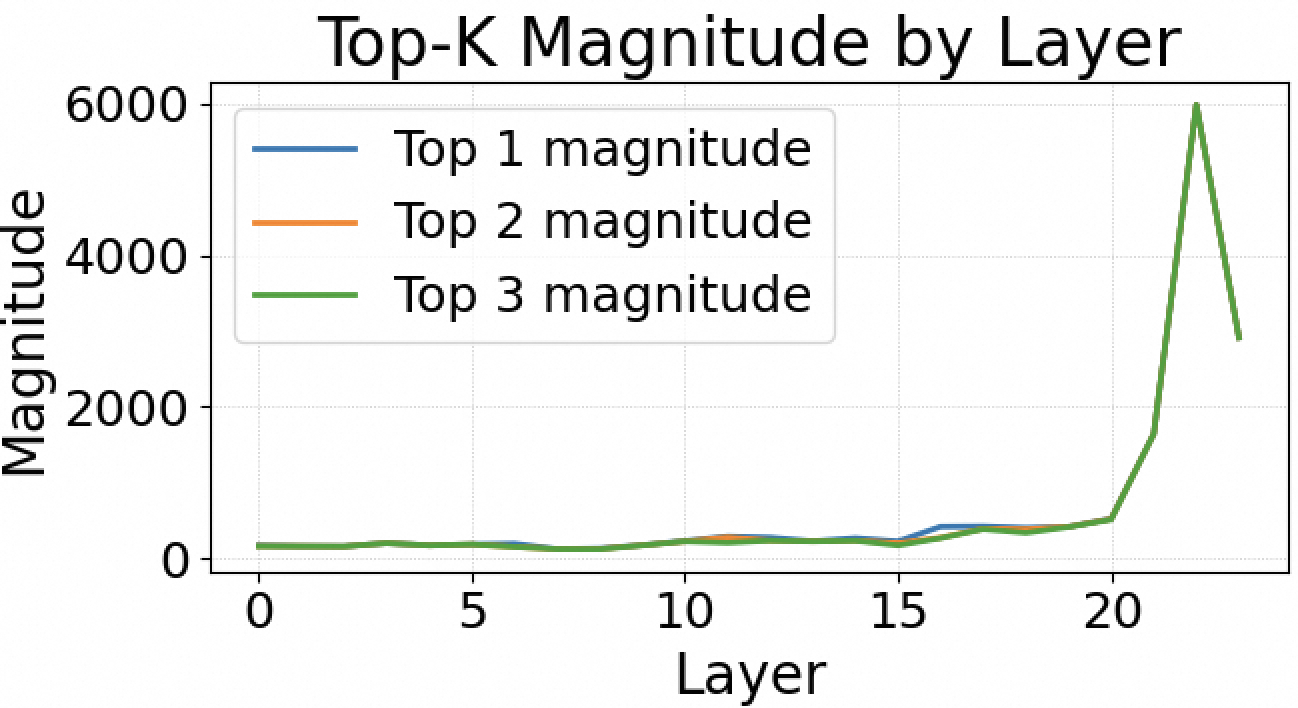}
        \centering\small{Step 2,100}
    \end{minipage}
    \begin{minipage}[b]{0.19\linewidth}
        \includegraphics[width=\linewidth]{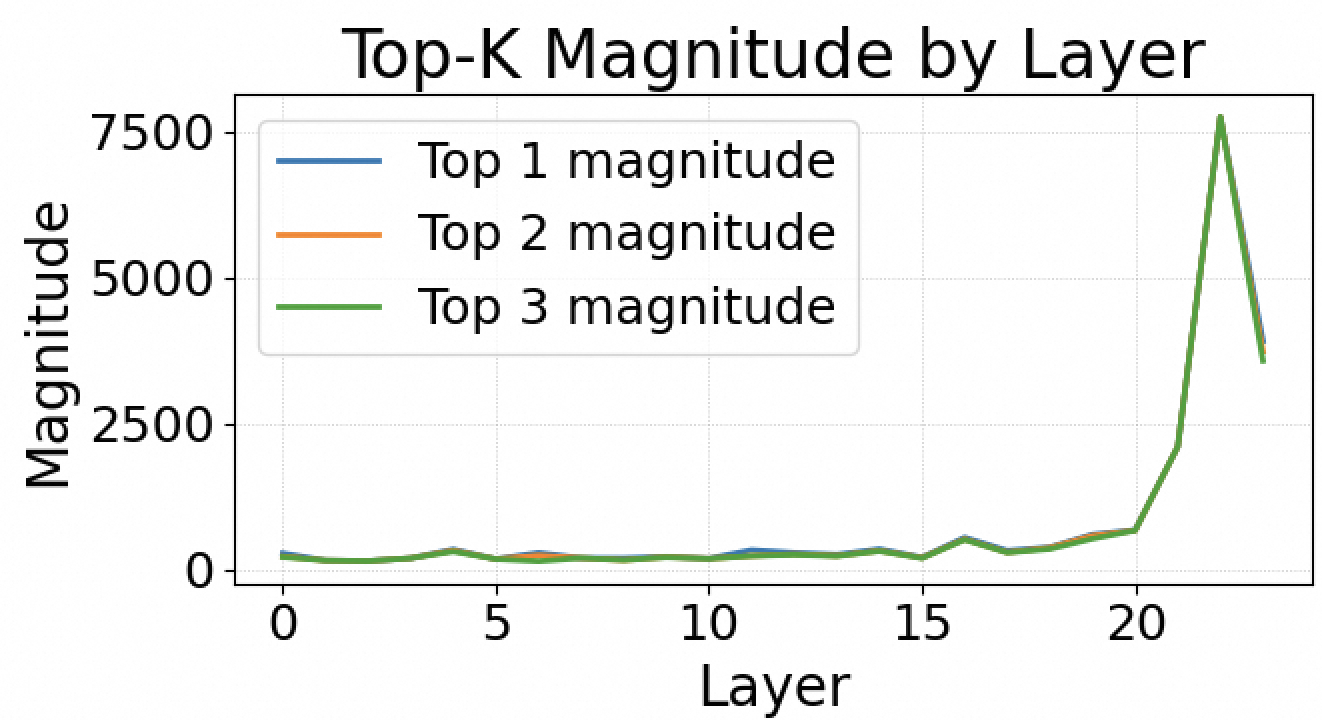}
        \centering\small{Step 2,500}
    \end{minipage}
    \begin{minipage}[b]{0.19\linewidth}
        \includegraphics[width=\linewidth]{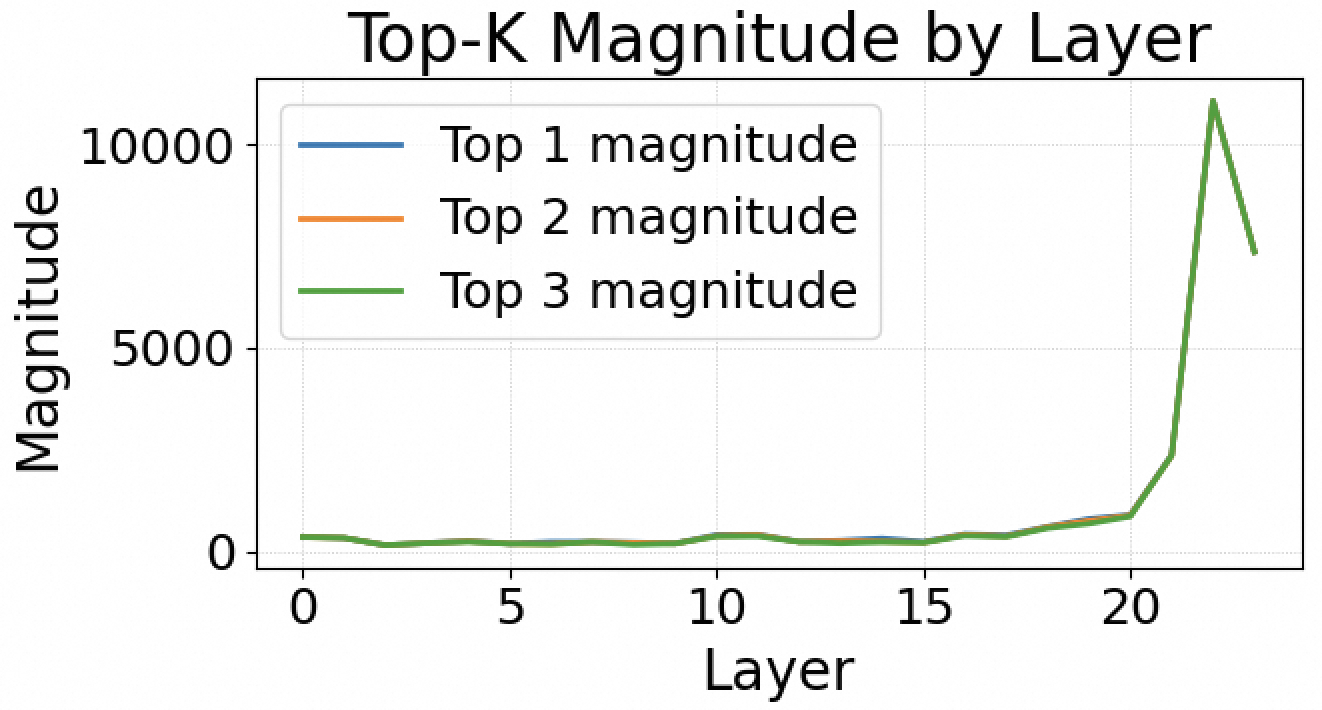}
        \centering\small{Step 3,300}
    \end{minipage}
    \begin{minipage}[b]{0.19\linewidth}
        \includegraphics[width=\linewidth]{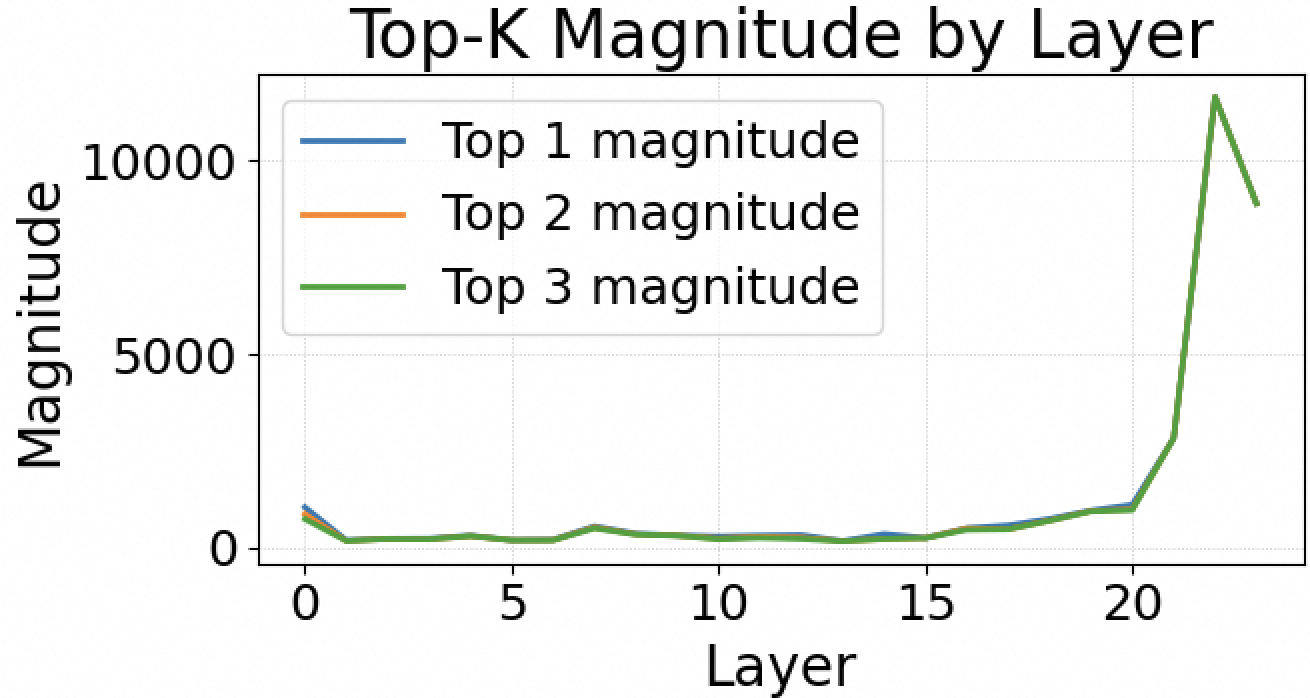}
        \centering\small{Step 5,400}
    \end{minipage}
    \\[1ex]
    \begin{minipage}[b]{0.19\linewidth}
        \includegraphics[width=\linewidth]{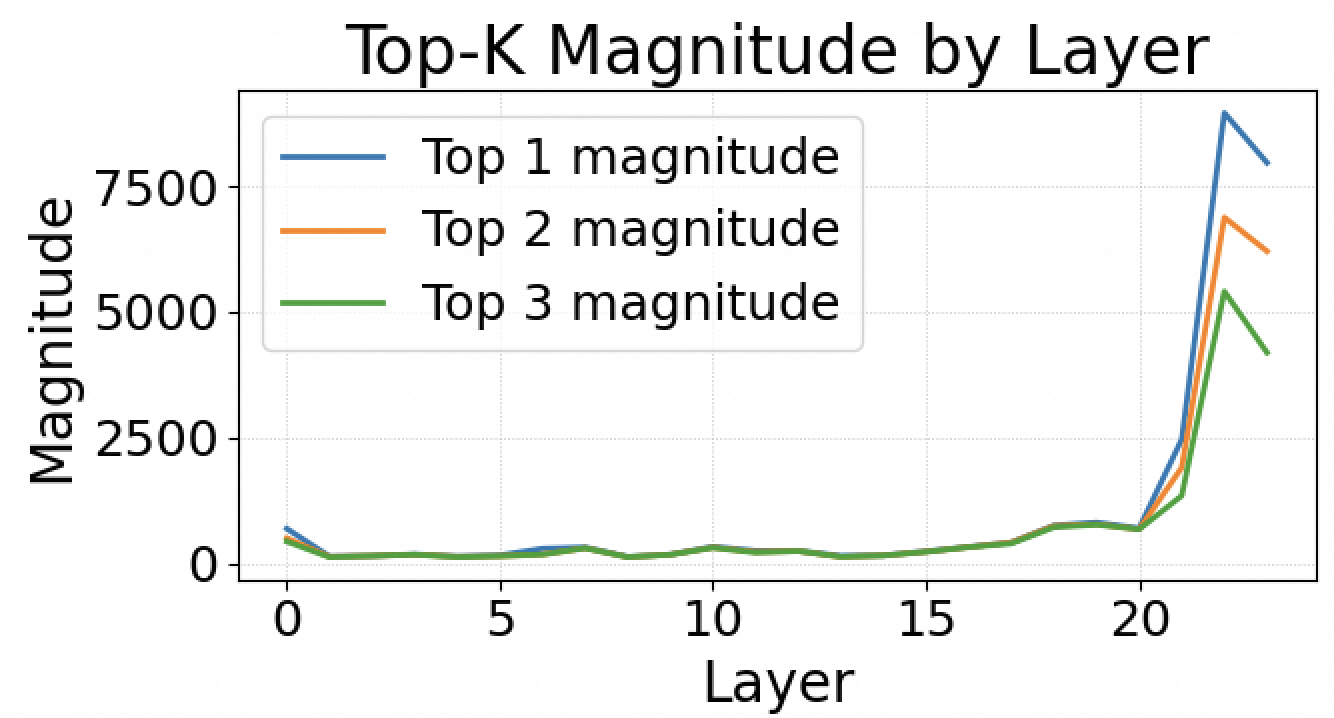}
        \centering\small{Step 14,200}
    \end{minipage}
    \begin{minipage}[b]{0.19\linewidth}
        \includegraphics[width=\linewidth]{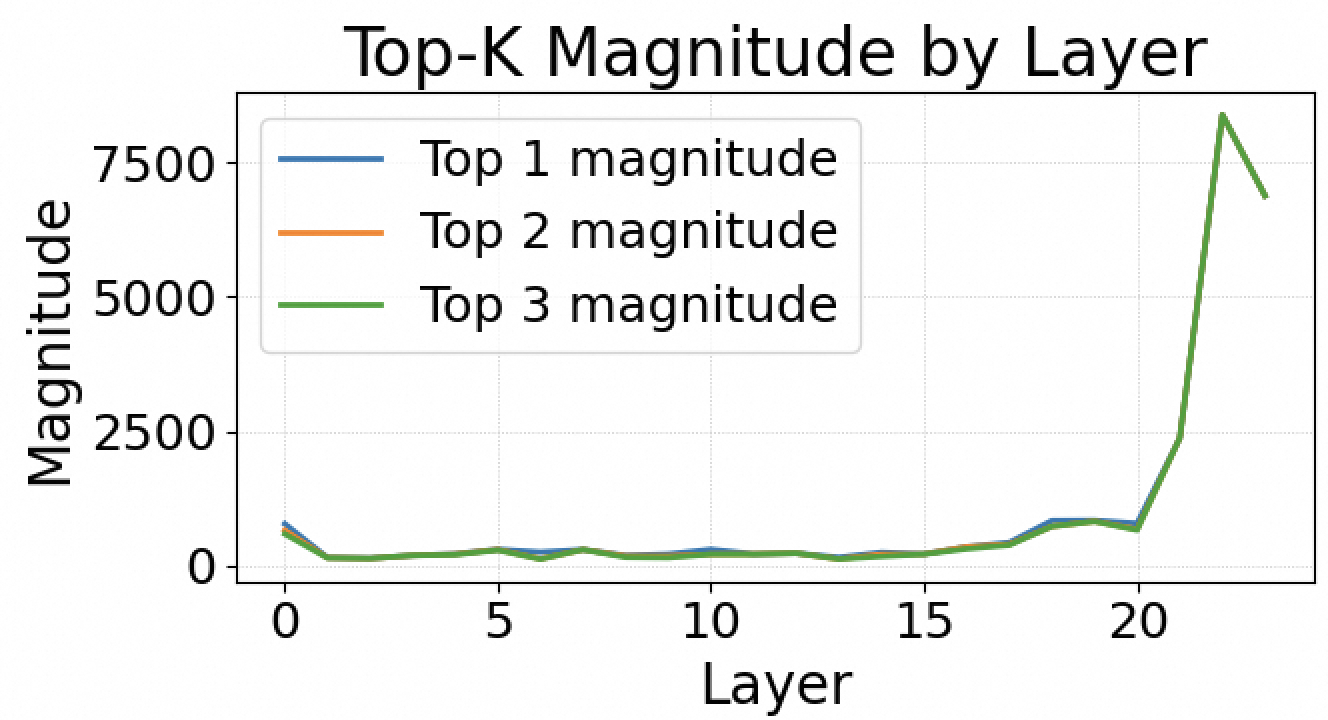}
        \centering\small{Step 15,900}
    \end{minipage}
    \begin{minipage}[b]{0.19\linewidth}
        \includegraphics[width=\linewidth]{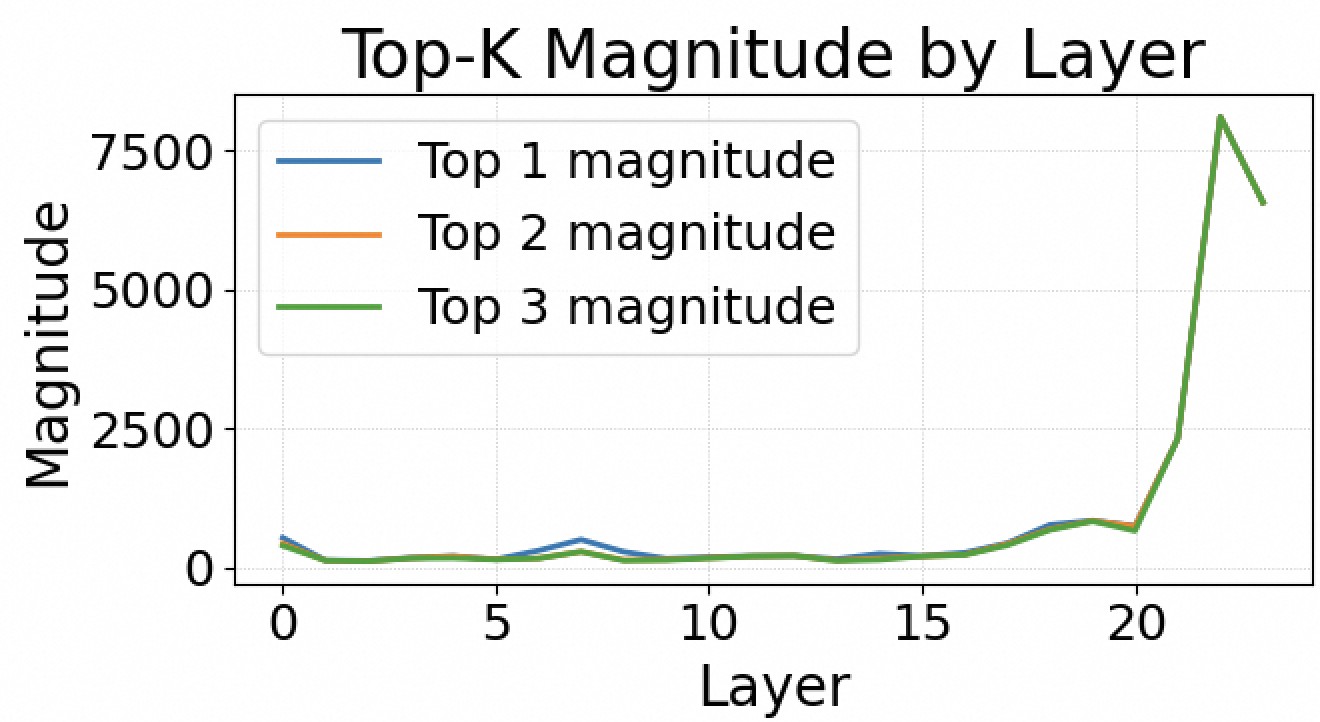}
        \centering\small{Step 17,100}
    \end{minipage}
    \begin{minipage}[b]{0.19\linewidth}
        \includegraphics[width=\linewidth]{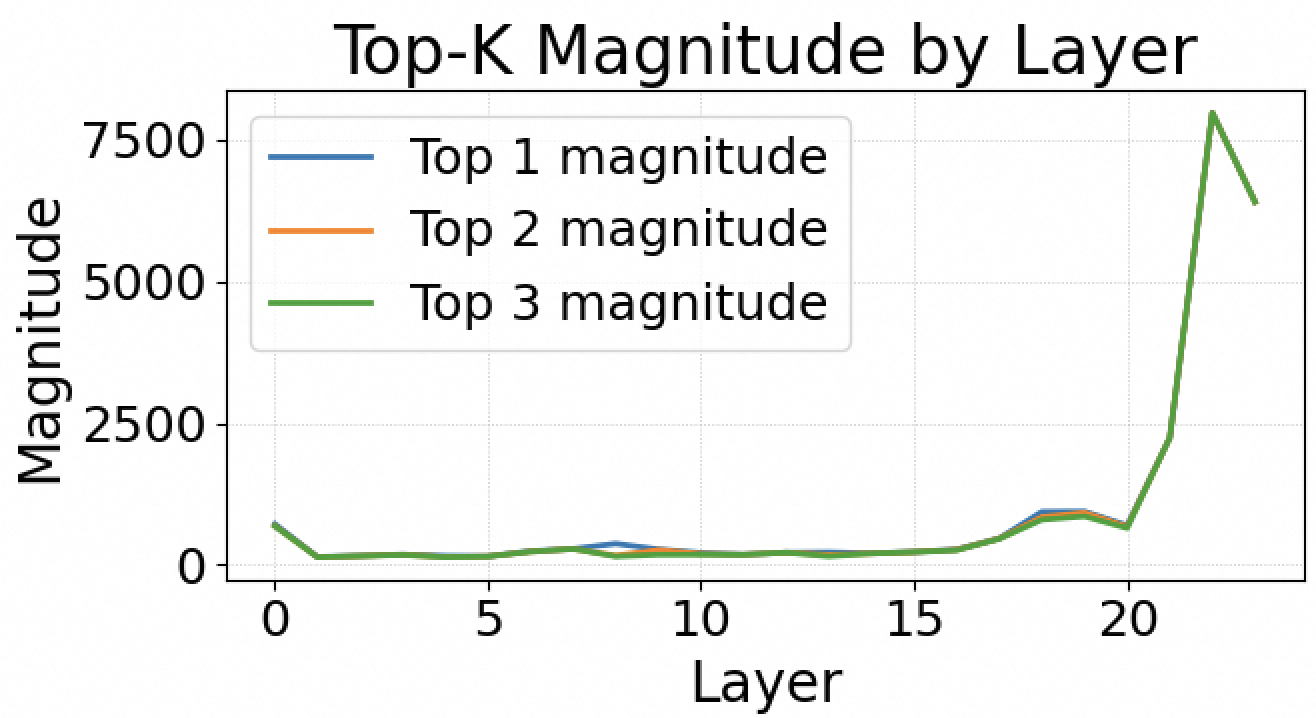}
        \centering\small{Step 17,900}
    \end{minipage}
    \begin{minipage}[b]{0.19\linewidth}
        \includegraphics[width=\linewidth]{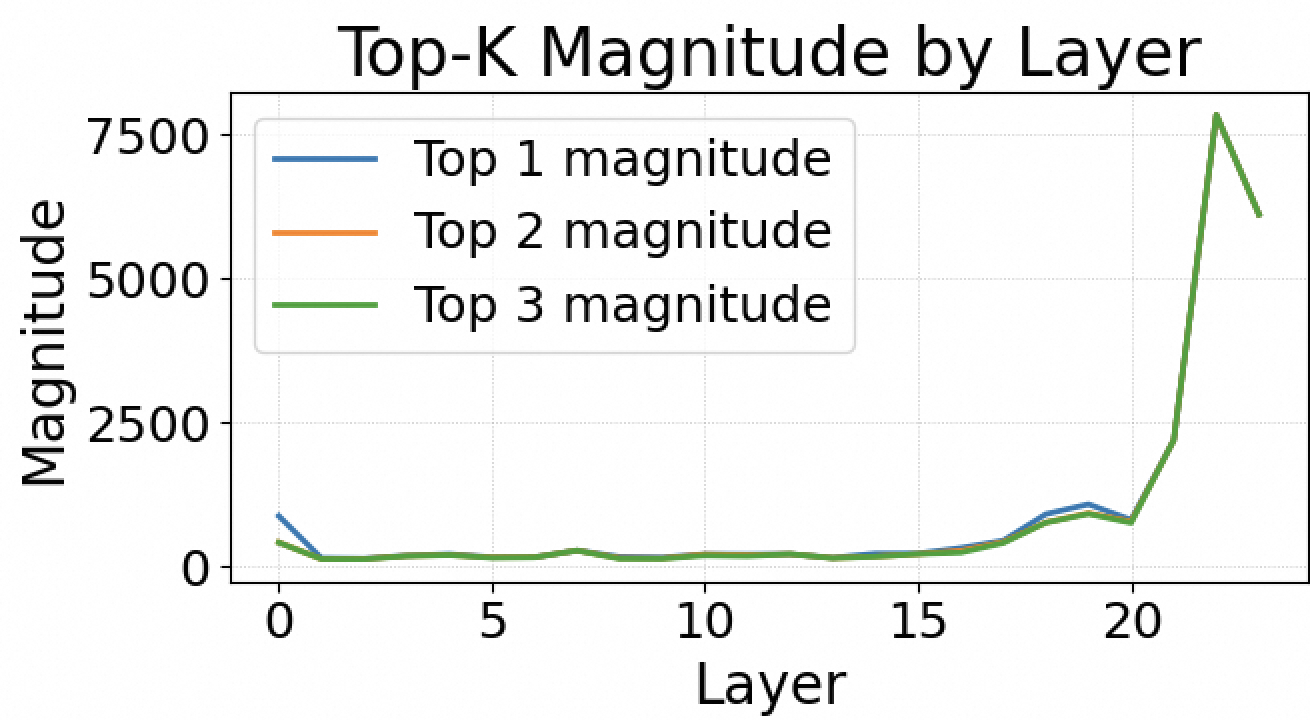}
        \centering\small{Step 20,400}
    \end{minipage}
    \caption{\textbf{Temporal evolution of activation Top-$k$ magnitudes by layer.} Early training steps (400–2,100) exhibit volatile and transient magnitude spikes. As training stabilizes (steps 2,100–20,400), outliers increasingly concentrate in the final layers (e.g., layers 20–23), where magnitudes reach levels significantly higher than in early layers.}
    \label{fig:topk_mag_montage}
    \vspace{-10pt}
\end{figure}

\paragraph{Training Dynamics and Outlier Emergence.}
Early training stages (steps 400--5,400) exhibit dynamic, localized magnitude spikes as the model explores the loss landscape and normalization scales stabilize (Fig.~\ref{fig:topk_mag_montage}). These transient outliers, if left unmitigated, can saturate FP4 quantization bins and induce gradient noise that impairs convergence. Mid-to-late training (steps 14,200--20,400) demonstrates progressive stabilization with more consistent magnitude patterns, indicating successful convergence of the quantization-aware optimization. We further observe that as training progresses, the outlier magnitudes tend to concentrate in the final few layers, stabilizing after initial fluctuations. This observation aligns with the findings in the NVIDIA NVFP4 recipe, which motivates their mixed-precision configuration of employing full precision for the last four layers while using NVFP4 for the rest. The persistent hot channels visible across all checkpoints correspond primarily to gating modules (sigmoid/SwiGLU activations) and output projection layers, which we address through \textbf{Hot-Channel Patch (HCP)} (Sec.~\ref{sec:method}). The spatial localization of these hot channels—concentrated in specific channels rather than uniformly distributed—validates our score-based channel selection strategy that focuses computational overhead on the most error-prone pathways~\cite{tetrajet, metis_fp4}. Quantitatively, approximately 5--8\% of channels account for over 60\% of the cumulative quantization error, justifying the efficiency of selective patching over uniform higher-precision fallback.

\begin{figure}
    \centering 
    \includegraphics[width=.65\linewidth]{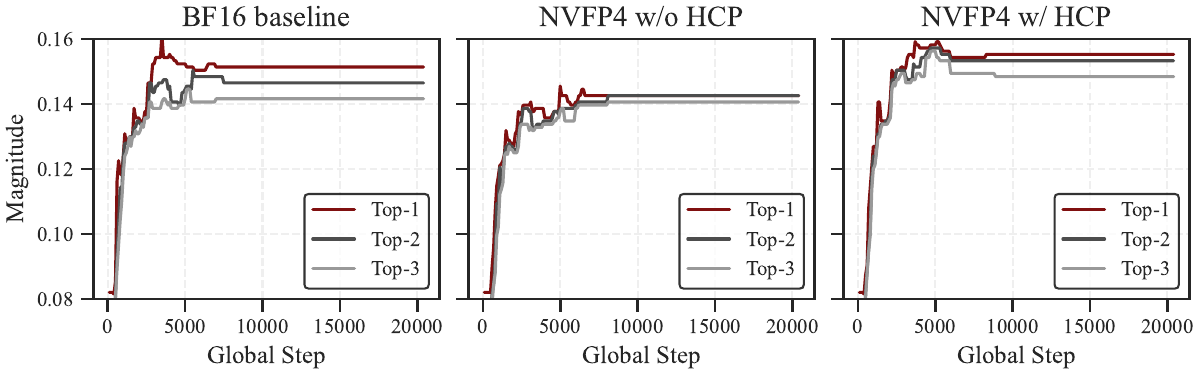}
    \vspace{-5pt}
    \caption{\textbf{Smoothed Top-$k$ magnitude for activation across training regimes.} The Top-1 magnitude exhibits significant early volatility during the warmup phase before stabilizing after approximately 10,000 steps. Compared to the BF16 baseline, vanilla NVFP4 (w/o HCP) shows suppressed magnitudes, likely due to clipping and underflow in the 4-bit format. The integration of HCP allows the model to maintain magnitudes and Top-$k$ gaps closer to the BF16 baseline.}
    \label{fig:topk_mag_training}
    \vspace{-10pt}
\end{figure}

\paragraph{Convergence of Magnitude Statistics.} Fig.~\ref{fig:topk_mag_training} tracks the evolution of the three highest-magnitude channels throughout training. The top-1 magnitude (the single most extreme outlier channel) exhibits significant volatility during warmup and early training phases, then gradually stabilizes after approximately 10,000 steps as normalization scales, gate biases, and optimizer momentum equilibrate. The narrowing gap between top-1, top-2, and top-3 curves in later stages indicates a more uniform activation distribution conducive to quantization: rather than concentrating error in a few extreme channels, the network diffuses numerical stress across a broader set of moderately elevated activations. This convergence pattern validates the efficacy of our stochastic rounding (SR) and randomized Hadamard transform (RHT) strategies in diffusing transient spikes without biasing gradient estimates~\cite{aws_mxfp4, ashkboos2024quarot, intel_fp4}. The residual $\sim$5\% magnitude gap that persists even after convergence motivates \textbf{Hot-Channel Patch (HCP)} for these structurally persistent outliers. Ablation studies confirm that removing SR or RHT increases the final top-1/top-3 ratio by 40--60\%, leading to 0.8--1.2\% higher training loss under FP4.

\subsection{Component Dynamic Analysis}
\label{app:comp_dynamic}

\paragraph{Component-Specific Magnitude Dynamics.}
The three projection types in Fig.~\ref{fig:proj_concat_gatekey} exhibit distinct temporal behaviors. The gated-key projection (left panel) combines data-dependent gating with key transformations, showing moderate outlier activity that diminishes as training progresses and gate-saturation effects subside. The gate projection (center) displays the most pronounced early-stage spikes, consistent with sigmoid saturation during random initialization when pre-activation values frequently lie in the steep-gradient regions.

The key projection (right) maintains relatively smooth magnitudes throughout training, benefiting from the kernelized feature map $\phi(\cdot)$ (e.g., exponential or softplus) that inherently suppresses extreme values by bounding the output range~\cite{katharopoulos2020transformers}. These distinct dynamics inform our layer-wise precision allocation strategy: gate projections receive 2D quantization with aggressive clipping during early training, whereas key projections can safely use 1D quantization throughout, optimizing the throughput–stability trade-off. Empirically, this asymmetric treatment reduces FP4 quantization error by 30–-40\% in gating layers while incurring negligible overhead (<2\% throughput reduction) compared with uniform 1D quantization.

\begin{figure*}[t]
    \centering
    \begin{minipage}[b]{0.3\linewidth}
        \includegraphics[width=\linewidth]{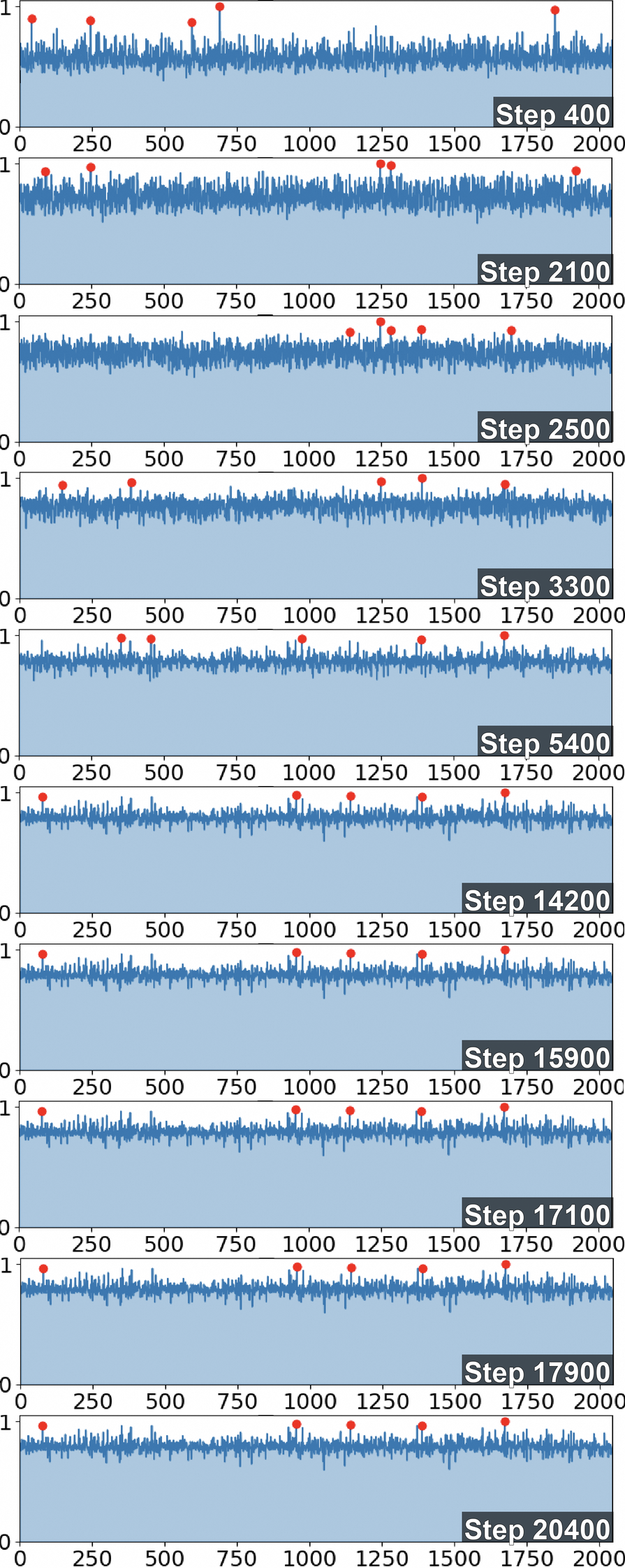}
        \centering\small{Value Projection}
    \end{minipage}\hfill
    \begin{minipage}[b]{0.3\linewidth}
        \includegraphics[width=\linewidth]{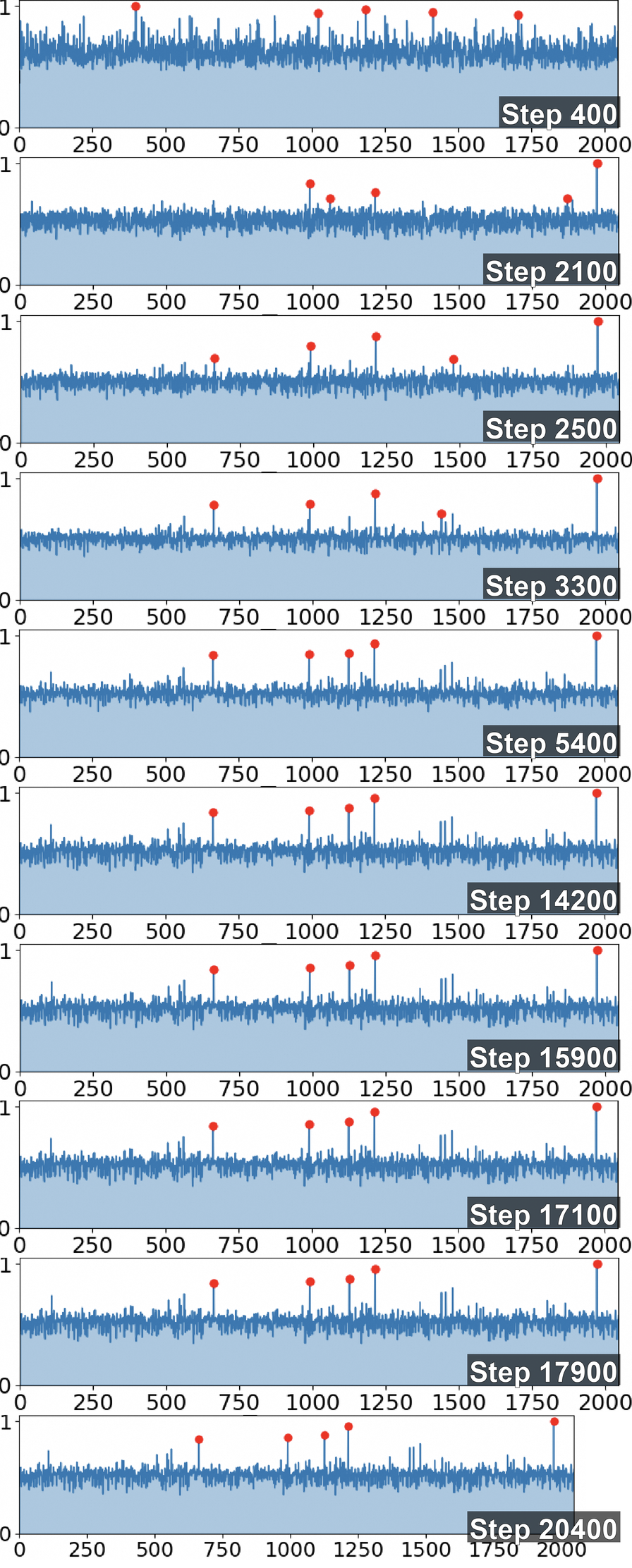}
        \centering\small{Output Projection}
    \end{minipage}\hfill
    \begin{minipage}[b]{0.3\linewidth}
        \includegraphics[width=\linewidth]{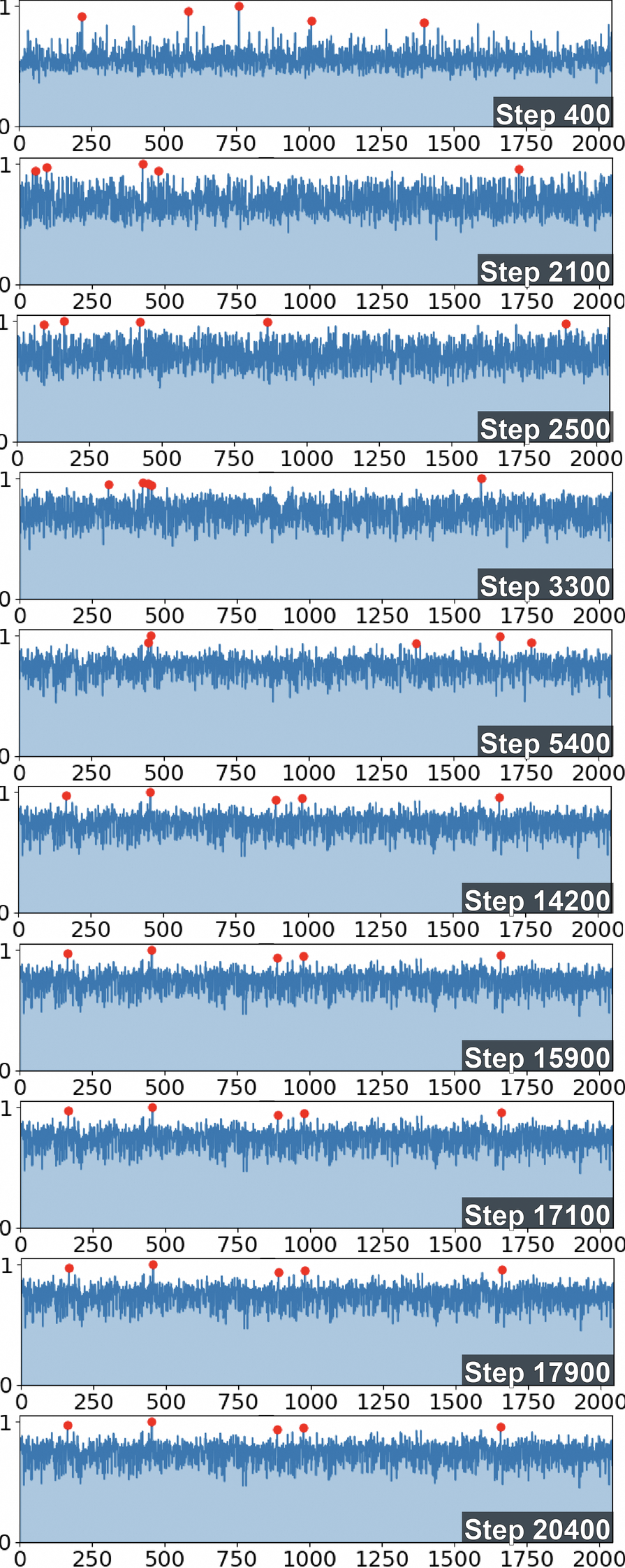}
        \centering\small{Query Projection}
    \end{minipage}
    \caption{\textbf{Evolutionary dynamics of activation outliers for Value, Output, and Query projections.} Red markers identify the top-$k$ magnitude outliers over 2,000 channels across different training stages. Early steps (400–-5,400) exhibit \textbf{transient spikes} that drift across the channel dimension. As training progresses to later stages (14,200–-20,400), these outliers transition into \textbf{persistent hot channels}—spatially fixed dimensions with consistently high magnitudes.}
    \label{fig:proj_concat_voq}
    \vspace{-10pt}
\end{figure*}

\paragraph{Post-QK Operations.}

As illustrated in Fig.~\ref{fig:component}, we identify a consistent architectural pattern where the most sensitive operations in both Softmax and Linear Attention are located within the ``post-QK'' stages (refer to~\ref{fig:ablation_loss} for experimental evidence). Specifically, we find that the Value projection is the primary sensitive operator in Softmax Attention (e.g., Qwen3-1.7B), whereas the Output projection exhibits the highest sensitivity in Linear Attention (e.g., GLA-1.3B). These sensitivities are directly attributable to their positions immediately downstream of outlier sources: in Softmax Attention, the Value projection follows the QK matrix processing and softmax normalization, which significantly increases outlier density, while in Linear Attention, the Output projection succeeds the $gk\_proj$ and $g\_proj$ stages where activation magnitudes reach their peak. As demonstrated in Fig.~\ref{fig:topk_mag_training_of_gkproj}, the activation range of $gk$ is remarkably broad, with values reaching as high as 350, which imposes extreme quantization pressure on the subsequent Output projection and necessitates specialized numerical handling to maintain stability.

\begin{figure*}[t]
    \centering
    \begin{minipage}[b]{0.3\linewidth}
        \includegraphics[width=\linewidth]{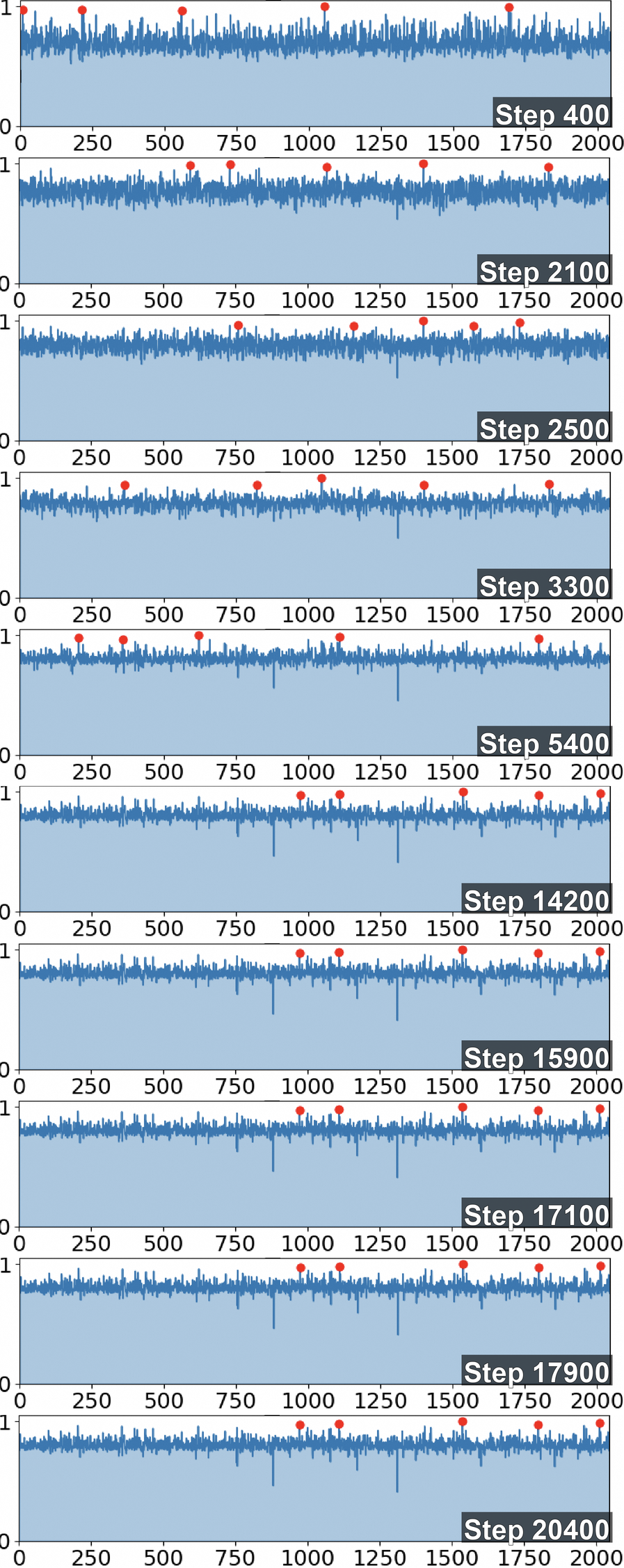}
        \centering\small{Up Projection}
    \end{minipage}\hfill
    \begin{minipage}[b]{0.3\linewidth}
        \includegraphics[width=\linewidth]{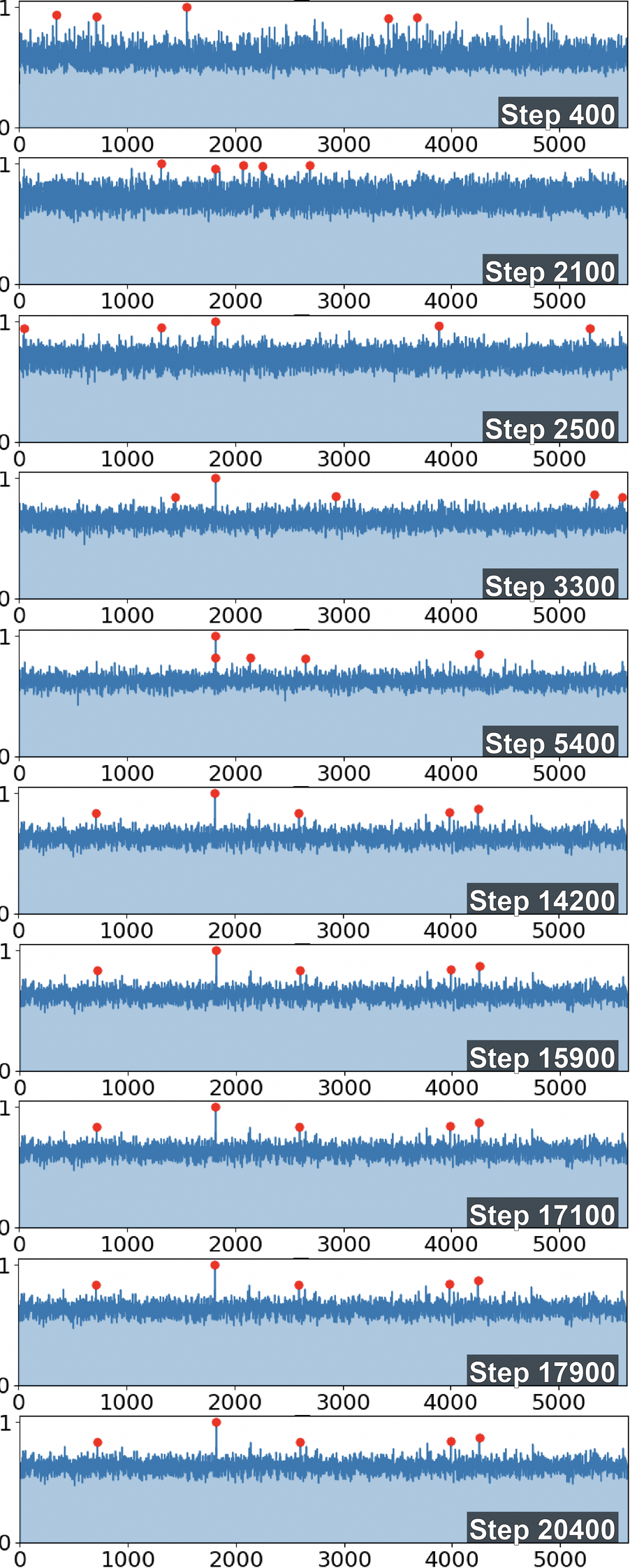}
        \centering\small{Down Projection}
    \end{minipage}\hfill
    \begin{minipage}[b]{0.3\linewidth}
        \includegraphics[width=\linewidth]{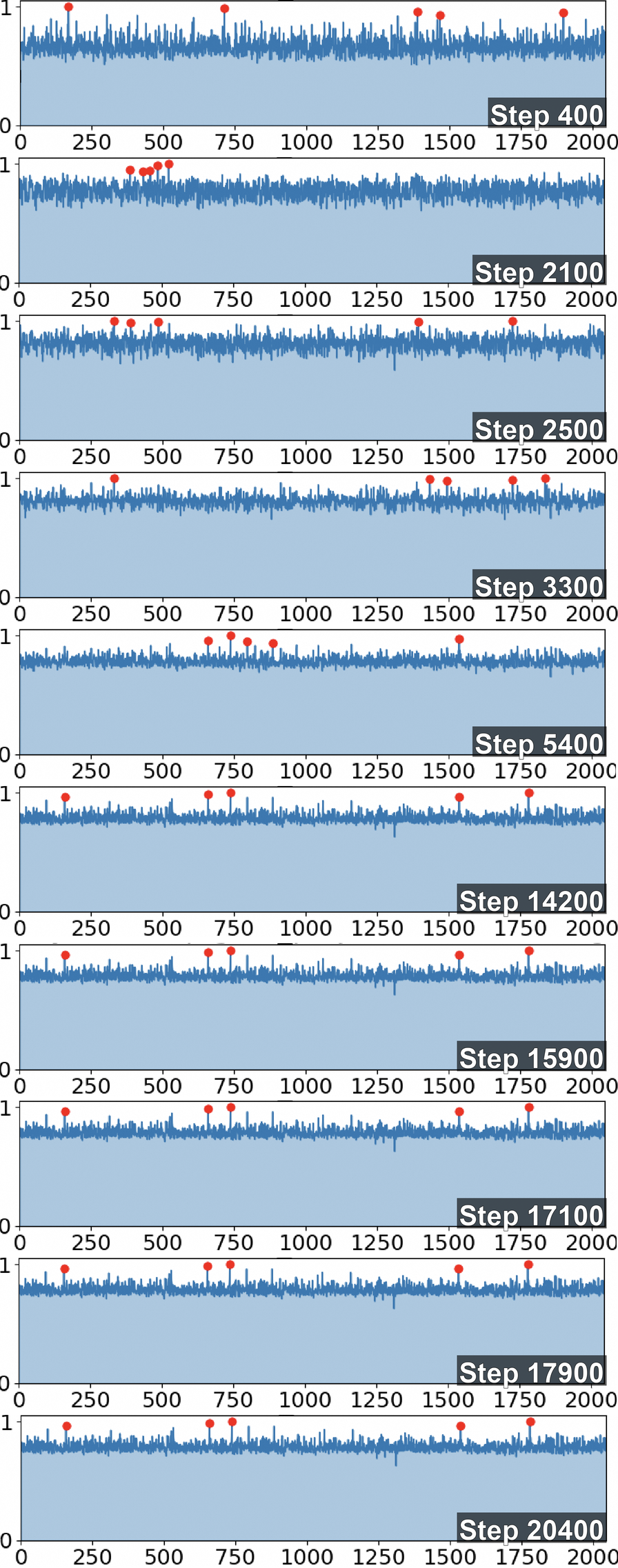}
        \centering\small{Gating Projection}
    \end{minipage}
    \caption{\textbf{Evolutionary dynamics of activation outliers for Up, Down, and Gating projections.}}
    \label{fig:proj_concat_gatefamily}
    \vspace{-10pt}
\end{figure*}

\paragraph{FFN Gating-Mechanism Dynamics.}
The feed-forward network in Fig.~\ref{fig:proj_concat_gatefamily} reveals the interplay among expansion, gating, and contraction. The up-projection (left) expands hidden states to the intermediate FFN dimension (typically $4\times$ the model dimension), showing relatively uniform magnitudes due to the distributional-smoothing effect of the expansion transformation.

\begin{figure*}[t]
    \centering
    \begin{minipage}[b]{0.3\linewidth}
        \includegraphics[width=\linewidth]{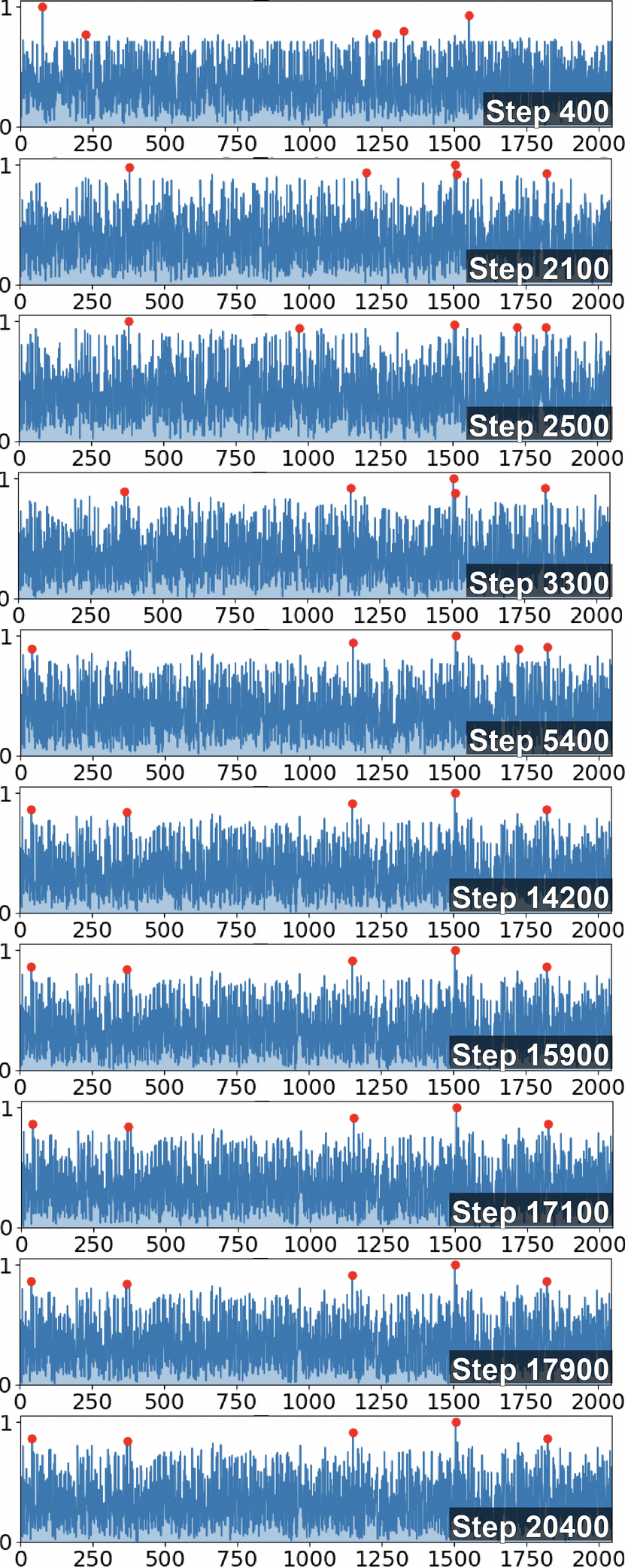}
        \centering\small{GK Projection}
    \end{minipage}\hfill
    \begin{minipage}[b]{0.3\linewidth}
        \includegraphics[width=\linewidth]{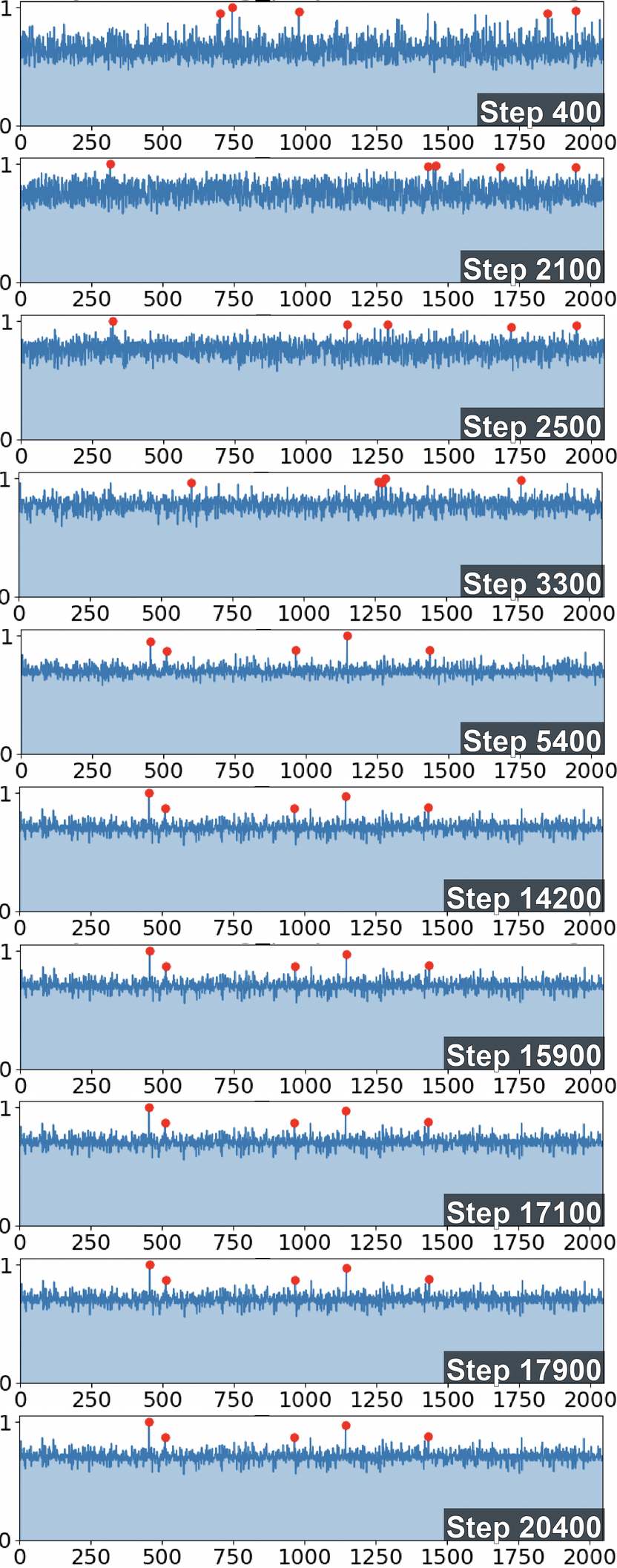}
        \centering\small{Gate Projection}
    \end{minipage}\hfill
    \begin{minipage}[b]{0.3\linewidth}
        \includegraphics[width=\linewidth]{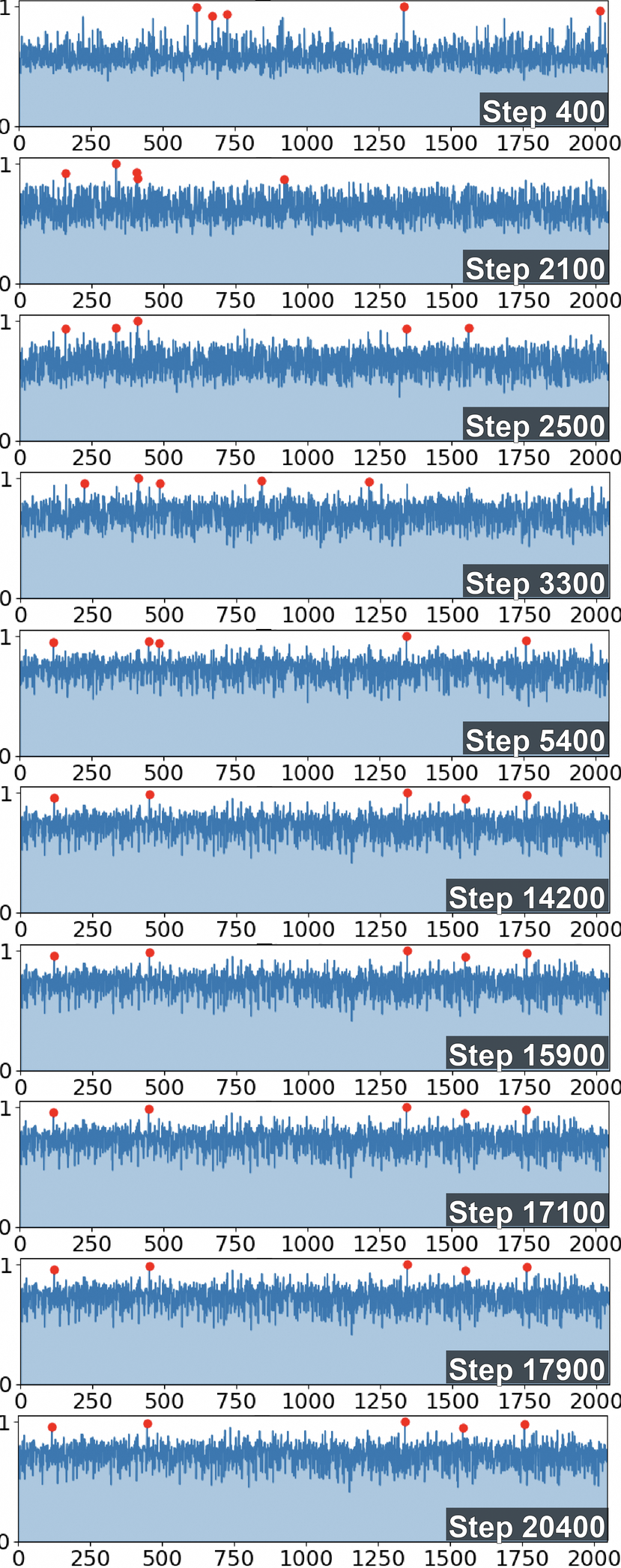}
        \centering\small{Key Projection}
    \end{minipage}
    \caption{\textbf{Evolutionary dynamics of activation outliers for GK, Gate, and Key projections.}}
    \label{fig:proj_concat_gatekey}
    \vspace{-10pt}
\end{figure*}

The down-projection (center), operating after the non-linearity, concentrates energy when contracting back to the model dimension and exhibits the highest outlier incidence among FFN components. This concentration arises because the down-projection must compress $4d$ dimensions into $d$, forcing representational bottlenecks that manifest as magnitude spikes in the most informative channels. The gating projection (right) modulates the FFN via element-wise multiplication (SwiGLU-style gating~\cite{shazeer2020glu}), introducing multiplicative amplification that occasionally produces localized spikes when gate values near saturation coincide with large up-projection activations. Our 2-D quantization in the backward pass and selective BF16 fallback for persistent outlier channels (identified via running $\ell_2$ norms) effectively stabilize these three interdependent pathways while keeping $>$90\% of GEMMs in FP4. Profiling indicates that the down-projection alone accounts for 45-–55\% of FFN quantization error, justifying its prioritization in our compensation budget.

\subsection{Frobenius Energy Analysis}
\label{app:frobenius_energy}

\begin{figure*}
    \centering
    \includegraphics[width=.65\linewidth]{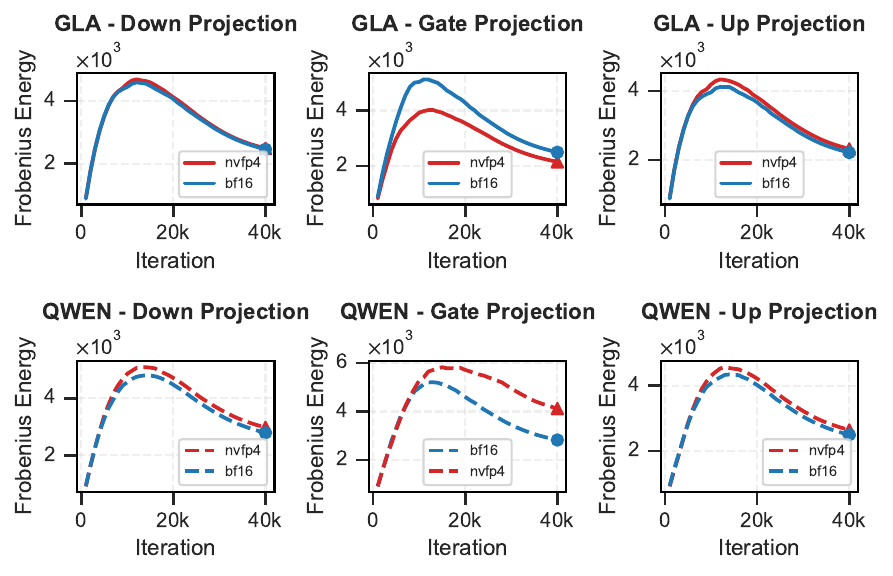}
    \caption{\textbf{Frobenius energy dynamics for FFN projections.} NVFP4 induces energy collapse in GLA and numerical instability in Qwen (notably in Gate projections). Our CHON recipe stabilizes these trajectories to closely match the BF16 baselines.}
    \label{fig:frobenius_energy_analysis}
\end{figure*}

\paragraph{Frobenius Energy Analysis.}
Fig.~\ref{fig:frobenius_energy_analysis} presents the Frobenius energy analysis across different projection layers throughout training. 
We observe distinct behaviors for GLA and Qwen models. For GLA, the standard NVFP4 setting (red curve) exhibits a significant drop in Frobenius energy compared to the BF16 baseline (blue curve) across all projections, suggesting a potential underestimation or damping of weight magnitudes. In contrast, incorporating HCP (green curve) effectively recovers the energy profile, aligning it closely with the BF16 trajectory.
For Qwen, particularly in the Gate projection, the NVFP4 setting leads to an energy explosion, indicating numerical instability. The HCP mechanism successfully suppresses this divergence, maintaining a stable energy evolution comparable to the precision baseline.
These results highlight that HCP acts as a bi-directional stabilizer—preventing both energy collapse (as in GLA) and explosion (as in Qwen)—thereby ensuring that FP4 training dynamics remain consistent with full-precision counterparts.

\subsection{Flush-to-Zero Dynamic Analysis}
\label{app:ftz_dynamic}

\begin{figure}[t]
    \centering
    \includegraphics[width=0.8\linewidth]{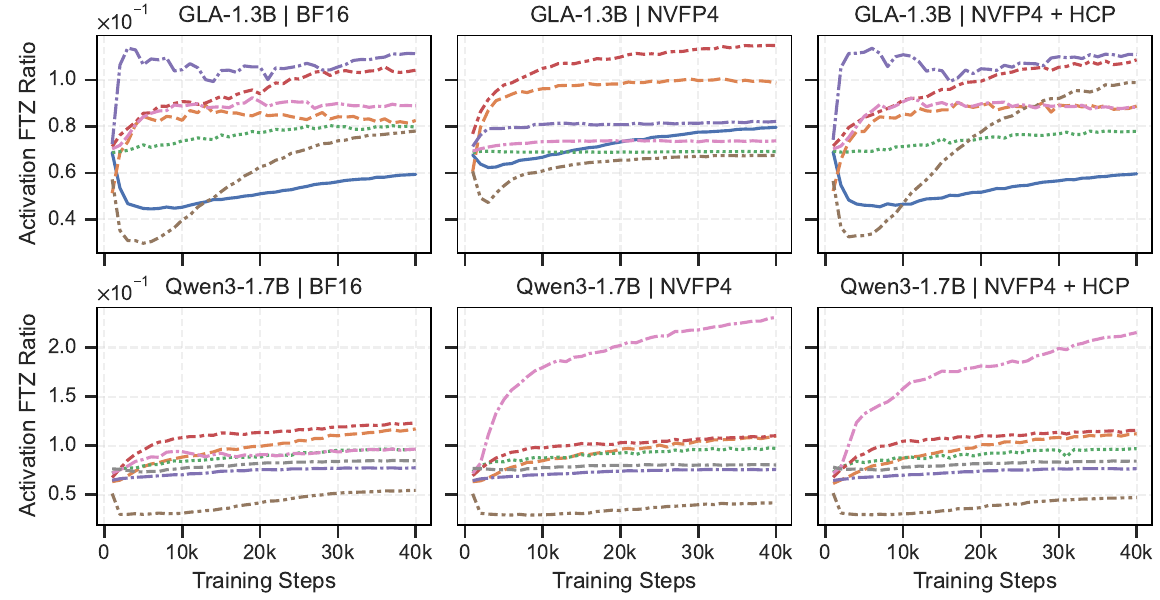}
    \caption{\textbf{Evolution of Activation Flush-to-Zero (FTZ) rates throughout training.} The FTZ rate (\texttt{ftz\_nvfp4}) quantifies the proportion of activation values that collapse to zero due to underflow in the NVFP4 format. Compared to the vanilla NVFP4 baseline, the \textbf{CHON} recipe (NVFP4 + HCP) significantly suppresses these underflow events, bringing the distribution closer to the BF16 baseline.}
    \label{fig:act-ftz-ablation}
\end{figure}

\begin{figure}[t]
    \centering
    \includegraphics[width=0.8\linewidth]{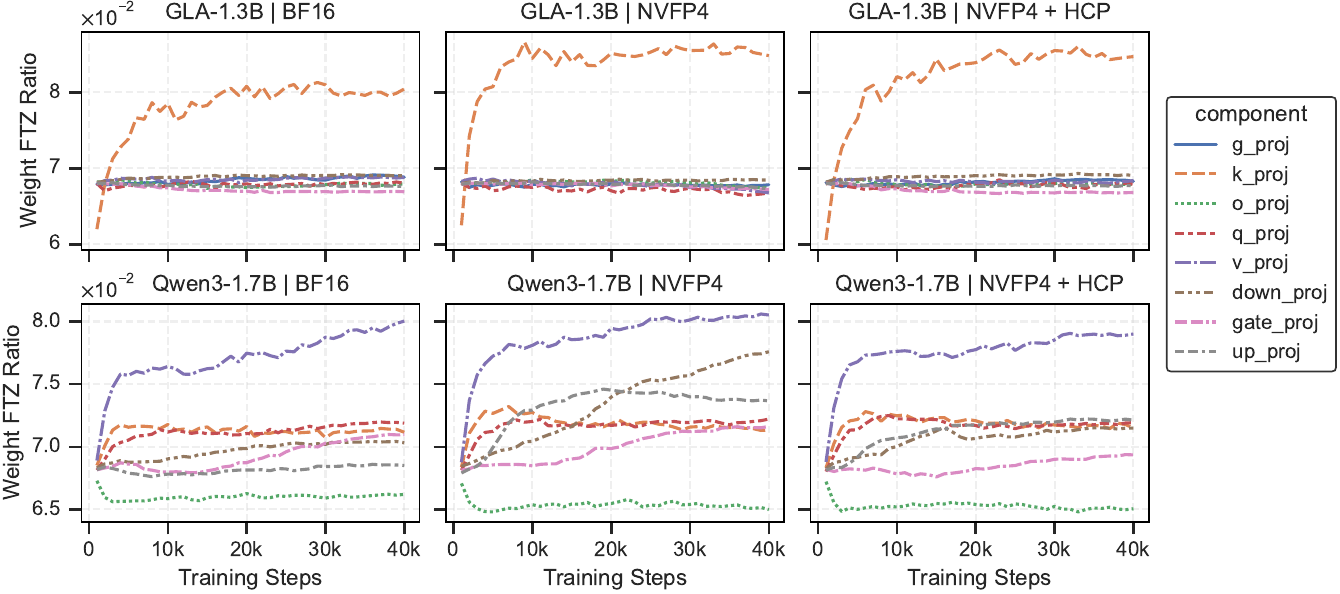}
    \caption{\textbf{Weight Flush-to-Zero (FTZ) dynamics across training stages.} Unlike activations, weight tensors exhibit significantly lower FTZ ratios (note the $10^{-2}$ scale), indicating they are inherently more robust to underflow-to-zero events in the NVFP4 format. The minimal separation between the BF16 baseline, vanilla NVFP4, and the CHON recipe curves confirms that weight distributions are well-behaved and effectively captured by the FP4 E2M1 representable range.}
    \label{fig:weight-ftz-ablation}
\end{figure}

We analyze the \,\texttt{ftz\_nvfp4}\, metric, defined as the fraction of tensor elements that become exactly zero after NVFP4 quantization (i.e., \emph{underflow-to-zero}). Higher FTZ indicates more irreversible loss of ``small-but-nonzero'' signals and is therefore a direct indicator of underflow severity in FP4 training.

\paragraph{Activations are the primary underflow bottleneck.}
Comparing Activation FTZ (Fig.~\ref{fig:act-ftz-ablation}) and Weight FTZ (Fig.~\ref{fig:weight-ftz-ablation}), we observe a clear scale gap: activation FTZ typically lies in the $\sim\!0.05$--$0.20$ range, whereas weight FTZ concentrates in a much narrower band (around $\sim\!0.065$--$0.085$). This suggests that, under NVFP4, activations place substantially more probability mass near zero and/or exhibit a stronger mismatch to the limited FP4 dynamic range, making them more prone to underflow-to-zero than weights.

\paragraph{CHON suppresses activation underflow induced by NVFP4.}
Across both models, enabling CHON consistently reduces the FTZ increases introduced by the baseline NVFP4 recipe. In particular, for GLA-1.3B, the elevated activation FTZ in attention-related projections (e.g., \texttt{q\_proj}/\texttt{k\_proj}) under NVFP4 is largely restored toward the BF16 baseline once CHON is applied, indicating that the compensation prevents excessive contraction of small values into exact zero. For Qwen3-1.7B, the \texttt{up\_proj} activation FTZ exhibits the most pronounced degradation under NVFP4 (reaching $>0.20$ in the worst case), and CHON substantially reduces this spike, though \texttt{up\_proj} remains the dominant contributor.

\paragraph{Weights are comparatively robust under NVFP4.}
Weight FTZ curves show limited separation across BF16, NVFP4, and NVFP4+CHON (Fig.~\ref{fig:weight-ftz-ablation}), indicating that weight distributions are relatively stable and largely covered by the FP4 E2M1 representable range. Consequently, the primary underflow bottleneck in NVFP4 training arises from activations rather than weights.

\paragraph{Training dynamics and early-warning signal.}
We further note that, for Qwen3-1.7B, even with CHON, the activation FTZ of \texttt{up\_proj} can gradually increase over training steps. This suggests that activation distributions can drift during training (e.g., increasing sparsity or accumulating small-magnitude activations), causing underflow-to-zero risk to compound over time. Therefore, tracking FTZ at the component level provides a practical early-warning signal: sustained FTZ growth in a specific module may foreshadow instability induced by low-precision quantization.

\subsection{Gating as Outlier Source}
\label{app:gk_outlier_analysis}

Vanilla Linear Attention~\cite{aueb2016one, qin2022cosformer, sun2023vicinity} suffers from the ``attention dilution'' issue~\cite{qin2022devil}, where attention scores are spread over long sequences and local neighborhood structure is under-emphasized, whereas softmax attention often becomes overly local and can produce sharp, outlier-prone distributions due to exponential normalization. To alleviate these issues, recent linear attention methods~\cite{nahshan2023linear, zhang2024hedgehog} adopt an element-wise exponential feature map, i.e., $\phi(x)=\exp(t x)$ with $t>2$, to better concentrate attention weights; in this work we choose Gated Linear Attention (GLA) as a representative for its simplicity and efficiency.

Given input projections for queries ($Q$), keys ($K$), and values ($V$) in $\mathbb{R}^{T \times d_h}$ for a sequence of length $T$, GLA computes the output $y_t$ at each timestep $t$ as follows. First, it applies a positive feature map $\phi: \mathbb{R}^{d_h} \to \mathbb{R}^{d_f}$ to queries and keys. Then, it maintains two streaming accumulators:
\begin{align}
S_t &= \sum_{i=1}^{t} \phi(K_i) V_i^\top \in \mathbb{R}^{d_f \times d_h} \label{eq:state_matrix} \\
z_t &= \sum_{i=1}^{t} \phi(K_i) \in \mathbb{R}^{d_f} \label{eq:normalizer}
\end{align}
These accumulators can be updated recursively: $S_t = S_{t-1} + \phi(K_t)V_t^\top$ and $z_t = z_{t-1} + \phi(K_t)$. The normalized output $\tilde{y}_t$ is then computed as:
\begin{align}
\tilde{y}_t = \frac{\phi(Q_t)^\top S_t}{\phi(Q_t)^\top z_t + \varepsilon} \in \mathbb{R}^{d_h}
\end{align}
Finally, a data-dependent gate $g_t$ modulates the output:
\begin{align}
y_t = g_t \odot \tilde{y}_t, \quad \text{where} \quad g_t = \sigma(W_g Q_t + b_g)
\end{align}
Here, $\sigma$ is a sigmoid function, and $\odot$ is the elementwise product. This formulation is associative, allowing for computation with $O(T)$ complexity and constant memory, while avoiding the numerical instability of Softmax.

To understand the emergence of extreme outliers in Gated Linear Attention (GLA), we analyze the parameterization of the state decay mechanism and its impact on the recurrent pathway. In GLA, the hidden state $\mathbf{S}_t \in \mathbb{R}^{d \times d}$ is updated via a cumulative sum modulated by a data-dependent decay factor $\lambda_t \in (0, 1]$:
\begin{equation}
\mathbf{S}_t = \lambda_t \odot \mathbf{S}_{t-1} + \mathbf{k}_t \mathbf{v}_t^\top,
\label{eq:state_update}
\end{equation}
where $\lambda_t$ governs the retention of historical information.

\paragraph{Gating Parameterization Bottleneck}
In the implementation of GLA, the decay factor $\lambda_t$ is derived from the output of the gated-key projection $\hat{g}_t = \text{proj}_{gk}(\mathbf{x}_t)$ through a composition of a log-sigmoid activation and a scaling constant $\gamma$ (typically $\gamma = 16$):
\begin{equation}
\lambda_t = \exp\left( \frac{\log \sigma(\hat{g}_t)}{\gamma} \right) = \left[ \sigma(\hat{g}_t) \right]^{1/\gamma},
\label{eq:lambda_param}
\end{equation}
where $\sigma(x) = (1 + e^{-x})^{-1}$ is the sigmoid function. This specific parameterization introduces significant numerical challenges, particularly in the extreme ends of the decay range.

\textbf{Lower Bound for State Resetting.} To achieve rapid forgetting of local context, $\lambda_t$ must be sufficiently small (e.g., $\lambda_t \approx \epsilon = 10^{-3}$). Solving for the pre-activation $\hat{g}_t$ in Eq.~\ref{eq:lambda_param} yields $\sigma(\hat{g}_t) = \epsilon^\gamma$. Given $\gamma = 16$, $\epsilon^\gamma$ reaches the order of $10^{-48}$. Using the asymptotic approximation $\sigma(x) \approx e^x$ for $x \to -\infty$, we obtain $\hat{g}_t \approx \gamma \ln(\epsilon)$. For $\epsilon = 10^{-3} \approx e^{-6.9}$, this implies $\hat{g}_t \approx -110.4$. This derivation explains the observed structural lower bound of approximately $\mathbf{-120}$ in $\text{proj}_{gk}$. 

\textbf{Upper Bound and Gradient Vanishing.} Conversely, for long-term modeling, $\lambda_t \to 1$ requires $\hat{g}_t$ to enter the positive saturation zone. The sensitivity $\frac{\partial \lambda_t}{\partial \hat{g}_t} = \frac{1}{\gamma} \sigma(\hat{g}_t)^{1/\gamma} (1 - \sigma(\hat{g}_t))$ vanishes exponentially as $\hat{g}_t \to \infty$. To effect infinitesimal increases in retention near $\lambda \approx 1$, the optimizer must push $\hat{g}_t$ significantly further into the saturation region, leading to the observed positive tail of approximately $\mathbf{80}$.

\paragraph{Scaling Asymmetry and Signal Recovery in $\text{proj}_k$}
Empirically, we observe that the key projection $\mathbf{k}_t = \text{proj}_k(\mathbf{x}_t)$ also exhibits abnormal magnitudes ($\approx 300$), significantly exceeding the query projection $\mathbf{q}_t$ ($\approx 50$). We attribute this to two structural factors:
\begin{itemize}
    \item \textbf{Numerical Compensation for Asymmetric Scaling:} In hardware-efficient GLA kernels, the attention scale $\tau = d^{-1/2}$ is typically applied solely to queries ($\tilde{\mathbf{q}}_t = \tau \mathbf{q}_t$), while keys remain unscaled during the state update in Eq.~\ref{eq:state_update}. To maintain a balanced dot-product magnitude $\langle \tilde{\mathbf{q}}, \mathbf{k} \rangle$, the optimizer compensates for the $d^{-1/2}$ reduction in $\mathbf{q}$ by expanding the dynamic range of $\mathbf{k}$.
    \item \textbf{SNR Preservation in Recurrent States:} When performing a ``state reset'' ($\lambda_t \to 0$), the term $\lambda_t \mathbf{S}_{t-1}$ vanishes. To ensure the new information $\mathbf{k}_t \mathbf{v}_t^\top$ accurately updates the state matrix against quantization noise and residual gradients, $\mathbf{k}_t$ must maintain a high Signal-to-Noise Ratio (SNR). This signal recovery pressure forces $\text{proj}_k$ to generate large-magnitude activations to dominate the state-transition dynamics.
\end{itemize}

\paragraph{Numerical Implications for NVFP4}
The extreme dynamic range $\mathcal{R} \approx [-120, 300+]$ in the recurrent pathway poses a fatal challenge for 4-bit quantization. A uniform scaling factor covering $\mathcal{R}$ results in a quantization step $\Delta \approx 420 / 16 = 26.25$. Such a coarse resolution is larger than the entire dynamic range of standard hidden layers (e.g., $[-20, 20]$), effectively annihilating the model's ability to switch between short-term and long-term memory. This observation necessitates a targeted stabilization strategy, such as Hot-Channel Patch (HCP), to preserve the high-magnitude signals in these critical gating and recurrent channels.

\begin{figure}[t]
    \centering
    \includegraphics[width=.85\linewidth]{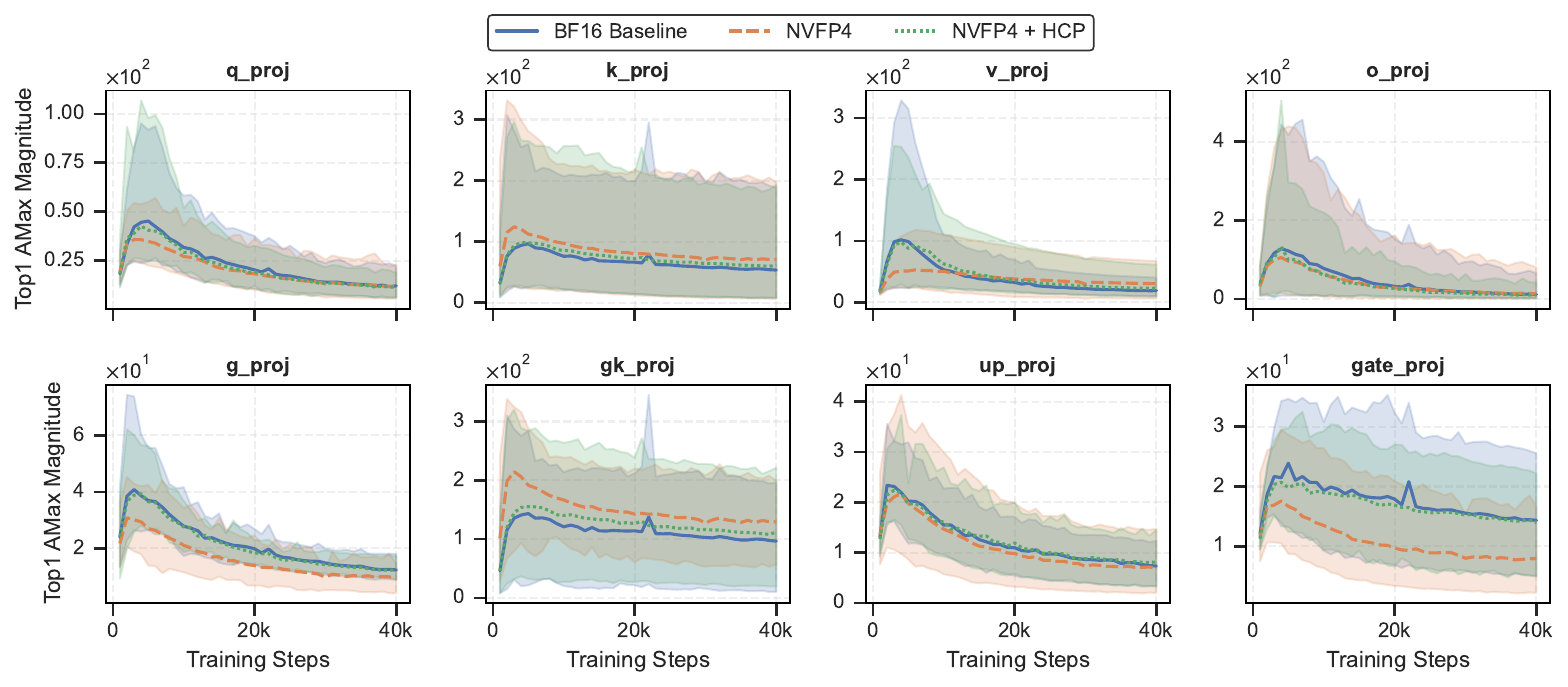}
    \caption{\textbf{Evolution of Top-1 activation magnitudes across GLA components.} The gated-key projection (\texttt{gk\_proj}) exhibits the most extreme outliers (magnitudes up to 350), which we attribute to the log-sigmoid parameterization required for state resetting ($\lambda_t \approx 0$). This extreme dynamic range imposes significant numerical pressure on downstream sensitive operations like the output projection (\texttt{o\_proj}).}
    \label{fig:gla_gate_dynamics}
\end{figure}

\subsection{Normalization Impact Analysis}
\label{app:norm_impact}

As illustrated in Fig.~\ref{fig:norm_comparison}, we visualize the distribution of the RMSNorm scale parameter $\gamma$. We find that GLA typically maintains $\gamma<1$, thereby limiting activation amplification and yielding smoother activation statistics. In contrast, Qwen (Softmax Attention) exhibits $\gamma>1$, consistent with the need to counteract softmax-induced activation spikes; this norm-induced amplification produces systematic outliers that are difficult to represent under NVFP4’s limited dynamic range.

Fig.~\ref{fig:rmnsnorm_range} further compares normalization behavior across Softmax Attention (SA), Linear Attention (LA), and Hybrid Attention (HA) models, revealing several consistent trends. (i) The magnitude of $\gamma$ generally increases with layer depth, which helps explain the higher incidence of activation outliers in deeper layers of LLMs. (ii) LA models exhibit substantially smaller $\gamma$ values (typically $\approx 2$) than SA models (often reaching $\approx 10^{2}$), indicating a markedly reduced tendency to amplify activations. (iii) These depth-wise scaling patterns are broadly consistent across model sizes, as evidenced by similar trends between the 1.3B and 2.7B variants of DeltaNet and Mamba. (iv) Among HA architectures, Qwen3-Next maintains a comparatively low-$\gamma$ profile throughout the network, which constrains activation ranges and improves quantization robustness. (v) RMSNorm layers immediately following gating projection show a more pronounced increase in $\gamma$ with depth, suggesting that gating can indirectly exacerbate activation magnitudes through normalization.

\begin{figure}[htbp]
    \centering
    \begin{subfigure}[b]{0.32\linewidth}
        \centering
        \includegraphics[width=\linewidth]{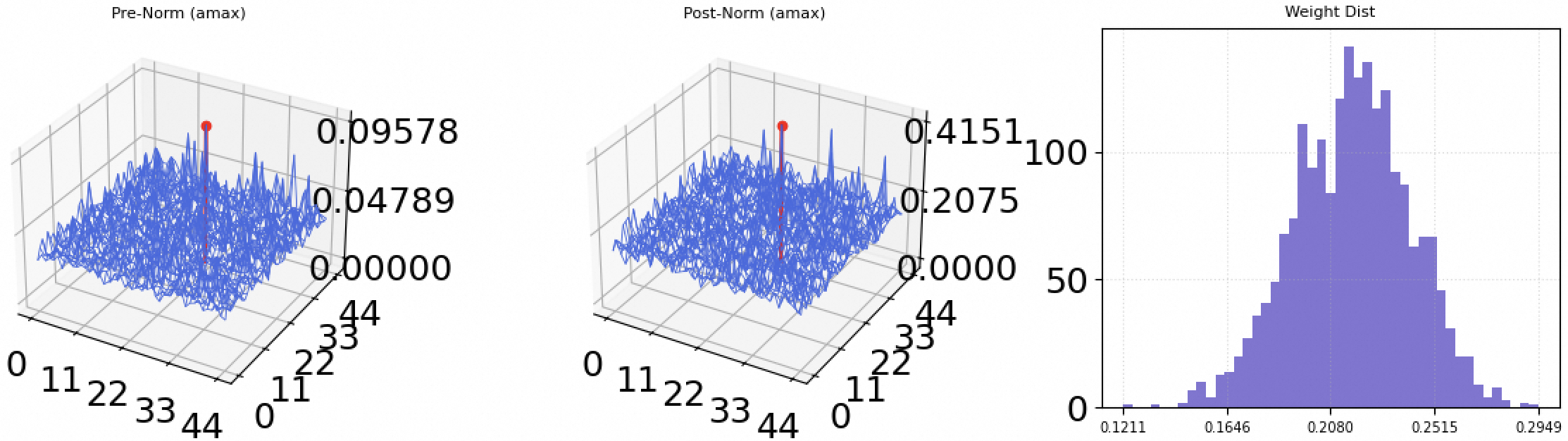}
        \caption{GLA: Attention Norm}
    \end{subfigure}
    \hfill
    \begin{subfigure}[b]{0.32\linewidth}
        \centering
        \includegraphics[width=\linewidth]{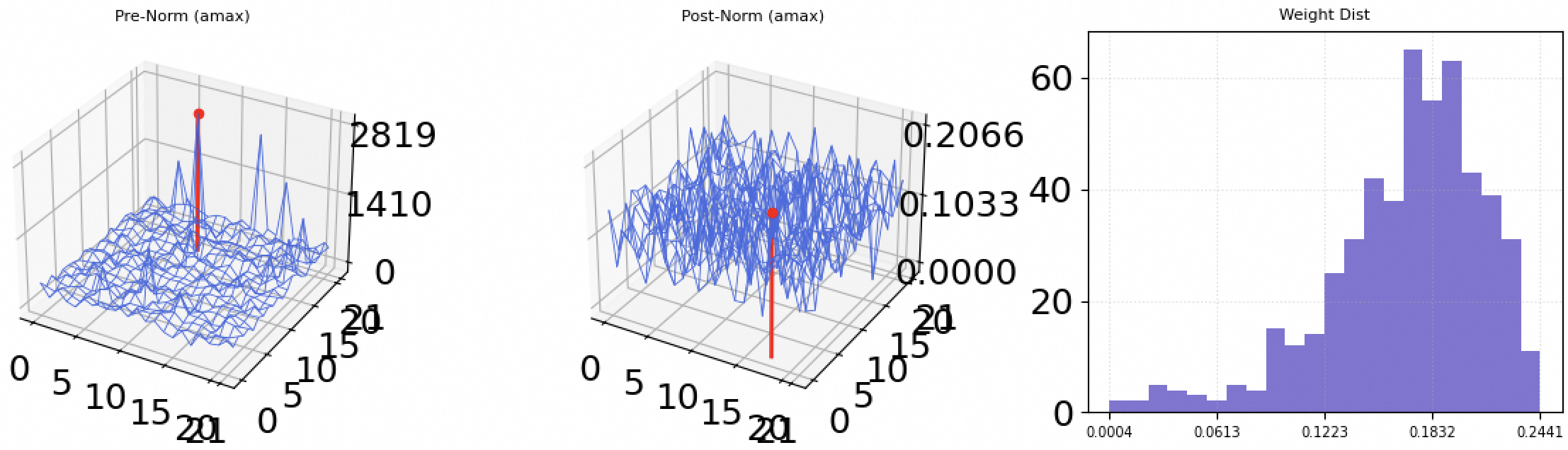}
        \caption{GLA: Gating Norm}
    \end{subfigure}
    \hfill
    \begin{subfigure}[b]{0.32\linewidth}
        \centering
        \includegraphics[width=\linewidth]{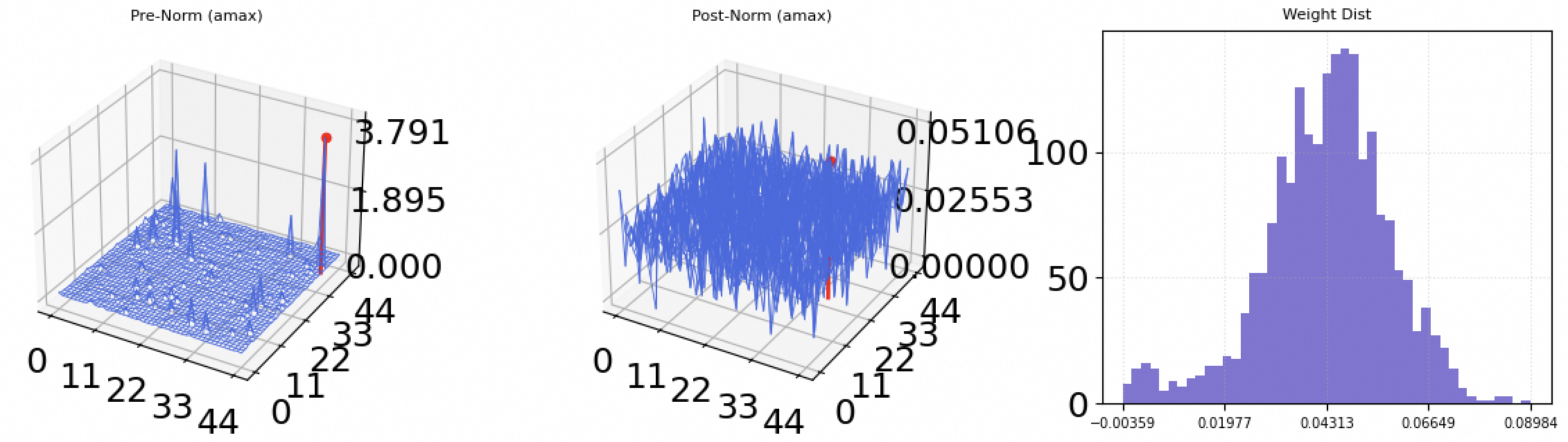}
        \caption{GLA: MLP Norm}
    \end{subfigure}

    \vspace{0.8em} 

    \begin{subfigure}[b]{0.32\linewidth}
        \centering
        \includegraphics[width=\linewidth]{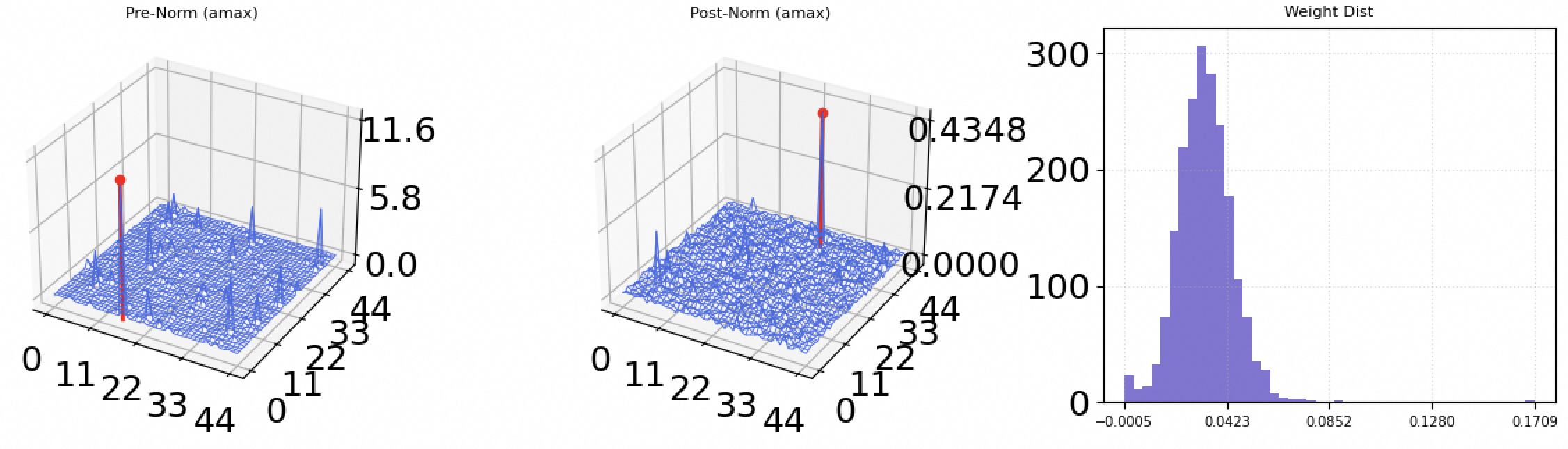}
        \caption{Qwen: Post-Attn Norm}
    \end{subfigure}
    \hfill
    \begin{subfigure}[b]{0.32\linewidth}
        \centering
        \includegraphics[width=\linewidth]{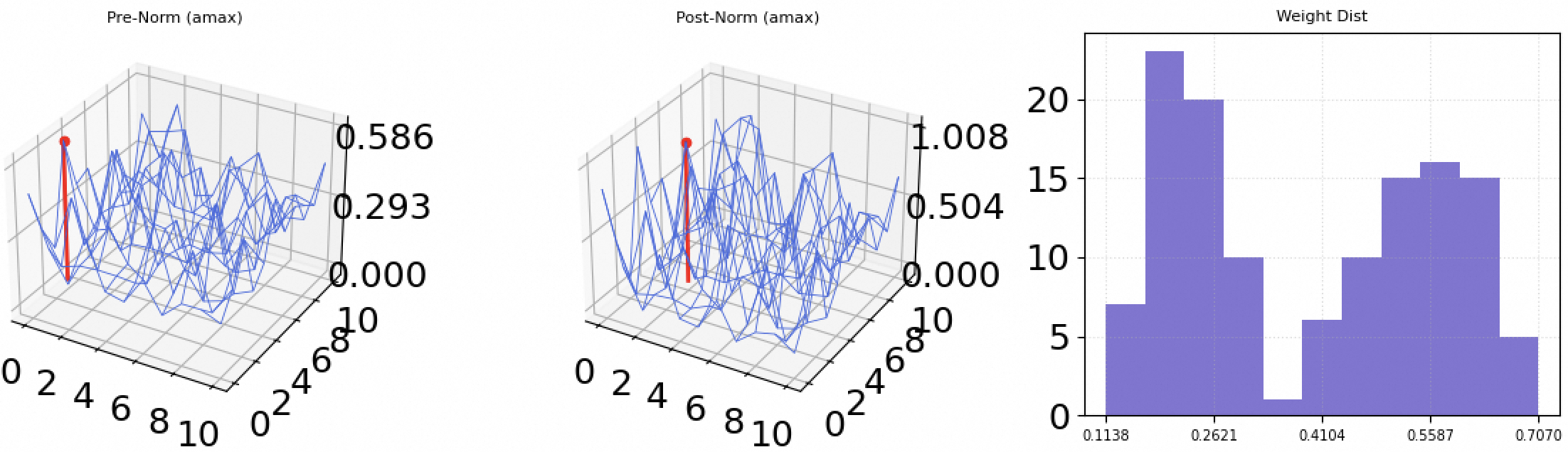}
        \caption{Qwen: Key Norm}
    \end{subfigure}
    \hfill
    \begin{subfigure}[b]{0.32\linewidth}
        \centering
        \includegraphics[width=\linewidth]{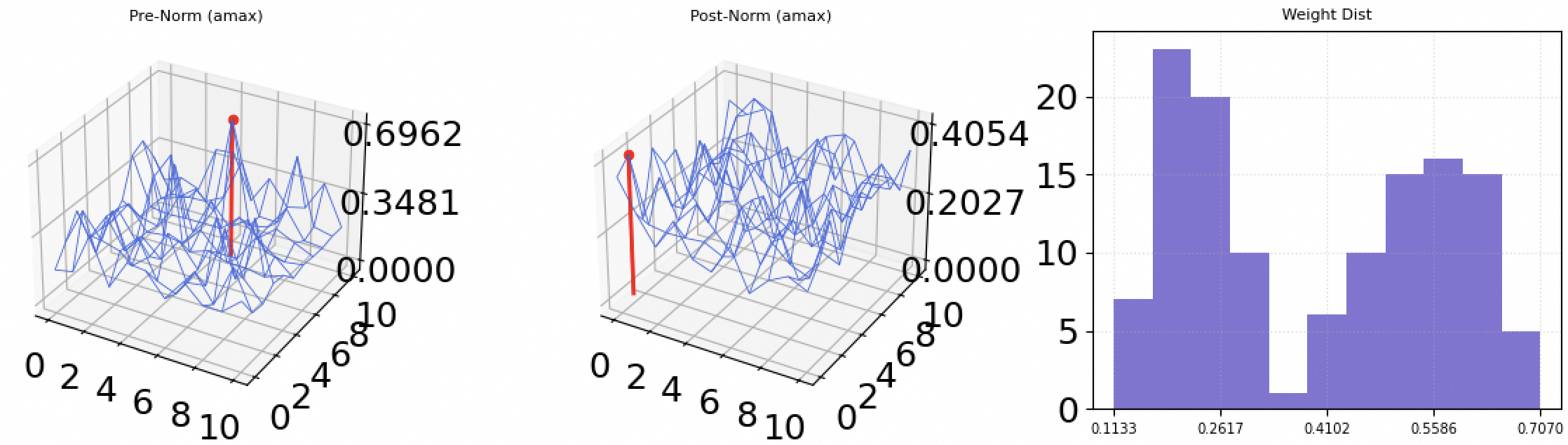}
        \caption{Qwen: Query Norm}
    \end{subfigure}
    
    \caption{\textbf{Comparison of normalization and $\gamma$ distributions in GLA (top) vs.\ Qwen (bottom).} Each panel displays the Pre-Norm (left) and Post-Norm (middle) activation landscapes, along with the \textbf{learnable scaling factor $\gamma$ distribution} (right).}
    \label{fig:norm_comparison}
\end{figure}

\begin{figure}[t]
    \centering
    \includegraphics[width=0.99\linewidth]{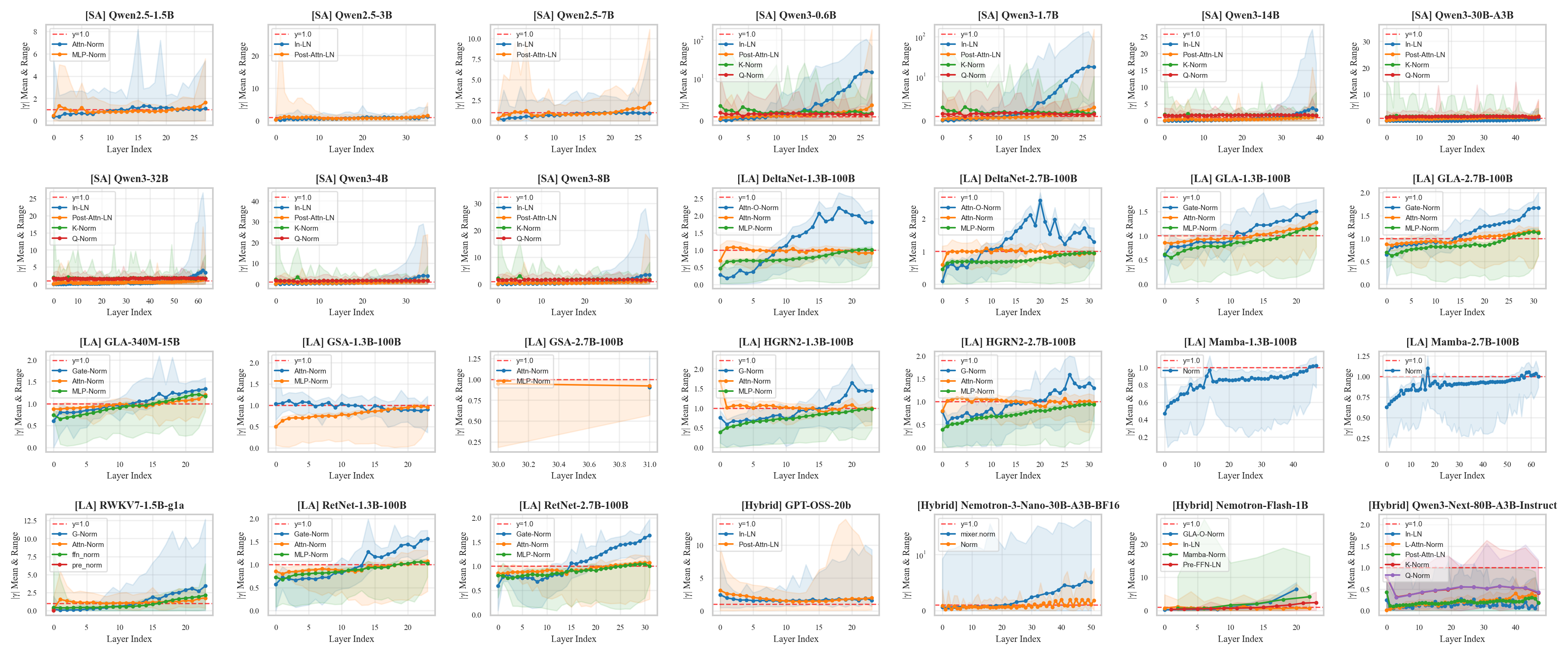}
    \caption{\textbf{Analysis of normalization distributions} across Softmax Attention (SA), Linear Attention (LA), and Hybrid Attention (HA) architectures.}
    \label{fig:rmnsnorm_range}
\end{figure}

\subsection{Empirical Analysis of Representational Overlap}
\label{app:superposition}

\begin{figure}[ht]
    \centering
    \includegraphics[width=.65\linewidth]{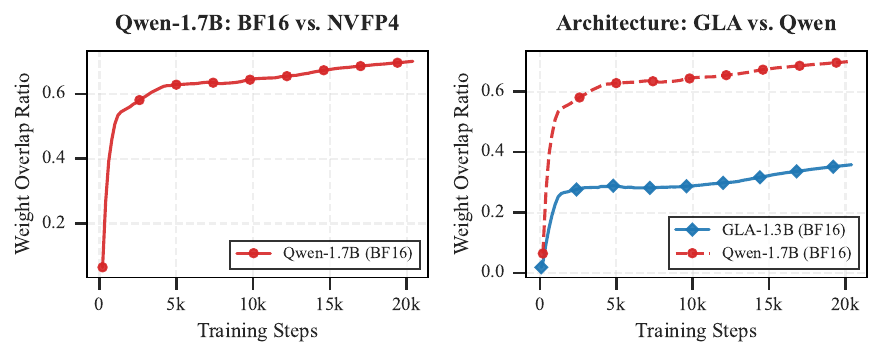}
    \vspace{-10pt}
    \caption{\textbf{Weight overlap dynamics and feature superposition analysis.} The weight overlap ratio serves as a proxy for superposition density and feature interference. (Left) Precision comparison: \textbf{NVFP4} quantization suppresses representational overlap compared to \textbf{BF16}, as the limited numerical resolution forces the model towards a quasi-orthogonal state to minimize interference noise. (Right) Architectural comparison: \textbf{GLA} (Linear Attention) maintains significantly lower overlap than \textbf{Qwen} (Softmax Attention), indicating an inductive bias toward feature decoherence.}
    \label{fig:overlap_plots}
\end{figure}

To investigate the mechanistic drivers of neural scaling, we analyze the geometric properties of the language model head (\texttt{lm\_head}) through the lens of feature superposition. Following the theoretical framework proposed by \citet{Liu2025SuperpositionYR}, we utilize the weight overlap magnitude---defined as the squared Frobenius norm of the off-diagonal correlation matrix of representation vectors---as a proxy for feature interference and superposition density.

\paragraph{Quantization-Induced Geometric Regularization.} A striking observation from our experiments is the disparity in overlap dynamics between BF16 and NVFP4 precision regimes (see Fig.~\ref{fig:overlap_plots}, Left). While BF16 training facilitates a rapid ascent into a high-overlap regime, characterizing \textit{strong superposition} where features are densely packed through complex interference-canceling biases, NVFP4 exhibits significant suppression of representational overlap.

We hypothesize that 4-bit quantization acts as a \textit{stochastic geometric bottleneck}. The limited numerical resolution of NVFP4 (spanning only 16 discrete levels) precludes the fine-grained weight adjustments necessary to manage the high-interference environment inherent in dense superposition. Consequently, the model is incentivized to converge toward a \textit{quasi-orthogonal} representational state to minimize cross-feature noise that cannot be effectively mitigated by low-precision error correction. This lower overlap suggests that while low-precision models may sacrifice total feature capacity (superposition density), they gain representational robustness by enforcing clearer geometric separation in the hidden space.

\paragraph{Architectural Influence on Superposition Regimes.}

Beyond numerical precision, the underlying architecture significantly dictates the evolution of the superposition state. Our comparison between the Qwen (Transformer-based) and GLA (Gated Linear Attention) architectures reveals distinct geometric trajectories (Fig.~\ref{fig:overlap_plots}, Right). 

The GLA architecture maintains a substantially lower weight overlap compared to the standard Transformer (Qwen-1.7B), even when trained under identical precision (BF16). This phenomenon suggests that the gating mechanisms and linear recurrence in GLA impart a structural inductive bias toward \textit{feature decoherence}. Unlike the global receptive field of \texttt{Softmax Attention}, which often leads to ``cluttered'' hidden spaces with high overlap to maximize associative capacity, the gated recurrent nature of GLA may necessitate more discrete, independent feature tracking. 

\paragraph{Implications for Scaling and Capacity.} The observed dynamics have profound implications for neural scaling laws. According to the theory that loss scales inversely with model dimension ($L \propto 1/m$) due to geometric interference, the lower overlap in NVFP4 and GLA indicates a different coefficient in the scaling function. Specifically, the lower overlap in NVFP4 suggests that quantization noise forces the model into a \textit{resolution-limited} regime earlier than its high-precision counterparts. However, the stability of the overlap magnitude in NVFP4 during late-stage training implies that the model reaches an equilibrium between capacity and precision-induced interference. For future large-scale deployments, these results suggest that low-precision formats like NVFP4 may inherently favor architectures with higher dimensions but lower superposition density, shifting the optimal balance of the width-depth trade-off in the scaling law.

\subsection{Quantization Error Dynamics}
\label{app:quant_error_dynamics}

To comprehensively understand the behavior of NVFP4 quantization across different attention mechanisms, we conduct a systematic analysis of quantization errors during training. We select GLA-1.3B as the representative model for \texttt{Linear Attention} architectures and Qwen-1.7B as the representative model for traditional \texttt{Softmax Attention} mechanisms. Throughout training, we monitor the mean squared error (MSE) between full-precision and quantized values for both weights and activations across all linear projection layers.

Fig.~\ref{fig:mse_error} presents the quantization error evolution across different components during training. Our analysis reveals two critical observations that characterize the quantization dynamics under NVFP4:

\textbf{Observation 1: Magnitude Hierarchy of Quantization Errors.}
Across all linear projection layers (\texttt{q\_proj}, \texttt{k\_proj}, \texttt{v\_proj}, \texttt{o\_proj}, \texttt{gate\_proj}, \texttt{up\_proj}, \texttt{down\_proj}) in both models, we consistently observe that activation quantization errors are approximately 1--2 orders of magnitude larger than weight quantization errors. Specifically, activation MSE ranges from $0.08$ to $0.24$, while weight MSE remains between $2.4\times10^{-3}$ and $4.4\times10^{-3}$. This disparity stems from the inherently dynamic nature of activations, which exhibit greater absolute value ranges and more pronounced distributional variations compared to the relatively static weight parameters. The limited representational capacity of FP4 (only 16 discrete values) introduces larger quantization step sizes when accommodating the wider dynamic range of activations, even with dynamic scaling mechanisms.

\textbf{Observation 2: Non-monotonic Training Dynamics.}
Both weight and activation quantization errors exhibit a characteristic non-monotonic pattern during early training phases. The errors initially increase from 0 to approximately 1,000 training steps, reach a peak around 15,000 steps, and then sharply decrease and stabilize. This phenomenon can be attributed to the interaction between gradient dynamics and quantization precision constraints. During the initial training phase, aggressive gradient updates cause rapid expansion of parameter dynamic ranges as the model transitions from random initialization toward meaningful feature representations. The scaling factors in NVFP4 dynamically adjust to cover these expanding ranges, inevitably increasing quantization step sizes and thus errors. As training progresses beyond 15,000 steps, the model enters a convergence regime where weight updates diminish, distributions stabilize, and normalization layers (\texttt{LayerNorm}/\texttt{RMSNorm}) learn appropriate scaling parameters. Furthermore, the quantization-aware training process enables the model to implicitly adapt its representations to be more amenable to low-precision quantization, effectively suppressing outliers and aligning activation distributions with the nonlinear representational characteristics of FP4.

These observations demonstrate that despite the extreme precision constraints of NVFP4, the quantization errors remain bounded and exhibit healthy convergence behavior, validating the effectiveness of our quantization-aware training approach across both Linear and Softmax Attention architectures.

\begin{figure}[ht]
    \centering
    \begin{tabular}{c}
        \includegraphics[width=0.9\linewidth]{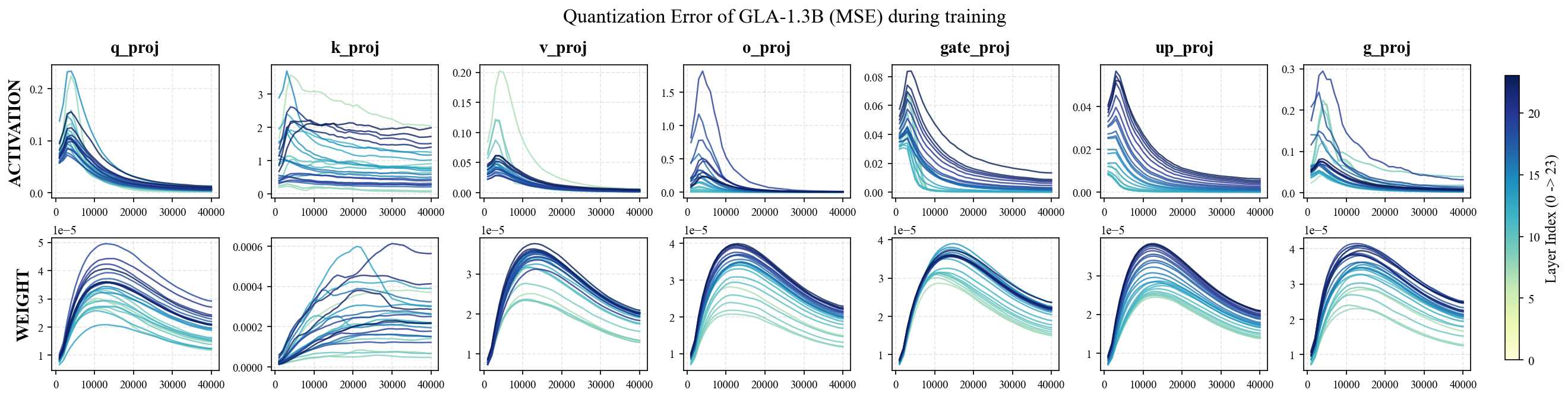}\\
        (a) GLA-1.3B\\
        \includegraphics[width=0.9\linewidth]{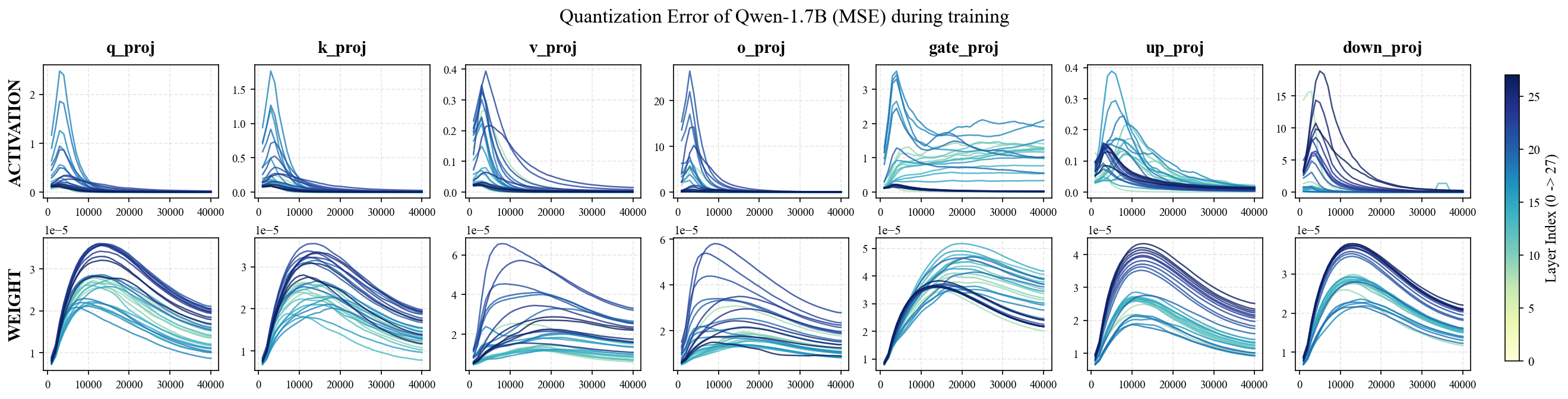}\\
        (b) Qwen-1.7B
    \end{tabular}
    \caption{Quantization error dynamics during NVFP4 training for (a) GLA-1.3B Linear Attention model and (b) Qwen-1.7B Softmax Attention model. Each subplot shows MSE between full-precision and quantized values for both activations (top row) and weights (bottom row) across different linear projection layers. Colors indicate different transformer layers (layer 0--23). Both models exhibit consistent patterns: \ding{172} activation errors are 1--2 orders of magnitude larger than weight errors; \ding{173} errors initially increase in the first 1k steps, peak around 15k steps, then sharply decrease as training stabilizes.}
    \label{fig:mse_error}
\end{figure}

\section{Discussions}
\label{app:discuss}

\paragraph{On the Quantization-Friendliness of Qwen3-Next.} 
Based on our internal analysis of the architectural design, we find that Qwen3-Next is inherently more suitable for quantization than its predecessors. The evidence for this improved quantization-friendliness is as follows:

\begin{enumerate}
    \item \textbf{QK Normalization:} QK-Norm constrains the dynamic range of queries and keys, effectively suppressing the formation of extreme outliers in the attention scores. Empirical data show that Qwen3 exhibits significantly lower kurtosis than Qwen2.5, indicating a lighter-tailed activation distribution that is easier to map to lower-bit integers.
    
    \item \textbf{Linear Attention Dynamics:} The adoption of Gated DeltaNet, a form of linear attention, alleviates numerical instability. Unlike traditional Softmax attention, which can amplify outliers exponentially, the linear transition in DeltaNet maintains a more uniform activation scale.
    
    \item \textbf{Robustness of Gated Mechanisms:} The gated-attention architecture acts as a natural regulator of feature magnitudes. By adaptively scaling information flow, the gating units mitigate the impact of high-variance features, making the model more robust to the precision loss induced by quantization.
    
    \item \textbf{Regularized RMSNorm Parameters:} We observe that the gain parameter ($\gamma$) of RMSNorm in Qwen3-Next is generally smaller than 1, achieved by applying weight decay specifically to the RMSNorm parameters during training. This constraint yields lower and more stable activation magnitudes, directly reducing quantization error in activation-sensitive layers.
\end{enumerate}

\paragraph{On the NVIDIA NVFP4 Recipe}

\begin{enumerate}
    \item \textbf{About Random Hadamard Transform.}
    Random Hadamard transforms (RHTs) mitigate outliers by redistributing activation magnitudes, reducing quantization error and training loss. A key distinction separates the NVIDIA NVFP4 recipe used here from MXFP4 approaches such as Quartet. Quartet applies the transform in both forward and backward passes to avoid overflow in MXFP4, which lacks a global scaling factor, and typically relies on a fused fast Walsh–Hadamard transform (FWHT) implementation. In contrast, our NVFP4-based CHON recipe uses a forward-free Hadamard design: NVFP4’s hierarchical microscaling accommodates dynamic-range variation without forward-pass rotation. This improves deployability by eliminating FWHT overhead and fusion complexity during inference. Restricting the transform to the backward pass therefore preserves hardware efficiency and enables standard deployment without de-rotation.
    
    \item \textbf{About the Last-4 Layer Strategy.}
    Our analysis shows that top-K activation magnitudes increase monotonically during training. We attribute this to (1) RMSNorm’s structural amplification and (2) cumulative outlier excitation across layers. Monitoring indicates that the largest activation spikes consistently occur in the final layers, making them most vulnerable to quantization-induced precision collapse. Quantization difficulty therefore increases with depth, motivating the ``Last4'' strategy (keeping the last four layers in BF16). This is not a heuristic: it reflects the observed magnitude evolution in the Transformer and protects the most volatile late-stage residual updates, preserving representational capacity during convergence.    
    
    \item \textbf{About Stochastic Rounding in the Backward Pass.} Unbiased gradients are crucial for stable low-bit optimization. Although MXFP4 training relies on stochastic rounding (SR) for unbiasedness, applying SR directly can add destabilizing variance. We find that the randomized Hadamard transform (RHT) stabilizes SR in the backward pass by scrambling gradients toward a more uniform distribution, reducing localized rounding variance. Together, RHT+SR yields unbiased gradients with controlled variance, which is key to narrowing the FP4–BF16 pretraining gap.

\end{enumerate}

\end{document}